\documentclass{article}
\usepackage{ijcai24}

\usepackage{times}
\usepackage{soul}
\usepackage[hyphens]{url}
\usepackage[hidelinks]{hyperref}
\usepackage[utf8]{inputenc}
\usepackage{caption}
\usepackage{graphicx}
\usepackage{amsmath}
\usepackage{amsthm}
\usepackage{booktabs}
\usepackage{algorithm}
\usepackage{algorithmic}
\usepackage[switch]{lineno}

\makeatletter
\newif\if@restonecol
\makeatother

\usepackage[ruled,vlined,algo2e]{algorithm2e}

\usepackage{newfloat}
\usepackage{listings}
\DeclareCaptionStyle{ruled}{labelfont=normalfont,labelsep=colon,strut=off} % DO NOT CHANGE THIS
\floatstyle{ruled}
\newfloat{listing}{tb}{lst}{}
\floatname{listing}{Listing}
\title{Are Logistic Models Really Interpretable?}
\author{
    Danial Dervovic\textsuperscript{\rm 1}
    \and
    Freddy L\'ecu\'e\textsuperscript{2}
    \and
    Nicol\'as Marchesotti\textsuperscript{3}
    \And
    Daniele Magazzeni\textsuperscript{3}\\
    \affiliations
    \textsuperscript{\rm 1}JP Morgan AI Research, Edinburgh, UK \\  
    \textsuperscript{\rm 2}JP Morgan AI Research, New York City, NY, USA \\
    \textsuperscript{\rm 3}JP Morgan AI Research, London, UK \\
    \emails
    \{danial.dervovic, freddy.lecue, nicolas.p.marchesotti, daniele.magazzeni\}@jpmchase.com
    }
\usepackage{todonotes}

\usepackage{amsthm}
\usepackage{amssymb}
\usepackage{enumitem}
\usepackage{wrapfig}

\usepackage{booktabs}       % professional-quality tables
\usepackage{amsfonts}       % blackboard math symbols
\usepackage{nicefrac}       % compact symbols for 1/2, etc.
\usepackage{microtype}      % microtypography
\usepackage{xcolor}         % colors

\definecolor{codegreen}{rgb}{0,0.6,0}
\definecolor{codegray}{rgb}{0.5,0.5,0.5}
\definecolor{codepurple}{rgb}{0.58,0,0.82}
\definecolor{backcolour}{rgb}{0.95,0.95,0.92}

\lstdefinestyle{mystyle}{
    backgroundcolor=\color{backcolour},   
    commentstyle=\color{codegreen},
    keywordstyle=\color{magenta},
    numberstyle=\tiny\color{codegray},
    stringstyle=\color{codepurple},
    basicstyle=\ttfamily\scriptsize,
    breakatwhitespace=false,         
    breaklines=true,                 
    captionpos=b,                    
    keepspaces=true,                 
    showspaces=false,                
    showstringspaces=false,
    showtabs=false,                  
    tabsize=2
}

\usepackage{pgfplots}
\DeclareUnicodeCharacter{2212}{−}
\usepgfplotslibrary{groupplots,dateplot}
\usetikzlibrary{patterns,shapes.arrows}
\pgfplotsset{compat=newest}

\usepackage{longtable}[=v4.13]

\usepackage{multirow}
\usepackage{caption}
\usepackage{subcaption}
\usepackage{rotating}

\usepackage{amssymb}
\usepackage{mathtools}

\usepackage{multibib}
\newcites{SM}{Supplementary References}

\theoremstyle{plain}
\newtheorem{theorem}{Theorem}[section]
\newtheorem{proposition}[theorem]{Proposition}
\newtheorem{lemma}[theorem]{Lemma}

\newtheorem{remark}[theorem]{Remark}
\theoremstyle{definition}
\newtheorem{definition}[theorem]{Definition}

\usepackage{physics}

\newcommand{\projunit}{\ensuremath\Pi_{[0, 1]}}

\usepackage{pdfpages}

\newcommand{\piecelin}{\ensuremath\mathrm{PL}_3}

\begin{document}

\maketitle

\begin{abstract}
The demand for open and trustworthy AI models points towards widespread publishing of model weights. Consumers of these model weights must be able to act accordingly with the information provided. That said, one of the simplest AI classification models, Logistic Regression (LR), has an unwieldy interpretation of its model weights, with greater difficulties when extending LR to generalised additive models. In this work, we show via a User Study that skilled participants are unable to reliably reproduce the action of small LR models given the trained parameters. As an antidote to this, we define Linearised Additive Models (LAMs), an optimal piecewise linear approximation that augments any trained additive model equipped with a sigmoid link function, requiring no retraining. We argue that LAMs are more interpretable than logistic models -- survey participants are shown to solve model reasoning tasks with LAMs much more accurately than with LR given the same information. Furthermore, we show that LAMs do not suffer from large performance penalties in terms of ROC-AUC and calibration with respect to their logistic counterparts on a broad suite of public financial modelling data.
\end{abstract}

\section{Introduction}
\label{sec:intro}
In high-stakes domains such as finance and healthcare, there is renewed interest in \emph{inherently interpretable} models~\cite{Lipton2018,molnar2022,DARPA}, where the model form is such that it admits useful explanations of its output without any post-processing.
Calls for transparency of algorithms being used on the public are widespread~\cite{veale2018fairness}.
Within finance there is already increased regulatory scrutiny~\cite{OCC2021,EUAIAct,RFI2021} being introduced with regards to the usage of AI, which could extend in the future -- within certain contexts -- to require full algorithmic transparency, i.e. sharing model coefficients.

One of the prototypical inherently interpretable classification models often used as a baseline is Logistic Regression (LR)~\cite{molnar2022}.
Other additive models such as Generalised Additive Models (GAMs) and variants thereof~\cite{EBM,gkolemis2023regionally} generalise LR to have more flexibility~\cite{hastie2009elements}  and are also considered to be interpretable. 
As with LR, such models entail evaluating a real-valued function of the input data in logit, or log-odds space, that is subsequently transformed into probability space via a non-linear logistic link function.
We call this class of models \emph{logistic models} -- a rigorous definition is given in Definition~\ref{def:logistic_additive_model}. 
In some sense, logistic models can be thought of as \emph{reasoning in logit space}, in that the model weights naturally find their interpretation in terms of log-odds rather than probabilities, the units of the eventual model output.

Logistic models are now ubiquitous within Explainable AI (XAI) \cite{EBM,wells2021,vaughan2018explainable} but the literature is scant on evaluating the interpretability of these models. Indeed, there is a small amount of evidence to the contrary~\cite{harris,vonHippel}, with no more thorough study to the authors' knowledge.
To what extent is this family of models truly interpretable?
The present work aims to (at least partially) answer this question.
We show via a User Study that for at least one definition of interpretability, based on Human-Grounded Evaluation~\cite{InterpretabilityRigorousScience}, such models provide limited and misleading explanations. 
For contexts where a certain level of interpretability is required we propose a remedy, \emph{Linearised Additive Models (LAM)}, that largely keeps the properties of any base logistic model the same, while dispensing with the non-linearities that cause confusion when using model weights as explanations. 

\paragraph{Contributions.} We outline the primary contributions below.
% In this work we make the following contributions:
\setlist{nolistsep}
\begin{enumerate}[noitemsep,leftmargin=.75cm]
    \item \textbf{Identification of interpretability limitations for LR.} A concrete motivating example demonstrating that model explanations provided in log-odds can be difficult for humans to interpret.
    \item \textbf{Linearised Additive Models (LAM).} An efficient procedure to convert any trained logistic additive model that reasons in log odds to one that reasons directly about probabilities, without any retraining. 
    For the special case of LR, LAM is rigorously proved to be the optimal approximation out of a large class of possible models.
    \item \textbf{Empirical evaluation of performance preservation}. On a collection of public datasets from credit modelling, we establish that there is only a very small penalty in classification performance and a somewhat larger -- but still small -- penalty in calibration incurred for using LAMs versus logistic models.
    \item \textbf{User evaluation.} We conduct a user study with $N=36$ participants, concluding via Human-Grounded Evaluation that LAMs are more interpretable than logistic models, as suggested by the motivating example.
    The measured outcomes of the user study are statistically significant.
\end{enumerate}

\paragraph{Related Work.}
There is a vast literature on defining and evaluating performance of inherently interpretable models in general, a non-exhaustive group of which is~\cite{Sanjeeb2018,EBM,ExNN,vaughan2018explainable,kraus2019forecasting,de2021spline,pmlr-v119-vidal20a,NEURIPS2019_567b8f5f}.
% None of these works incorporate both monotone constraints and explicit subscale modelling.
The evaluation of interpretability itself is still an open question, with detailed discussion of this issue in the references~\cite{halliwell2022need,chen2022does,narayanan2018humans,Lipton2018}.
There are several different approaches, with the current preferred (and most expensive) approach being User Studies, as this allows one to directly measure resulting outcomes from explanations~\cite{InterpretabilityRigorousScience}.
Indeed, measuring explanation quality via measured outcomes when humans are asked to simulate the action of an algorithm predates most of the AI interpretability literature, for instance works such as~\cite{Kulesza2013}. 
\cite{rong2022towards} provide a recent survey paper on User Studies for evaluating explanations and interpretability of AI models.

The closest works in the literature to this paper are~\cite{COGAM,poursabzi2021manipulating}.
In~\cite{COGAM}, the authors measure interpretability of sparse linear models and GAMs via user studies that measure cognitive load on participants carrying out tasks with these models and their associated explanations.
In the work by~\cite{poursabzi2021manipulating}, users are presented with the coefficients of 2 and 8-variable linear models. 
The quality of the explanations from each model is measured by how adept users are at simulating the action of the model.
Crucially, both these works evaluate regression models and are not concerned with issues arising due to the non-linearity of the sigmoid transformation required for logistic classification models.

The present work highlights that experts are not immune from misinterpreting certain model explanations, as also observed in the works:~\cite{InterpretingInterpretability,XAIDeployment}.

Low-degree polynomial approximations to the sigmoid are employed in Private Machine Learning, e.g.~\cite{Kim2018b,Chen2018,Kim2019}, but to our knowledge the piecewise linear approximation in this work is unique.

\paragraph{Structure.} 
In Section~\ref{sec:lin_prob_modelling} we provide the motivating example and present the LAM definition and optimality results.
Section~\ref{sec:experiments} contains the performance evaluation of LAMs against their logistic counterparts and Section~\ref{sec:user_survey} details the User Study.
We include most detail in the main text for the User Study, providing derivations, proofs and experimental details in the Supplementary Material (SM). 
We conclude with limitations and future work in Section~\ref{sec:conclusion}.

\paragraph{Notation.}
% \label{sec:prelims}

This work considers binary classification, and we use $\{(\vb*{x}^{(j)}, y^{(j)})\}_{j=1}^M$ to denote the data, where $\vb*{x} \in \mathbb{R}^d$ is a vector of numerical features in a given dataset.
The binary labels are indicators of some (usually bad) event such as a loan default: $y_i \in \{0 , 1\}$.
We assume data points are drawn i.i.d. 
For a point $\vb*{x} \in \mathbb{R}^d$, we denote the $i$\textsuperscript{th} element of $\vb*{x}$ by $x_i$.
% Individual data points are indexed in the superscript, so that datapoint $j$ is given by $\vb*{x}^{(j)}$ with $i$\textsuperscript{th} element denoted by $x_i^{(j)}$.
The $i$\textsuperscript{th} Euclidean basis vector is denoted by $\vb{e}_i$.
The sigmoid function is defined as $\sigma(z) := (1 + e^{-z})^{-1}$ for $z \in \mathbb{R}$.
We follow credit modelling terminology, where \emph{risk} $\hat{y}(\vb*{x}^{(j)})$ is a model's subjective probability in $[0, 1]$ for a data point $\vb*{x}^{(j)}$ to be of positive class, that is $y^{(j)} = 1$.
The set $[d] := \{1, \ldots, d\}$ for $d \in \mathbb{N}$.

\section{Logistic and Linear Probability Modelling}
\label{sec:lin_prob_modelling}

GAMs~\cite{hastie2009elements,molnar2022} are a widely-known and long standing class of models considered to be inherently interpretable.
In this work we restrict attention to GAMs without feature interactions.
We formalise the notion of additive models as understood in this paper in Definition~\ref{def:logistic_additive_model}.

\begin{definition}[Logistic Additive Model]
\label{def:logistic_additive_model}
Let $\vb*{x} \in \mathbb{R}^d$. We call $\hat{y} : \mathbb{R}^d \to [0, 1]$ a \emph{logistic additive model} if takes the form
$\hat{y}(\vb*{x}) = \sigma(f(\vb*{x})) := \sigma ( \beta_0 + \sum_{i=1}^{d} \beta_i f_i(x_i) )$,
where the \emph{bias} $\beta_0 \in \mathbb{R}$ and for all $i \in [d]$, $\beta_i \in \mathbb{R}$ and the $f_i : \mathbb{R} \to \mathbb{R}$ are univariate shape functions.
We may also refer to $f(\vb*{x})$ as a logistic additive model with $\sigma \circ f$ implicit when the context is clear.
\end{definition}
We refer to logistic additive models as defined in Definition~\ref{def:logistic_additive_model} simply as logistic models or additive models when it is clear from context.
The simplest and most common additive model in wide usage is LR, where $f_i(x_i) = x_i$ for all $i \in [d]$~\cite{hastie2009elements}. 
Another example would be an Explainable Boosting Machine of~\cite{EBM} for classification, where the $f_i$ are piecewise constant functions (when there are no feature interactions).

\subsection{Logistic Modelling and Interpretability}
Historically, linear probability modelling, i.e. linear regression on dichotomous variables, was used prior to the advent of efficient methods for fitting LR models; see~\cite{Aldrich1984,hastie2009elements}. %, with a direct interpretation of the regression coefficients as the change in predicted probability arising from a unit change in the corresponding input variable.
It is generally accepted that the application of linear regression to binary classification problems is unwise due to the propensity of the model returning probability estimates outside the $[0, 1]$ interval and sensitivity to outliers~\cite{NgLogReg,hastie2009elements,molnar2022}.
In certain circumstances, these issues are not observed, with linear regression obtaining similar classification performance to LR~\cite{Hellevik2009-cw,vonHippel}.
LR models have superceded linear probability models.

The LR model coefficients are typically interpreted as follows~\cite{molnar2022,hastie2009elements}: a unit change in variable $x_i$ leads to a multiplicative increase in odds for the positive class of $\exp(\beta_i)$.
However, as observed by~\cite{harris,vonHippel} this interpretation can be unwieldy for experts and nigh-on impossible for non-experts to reason with when we are concerned with probabilities, which is how the model outputs are typically presented and thought about.

\paragraph{Motivating Example.}\label{sec:motivating_example} Suppose there is a LR model $\hat{y}$ used to predict some negative outcome and coefficients are shared with downstream users of the model.
Inputs with risk $\geq 0.5$ are considered ``high-risk'' and ``low-risk'' otherwise.
Referring to Figure~\ref{fig:motivating_example}, Alice has a predicted risk of $\hat{y}(\vb*{x}^{(A)}) = 0.1$ and Bob has a predicted risk of $\hat{y}(\vb*{x}^{(B)}) = 0.25$.
Both Alice and Bob are interested in what happens to their risk if they increase the value of feature $x_i$ by one unit, all else equal.
They are told the model coefficient $\beta_i = 1.61 \approx \ln 5$, a common form of transparent model explanation.
First, both exponentiate $\beta_i$ which gives 5. 
% This is already a calculation one cannot expect a user to do in their head.
They now know that increasing feature $x_i$ by increases their odds by a factor 5.
The model outputs a risk score in units of probability, so they now have to compute what increasing their odds corresponds to in probabilities.

\begin{figure}
  \begin{center}
        \includegraphics[width=0.8\columnwidth]{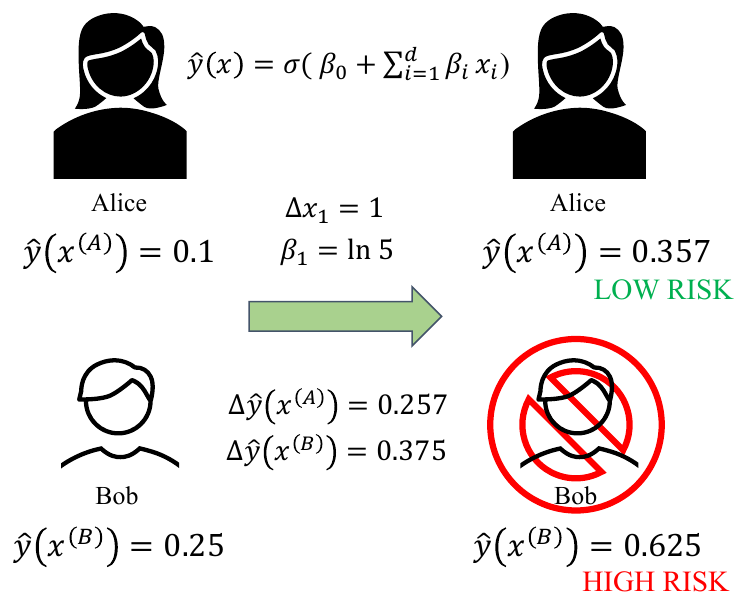}
  \end{center}
  \caption{The cost of misinterpration of model coefficients as explanations. Alice and Bob receive the same explanation, but incur a different change in model output, ultimately leading to different outcomes.}
  \label{fig:motivating_example}
\end{figure}

In this instance for Alice, $\operatorname{odds}(A) = \hat{y}(\vb*{x}^{(A)}) / (1 - \hat{y}(\vb*{x}^{(A)})) = 0.1 / (1 - 0.1) \approx 0.111$.
Alice then multiplies her odds by 5, yielding $0.556$.
Converting back to probabilities we have $\hat{y}(\vb*{x}^{(A)} + \vb{e}_i) \approx 0.556 / (1 + 0.556) \approx 0.357$.
Subtracting her original risk score, we have that increasing $x_i$ by one unit increases the risk by a probability of $0.257$ and Alice remains low-risk.
A similarly laborious computation gives an increase in risk for Bob to approximately $0.625$. 
Bob would be considered high-risk under a unit increase in $x_i$.
This was not obvious on first inspection before carrying out the computation explicitly. 

The nonlinearity of odds as a function of probabilities (and vice versa) means that  \emph{users with different risk scores cannot attribute logistic regression model outputs to the model coefficients in the same way.}
Moreover, the necessary computations are such that one cannot easily reason about the model's input-output relationship without a significant amount of practice.
On the contrary, the coefficients $\beta_i$ of a linear probability model admit the more direct interpretation of the increase in output model probability arising from a unit increase in $x_i$, regardless of the risk value of the user in question.

\subsection{Linearised Additive Models (LAMs)}
Given the preceding example, we wish to keep the interpretability characteristics of linear probability modelling while simultaneously sidestepping the issues arising from using a non-linear link function as in LR.
To this end, we present the Linearised Additive Model (LAM), which is defined with respect to an already trained logistic model, $\sigma \circ f$.
Informally, a LAM replaces the sigmoid link function with a clipping function and scales $f$ by an affine transformation.
Denote by $\projunit$ the projector from $\mathbb{R}$ onto the unit interval, that is $\projunit(z) = \max(0, \min(1, z))$.

\begin{definition}[Linearised Additive Models (LAM)]
\label{def:linearised_additive_models}
Let $\vb*{x} \in \mathbb{R}^{d}$ and let $\hat{y}(\vb*{x}) = \sigma(f(\vb*{x}))$ be an additive model per Definition~\ref{def:logistic_additive_model}.
Moreover, set $\alpha^\star := \frac{80000}{30773} \approx 2.5996$ as a universal constant.
Then, the \emph{Linearised Additive Model}, $\hat{y}_{\text{LAM}}$, relative to $f$ is given by
$$\hat{y}_{\text{LAM}}(\vb*{x})  = \projunit\qty(\frac{1}{2} + \frac{\beta_0}{2\alpha^\star} + \sum_{i=1}^{d} \frac{\beta_i}{2 \alpha^\star} f_i(x_i)).$$
For brevity we refer to $\hat{y}_{\text{LAM}}$ as a \emph{linearised} model and say that $\hat{y}_{\text{LAM}}$ is the \emph{LAM induced by $\hat{y}$}.
\end{definition}

\begin{figure}
  \begin{center}
        \resizebox{\columnwidth}{!}{%
        \input{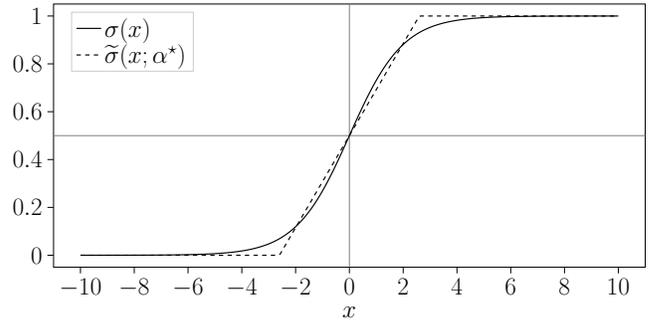}
        }
  \end{center}
  \caption{Optimal approximation $\widetilde{\sigma}(x; \alpha \approx 2.5996)$ to sigmoid function $\sigma(x)$. The parameter $\alpha$ corresponds the half the width along the $x$-axis of the middle line segment in the piecewise linear function.}
  \label{fig:piecewise_approx_sig}
\end{figure}

The LAM as defined in Definition~\ref{def:linearised_additive_models} is derived by considering the optimal 3-piece piecewise linear approximation (see Figure~\ref{fig:piecewise_approx_sig}) to the sigmoid function in terms of squared error, using this function as a link function for $f(\vb*{x})$ and invoking linearity. 
We choose this family of approximating functions as 3 pieces gives the simplest non-trivial approximation to the logistic sigmoid, while simulataneously allowing for a similar interpretation to a linear probability model. 
Squared error is chosen as it is a common metric for function approximation~\cite{hastie2009elements} for which we are able to derive a universal tractable approximator.

As an example, an LR model is written as $\hat{y}(\vb*{x}) = \sigma ( \sum_{i=0}^{d} \beta_i x_i)$, where we fix $x_0 = 1$ according to convention.
Then, the linearised version will be $\hat{y}_{\text{LAM}}(\vb*{x}) = \projunit(\frac{1}{2} + \sum_{i=0}^{d} \frac{\beta_i}{2 \alpha^\star} x_i)$.
In the case of linearised LR we can interpret the coefficients $\beta_i / 2\alpha^\star$ as the contribution to model output in probability space for a unit increase in $x_i$, as with a linear probability model.

\begin{remark}
To train a LAM, all that is required is to train the underlying logistic additive model, then apply Definition~\ref{def:linearised_additive_models} using the trained coefficients $\{\beta_i\}_{i=0}^{d}$.
\end{remark}

\paragraph{Motivating Example Revisited.} 
Under an LR model, Alice and Bob had to interpret a unit increase in feature $x_i$ as having different and unintuitive effects on their respective risk scores.
Using the induced LAM, a unit increase in $x_i$ gives rise to \emph{the same change in model output, regardless of the input}, i.e. $\frac{\beta_i}{2\alpha^\star} = \frac{1.61}{2 \times 2.5996}\approx 0.310$, modulo outputs greater than unity which will be clipped.
Alice and Bob's respective inputs of 0.1 and 0.25 can be readily seen to change to 0.41 and 0.56.
Note this gives the same qualitative result as the non-linearised model, namely Alice stays low-risk and Bob becomes high-risk. 
We surmise that this direct interpretation of the LAM coefficients incurs smaller cognitive overhead in answering questions of this type, in constrast to LR.

The following optimality result for the LAM approximation to LR lends theoretical support to our definition of LAMs.
\begin{theorem}[LAM Optimality]
\label{thm:lam_opt_main}
   Let $\piecelin$ be the space of 3-piece piecewise linear functions of one variable and $\mathcal{X} = \mathbb{R}^d$.
    For any LR model $\hat{y}(\vb*{x}) = \sigma(f(\vb*{x})) = \sigma(\beta_0 + \sum_{i = 1}^d \beta_i x_i)$ on $\mathcal{X}$, an approximator  $\widetilde{\sigma}(f(\vb*{x}))$ is defined for all $\widetilde{\sigma} \in \piecelin$. 
    Then, $\hat{y}_{\mathrm{LAM}}$ is the squared-error optimal approximator for arbitrary $f$, that is,
    \begin{align*}
        \hat{y}_{\mathrm{LAM}}(\vb*{x}) &= \widetilde{\sigma}(f(\vb*{x}); \alpha^\star) \qq{where}\\
    \widetilde{\sigma}(\,\cdot\,; \alpha^\star) &= \arg\min_{\widetilde{\sigma} \in \piecelin} \qty{ \int_{\mathcal{X}} \qty(\widetilde{\sigma}(f(\vb*{x})) - \sigma(f(\vb*{x})))^2 \dd \vb*{x} }, 
    \end{align*}
    with $\widetilde{\sigma}(z; \alpha^\star) := \projunit(\frac{1}{2}(1 + \frac{z}{\alpha^\star}))$, $\alpha^\star \approx 2.5996$.
\end{theorem}
\begin{proof}[Proof (Sketch)]
One can show via symmetry arguments that the minimising approximator comes from a one-parameter function family, when $d = 1$ with $f(x) = x$.
The error is a convex function of this parameter $\alpha$, which is minimised at $\alpha^\star$. 
The argument can then be extended to $d$-dimensional, affine $f(\vb*{x})$ by considering the error integral explicitly, from which the result follows.
\end{proof}
The proof of Theorem~\ref{thm:lam_opt_main} and supporting results is given in full Appendix~\ref{sec:lam_derivation}.
% We shall see in Section~\ref{sec:experiments} that the quality of approximation of LAMs to the underlying logistic models holds up empirically. 

\section{Performance Comparison}
\label{sec:experiments}

In this section we detail our experiments comparing the model performance of logistic additive models against their linearised (LAM) counterparts.
% {\color{purple} Given that LAMs modify the probability assignments given by the underlying additive models, we are particularly interested in the effect this may have on classifier calibration, which we investigate in Section~\ref{sec:results_calib}. }

\subsection{Experimental Setup}

\paragraph{Models.}
We compare all models to XGBoost~\cite{xgboost}, with shorthand XGB, and XGB with monotone constraints imposed (MonoXGB).
XGB classification performance serves as an effective upper-bound on the competing models.
Indeed, there has been much discussion in recent years about a tradeoff between model accuracy and interpretability~\cite{Rudin19,Dziugaite2020} and XGB is included here to illustrate this tradeoff.
Nonetheless, when using black-box models with post-hoc explanations care must be taken~\cite{zhou2022feature,SlackEtAl20,Shaikhina2021}, especially in sensitive domains.
As state-of-the-art baseline logistic additive models we consider the Additive Risk Models (ARMs) of~\cite{CHEN2022113647}, since they are GAMs that explicitly incorporate monotone constraints that are often required in sensitive domains such as finance~\cite{SR11-7}.
These models come in 1-layer (ARM1) and 2-layer (ARM2) variants.
Further baselines include NNLR (Non-negative LR~\cite{CHEN2022113647}) with raw feature inputs as an effective lower bound on model performance.
For an additive model with shorthand $M$, we denote its linearised version (in the sense of Definition~\ref{def:linearised_additive_models}) by LAM-$M$.
We include the linearised models LAM-NNLR, LAM-ARM1 and LAM-ARM2 in our experiments.
LAM-ARM2 has both the individual subscale models and global NNLR model linearised.
Table~\ref{tab:classifier_desc} provides a quick lookup table of the $k=8$ models.
Model hyperparameters are provided in Appendix~\ref{sec:hyperparams} and no attempts were made to tune them due to computational constraints.

\begin{table}
\centering
        {\scriptsize
        \begin{tabular}{lc}
\toprule
\textbf{Classifier} & \textbf{Description} \\
\midrule
NNLR & Non-negative logistic regression on unprocessed features. \\ 
LAM-NNLR & Linearised Additive Model (LAM) induced by NNLR. \\ 
ARM1 & One-Layer Additive Risk Model from~\citeSM{CHEN2022113647}. \\ 
LAM-ARM1 & LAM induced by ARM1. \\
XGB & XGBoost~\citeSM{xgboost} model. \\
MonoXGB & XGBoost~\citeSM{xgboost} model with monotone constraints imposed. \\
\midrule
ARM2 & Two-Layer Additive Risk Model from~\citeSM{CHEN2022113647}. \\
LAM-ARM2 & LAM induced by ARM2. Both logistic layers are linearised. \\
\bottomrule
\end{tabular}
        }
        \caption{Shorthand and descriptions of $k = 8$ classifiers under comparison. 
        Models in the second section have two layers, with first layer operating on subscales and the second layer combining the individual subscale scores.
        }
        \label{tab:classifier_desc}
\end{table}

\paragraph{Datasets.}
In this work we are principally interested in the consumer credit domain.
Bankruptcy prediction datasets are also included due to their similar problem structure and origin. 
We consider publically available datasets from the UCI repository~[\citeauthor{UCI}], namely, the German Credit dataset~\cite{misc_statlog_(german_credit_data)_144}, Australia credit approvals~\cite{misc_statlog_(australian_credit_approval)_143}, Taiwanese bankruptcy~\cite{misc_taiwanese_bankruptcy_prediction_572} prediction, Japanese credit screening~\cite{misc_japanese_credit_screening_28} and the Polish companies bankruptcy~\cite{misc_polish_companies_bankruptcy_365} dataset.
We consider also the FICO Home Equity Line of Credit dataset~(HELOC)~\cite{heloc}, Give Me Some Credit~(GMSC) and Lending Club~(LC)~\cite{lending-club} datasets.
% Full dataset preprocessing is described in Appendix~B.5.

\subsection{Performance Metrics}

We are chiefly interested to what extent linearising relative to logistic models introduces degredation (if any) of both classification performance and calibration -- the latter being of interest as we are modifying the probability estimates of a trained logistic model.
For each metric the 10-fold stratified cross-validation score is computed for every (classifier, dataset) combination. 

\textbf{Classification Performance.}
To measure of classification performance, we use the area under the curve of the receiver operating characteristic~\cite{BRADLEY19971145,Hanley1983}, denoted as AUC.
% As a metric for binary classification performance, AUC is ubiquitous within the Machine Learning community~\cite{BRADLEY19971145,Hanley1983} with numerous interpretations~\cite{Flach2011,FAWCETT2006861}.
% Briefly, AUC gives a decision threshold invariant measure of the quality of the ranking of data points induced by a classifier, with a minimum value of 0 (inverse of perfect classifier) and maximum of 1 (perfect classifier), with $\mathrm{AUC} = \frac{1}{2}$ corresponding to a random guess.  

\textbf{Calibration.}
We consider two widely-used numerical summary statistics for the calibration, \emph{Expected Calibration Error (ECE)} and \emph{Maximum Calibration Error, (MCE)} with full details provided in Appendix~B.4.
Lower values of ECE and MCE correspond to better calibration of a particular model, with the idealised model having a value of zero for both.

\paragraph{Statistical Methodology.}

The statistical testing for the performance evaluation is described in full in Appendix~\ref{sec:stats}. 
For the purposes of discussion here, consider a graph where for each classification algorithm $\mathcal{A}$ we draw a node. 
We draw an edge between any nodes corresponding to algorithm pairs $(\mathcal{A}, \mathcal{A}')$ such that the performance of $\mathcal{A}$ cannot be distinguished from the performance $\mathcal{A}'$ with significance $\alpha=0.05$ according to a a Wilcoxon signed-rank test~\cite{Wilcoxon1945} conducted over the considered datasets.
Cliques in this graph correspond to algorithms that are mutually indistinguishable.
We display this graph, where the nodes $\mathcal{A}$ are arranged according to their average rank $R_{\mathcal{A}}$ in Figure~\ref{fig:three graphs}.
These are the Critical Difference (CD) diagrams of~\cite{Demsar2006} for AUC, ECE and MCE .
Not only are we interested in whether an observed difference in cross validated score between two algorithms is statistically significant, but also the size of this difference.
In Table~\ref{tab:comparison_auc} for the AUC metric,~\ref{tab:comparison_ece} (Appendix)~and~\ref{tab:comparison_mce}
~(Appendix) we tabulate the differences $\hat{\theta}_{\text{HL}}$ between all pairings $(\mathcal{A}, \mathcal{A}')$ for the AUC, ECE and MCE metrics respectively.
The quantity $\hat{\theta}_{\text{HL}}$ is a robust point estimate\footnote{The Hodges-Lehmann estimator associated to Wilcoxon's signed rank test~\cite{WILCOX202245}.} computed across the datasets.

\begin{figure}
     \centering
     \begin{subfigure}[b]{0.49\textwidth}
        \centering
        \resizebox{\columnwidth}{!}{%
            % This file was created with tikzplotlib v0.9.17.
\begin{tikzpicture}[font=\huge]

\begin{axis}[
clip=false,
height=6cm,
hide x axis,
hide y axis,
tick align=outside,
tick pos=left,
title style={yshift=-0.4cm},
title={AUC},
width=16cm,
x grid style={white!69.0196078431373!black},
xmin=0, xmax=1,
xtick style={color=black},
y dir=reverse,
y grid style={white!69.0196078431373!black},
ymin=0, ymax=1,
ytick style={color=black}
]
\addplot [draw=black, fill=black, mark=*, only marks]
table{%
x  y
0.209821428571429 0.333333333333333
0.214464285714286 0.333333333333333
};
\addplot [draw=black, fill=black, mark=*, only marks]
table{%
x  y
0.683392857142857 0.352941176470588
0.569642857142857 0.352941176470588
};
\addplot [draw=black, fill=black, mark=*, only marks]
table{%
x  y
0.683392857142857 0.392156862745098
0.718214285714286 0.392156862745098
};
\addplot [semithick, white]
table {%
0 0
1 1
};
\addplot [thick, black]
table {%
0.11 0.254901960784314
0.89 0.254901960784314
};
\addplot [thick, black]
table {%
0.89 0.196078431372549
0.89 0.254901960784314
};
\addplot [thick, black]
table {%
0.834285714285714 0.225490196078431
0.834285714285714 0.254901960784314
};
\addplot [thick, black]
table {%
0.778571428571429 0.196078431372549
0.778571428571429 0.254901960784314
};
\addplot [thick, black]
table {%
0.722857142857143 0.225490196078431
0.722857142857143 0.254901960784314
};
\addplot [thick, black]
table {%
0.667142857142857 0.196078431372549
0.667142857142857 0.254901960784314
};
\addplot [thick, black]
table {%
0.611428571428571 0.225490196078431
0.611428571428571 0.254901960784314
};
\addplot [thick, black]
table {%
0.555714285714286 0.196078431372549
0.555714285714286 0.254901960784314
};
\addplot [thick, black]
table {%
0.5 0.225490196078431
0.5 0.254901960784314
};
\addplot [thick, black]
table {%
0.444285714285714 0.196078431372549
0.444285714285714 0.254901960784314
};
\addplot [thick, black]
table {%
0.388571428571429 0.225490196078431
0.388571428571429 0.254901960784314
};
\addplot [thick, black]
table {%
0.332857142857143 0.196078431372549
0.332857142857143 0.254901960784314
};
\addplot [thick, black]
table {%
0.277142857142857 0.225490196078431
0.277142857142857 0.254901960784314
};
\addplot [thick, black]
table {%
0.221428571428571 0.196078431372549
0.221428571428571 0.254901960784314
};
\addplot [thick, black]
table {%
0.165714285714286 0.225490196078431
0.165714285714286 0.254901960784314
};
\addplot [thick, black]
table {%
0.11 0.196078431372549
0.11 0.254901960784314
};
\addplot [thick, black]
table {%
0.209821428571429 0.254901960784314
0.209821428571429 0.647058823529412
0.1 0.647058823529412
};
\addplot [thick, black]
table {%
0.214464285714286 0.254901960784314
0.214464285714286 0.741176470588235
0.1 0.741176470588235
};
\addplot [thick, black]
table {%
0.374642857142857 0.254901960784314
0.374642857142857 0.835294117647059
0.1 0.835294117647059
};
\addplot [thick, black]
table {%
0.411785714285714 0.254901960784314
0.411785714285714 0.929411764705882
0.1 0.929411764705882
};
\addplot [thick, black]
table {%
0.569642857142857 0.254901960784314
0.569642857142857 0.929411764705882
0.9 0.929411764705882
};
\addplot [thick, black]
table {%
0.683392857142857 0.254901960784314
0.683392857142857 0.835294117647059
0.9 0.835294117647059
};
\addplot [thick, black]
table {%
0.718214285714286 0.254901960784314
0.718214285714286 0.741176470588235
0.9 0.741176470588235
};
\addplot [thick, black]
table {%
0.818035714285714 0.254901960784314
0.818035714285714 0.647058823529412
0.9 0.647058823529412
};
\addplot [very thick, black]
table {%
0.211821428571429 0.333333333333333
0.212464285714286 0.333333333333333
};
\addplot [very thick, black]
table {%
0.571642857142857 0.352941176470588
0.681392857142857 0.352941176470588
};
\addplot [very thick, black]
table {%
0.685392857142857 0.392156862745098
0.716214285714286 0.392156862745098
};
\draw (axis cs:0.89,0.176470588235294) node[
  scale=0.9,
  anchor=south,
  text=black,
  rotate=0.0
]{1};
\draw (axis cs:0.778571428571429,0.176470588235294) node[
  scale=0.9,
  anchor=south,
  text=black,
  rotate=0.0
]{2};
\draw (axis cs:0.667142857142857,0.176470588235294) node[
  scale=0.9,
  anchor=south,
  text=black,
  rotate=0.0
]{3};
\draw (axis cs:0.555714285714286,0.176470588235294) node[
  scale=0.9,
  anchor=south,
  text=black,
  rotate=0.0
]{4};
\draw (axis cs:0.444285714285714,0.176470588235294) node[
  scale=0.9,
  anchor=south,
  text=black,
  rotate=0.0
]{5};
\draw (axis cs:0.332857142857143,0.176470588235294) node[
  scale=0.9,
  anchor=south,
  text=black,
  rotate=0.0
]{6};
\draw (axis cs:0.221428571428571,0.176470588235294) node[
  scale=0.9,
  anchor=south,
  text=black,
  rotate=0.0
]{7};
\draw (axis cs:0.11,0.176470588235294) node[
  scale=0.9,
  anchor=south,
  text=black,
  rotate=0.0
]{8};
\draw (axis cs:0.14,0.617647058823529) node[
  scale=0.7,
  anchor=east,
  text=black,
  rotate=0.0
]{7.1};
\draw (axis cs:0.09,0.647058823529412) node[
  scale=0.8,
  anchor=east,
  text=black,
  rotate=0.0
]{NNLR};
\draw (axis cs:0.14,0.711764705882353) node[
  scale=0.7,
  anchor=east,
  text=black,
  rotate=0.0
]{7.1};
\draw (axis cs:0.09,0.741176470588235) node[
  scale=0.8,
  anchor=east,
  text=black,
  rotate=0.0
]{LAM-NNLR};
\draw (axis cs:0.14,0.805882352941176) node[
  scale=0.7,
  anchor=east,
  text=black,
  rotate=0.0
]{5.6};
\draw (axis cs:0.09,0.835294117647059) node[
  scale=0.8,
  anchor=east,
  text=black,
  rotate=0.0
]{LAM-ARM2};
\draw (axis cs:0.14,0.9) node[
  scale=0.7,
  anchor=east,
  text=black,
  rotate=0.0
]{5.3};
\draw (axis cs:0.09,0.929411764705882) node[
  scale=0.8,
  anchor=east,
  text=black,
  rotate=0.0
]{ARM2};
\draw (axis cs:0.86,0.9) node[
  scale=0.7,
  anchor=west,
  text=black,
  rotate=0.0
]{3.9};
\draw (axis cs:0.91,0.929411764705882) node[
  scale=0.8,
  anchor=west,
  text=black,
  rotate=0.0
]{LAM-ARM1};
\draw (axis cs:0.86,0.805882352941176) node[
  scale=0.7,
  anchor=west,
  text=black,
  rotate=0.0
]{2.9};
\draw (axis cs:0.91,0.835294117647059) node[
  scale=0.8,
  anchor=west,
  text=black,
  rotate=0.0
]{MonoXGB};
\draw (axis cs:0.86,0.711764705882353) node[
  scale=0.7,
  anchor=west,
  text=black,
  rotate=0.0
]{2.5};
\draw (axis cs:0.91,0.741176470588235) node[
  scale=0.8,
  anchor=west,
  text=black,
  rotate=0.0
]{ARM1};
\draw (axis cs:0.86,0.617647058823529) node[
  scale=0.7,
  anchor=west,
  text=black,
  rotate=0.0
]{1.6};
\draw (axis cs:0.91,0.647058823529412) node[
  scale=0.8,
  anchor=west,
  text=black,
  rotate=0.0
]{XGB};
\end{axis}

\end{tikzpicture}
        }
     \end{subfigure} 
     \begin{subfigure}[b]{0.49\textwidth}
        \centering
        \resizebox{\columnwidth}{!}{%
            % This file was created with tikzplotlib v0.9.17.
\begin{tikzpicture}[font=\huge]

\begin{axis}[
clip=false,
height=6cm,
hide x axis,
hide y axis,
tick align=outside,
tick pos=left,
title style={yshift=-0.4cm},
title={ECE},
width=16cm,
x grid style={white!69.0196078431373!black},
xmin=0, xmax=1,
xtick style={color=black},
y dir=reverse,
y grid style={white!69.0196078431373!black},
ymin=0, ymax=1,
ytick style={color=black}
]
\addplot [draw=black, fill=black, mark=*, only marks]
table{%
x  y
0.263214285714286 0.333333333333333
0.300357142857143 0.333333333333333
};
\addplot [draw=black, fill=black, mark=*, only marks]
table{%
x  y
0.632321428571429 0.352941176470588
0.64625 0.352941176470588
};
\addplot [draw=black, fill=black, mark=*, only marks]
table{%
x  y
0.632321428571429 0.392156862745098
0.757678571428571 0.392156862745098
0.794821428571429 0.392156862745098
};
\addplot [semithick, white]
table {%
0 0
1 1
};
\addplot [thick, black]
table {%
0.11 0.254901960784314
0.89 0.254901960784314
};
\addplot [thick, black]
table {%
0.89 0.196078431372549
0.89 0.254901960784314
};
\addplot [thick, black]
table {%
0.834285714285714 0.225490196078431
0.834285714285714 0.254901960784314
};
\addplot [thick, black]
table {%
0.778571428571429 0.196078431372549
0.778571428571429 0.254901960784314
};
\addplot [thick, black]
table {%
0.722857142857143 0.225490196078431
0.722857142857143 0.254901960784314
};
\addplot [thick, black]
table {%
0.667142857142857 0.196078431372549
0.667142857142857 0.254901960784314
};
\addplot [thick, black]
table {%
0.611428571428571 0.225490196078431
0.611428571428571 0.254901960784314
};
\addplot [thick, black]
table {%
0.555714285714286 0.196078431372549
0.555714285714286 0.254901960784314
};
\addplot [thick, black]
table {%
0.5 0.225490196078431
0.5 0.254901960784314
};
\addplot [thick, black]
table {%
0.444285714285714 0.196078431372549
0.444285714285714 0.254901960784314
};
\addplot [thick, black]
table {%
0.388571428571429 0.225490196078431
0.388571428571429 0.254901960784314
};
\addplot [thick, black]
table {%
0.332857142857143 0.196078431372549
0.332857142857143 0.254901960784314
};
\addplot [thick, black]
table {%
0.277142857142857 0.225490196078431
0.277142857142857 0.254901960784314
};
\addplot [thick, black]
table {%
0.221428571428571 0.196078431372549
0.221428571428571 0.254901960784314
};
\addplot [thick, black]
table {%
0.165714285714286 0.225490196078431
0.165714285714286 0.254901960784314
};
\addplot [thick, black]
table {%
0.11 0.196078431372549
0.11 0.254901960784314
};
\addplot [thick, black]
table {%
0.147142857142857 0.254901960784314
0.147142857142857 0.647058823529412
0.1 0.647058823529412
};
\addplot [thick, black]
table {%
0.263214285714286 0.254901960784314
0.263214285714286 0.741176470588235
0.1 0.741176470588235
};
\addplot [thick, black]
table {%
0.300357142857143 0.254901960784314
0.300357142857143 0.835294117647059
0.1 0.835294117647059
};
\addplot [thick, black]
table {%
0.458214285714286 0.254901960784314
0.458214285714286 0.929411764705882
0.1 0.929411764705882
};
\addplot [thick, black]
table {%
0.632321428571429 0.254901960784314
0.632321428571429 0.929411764705882
0.9 0.929411764705882
};
\addplot [thick, black]
table {%
0.64625 0.254901960784314
0.64625 0.835294117647059
0.9 0.835294117647059
};
\addplot [thick, black]
table {%
0.757678571428571 0.254901960784314
0.757678571428571 0.741176470588235
0.9 0.741176470588235
};
\addplot [thick, black]
table {%
0.794821428571429 0.254901960784314
0.794821428571429 0.647058823529412
0.9 0.647058823529412
};
\addplot [very thick, black]
table {%
0.265214285714286 0.333333333333333
0.298357142857143 0.333333333333333
};
\addplot [very thick, black]
table {%
0.634321428571429 0.352941176470588
0.64425 0.352941176470588
};
\addplot [very thick, black]
table {%
0.634321428571429 0.392156862745098
0.792821428571429 0.392156862745098
};
\draw (axis cs:0.89,0.176470588235294) node[
  scale=0.9,
  anchor=south,
  text=black,
  rotate=0.0
]{1};
\draw (axis cs:0.778571428571429,0.176470588235294) node[
  scale=0.9,
  anchor=south,
  text=black,
  rotate=0.0
]{2};
\draw (axis cs:0.667142857142857,0.176470588235294) node[
  scale=0.9,
  anchor=south,
  text=black,
  rotate=0.0
]{3};
\draw (axis cs:0.555714285714286,0.176470588235294) node[
  scale=0.9,
  anchor=south,
  text=black,
  rotate=0.0
]{4};
\draw (axis cs:0.444285714285714,0.176470588235294) node[
  scale=0.9,
  anchor=south,
  text=black,
  rotate=0.0
]{5};
\draw (axis cs:0.332857142857143,0.176470588235294) node[
  scale=0.9,
  anchor=south,
  text=black,
  rotate=0.0
]{6};
\draw (axis cs:0.221428571428571,0.176470588235294) node[
  scale=0.9,
  anchor=south,
  text=black,
  rotate=0.0
]{7};
\draw (axis cs:0.11,0.176470588235294) node[
  scale=0.9,
  anchor=south,
  text=black,
  rotate=0.0
]{8};
\draw (axis cs:0.14,0.617647058823529) node[
  scale=0.7,
  anchor=east,
  text=black,
  rotate=0.0
]{7.7};
\draw (axis cs:0.09,0.647058823529412) node[
  scale=0.8,
  anchor=east,
  text=black,
  rotate=0.0
]{LAM-ARM1};
\draw (axis cs:0.14,0.711764705882353) node[
  scale=0.7,
  anchor=east,
  text=black,
  rotate=0.0
]{6.6};
\draw (axis cs:0.09,0.741176470588235) node[
  scale=0.8,
  anchor=east,
  text=black,
  rotate=0.0
]{LAM-ARM2};
\draw (axis cs:0.14,0.805882352941176) node[
  scale=0.7,
  anchor=east,
  text=black,
  rotate=0.0
]{6.3};
\draw (axis cs:0.09,0.835294117647059) node[
  scale=0.8,
  anchor=east,
  text=black,
  rotate=0.0
]{ARM1};
\draw (axis cs:0.14,0.9) node[
  scale=0.7,
  anchor=east,
  text=black,
  rotate=0.0
]{4.9};
\draw (axis cs:0.09,0.929411764705882) node[
  scale=0.8,
  anchor=east,
  text=black,
  rotate=0.0
]{ARM2};
\draw (axis cs:0.86,0.9) node[
  scale=0.7,
  anchor=west,
  text=black,
  rotate=0.0
]{3.3};
\draw (axis cs:0.91,0.929411764705882) node[
  scale=0.8,
  anchor=west,
  text=black,
  rotate=0.0
]{XGB};
\draw (axis cs:0.86,0.805882352941176) node[
  scale=0.7,
  anchor=west,
  text=black,
  rotate=0.0
]{3.2};
\draw (axis cs:0.91,0.835294117647059) node[
  scale=0.8,
  anchor=west,
  text=black,
  rotate=0.0
]{MonoXGB};
\draw (axis cs:0.86,0.711764705882353) node[
  scale=0.7,
  anchor=west,
  text=black,
  rotate=0.0
]{2.2};
\draw (axis cs:0.91,0.741176470588235) node[
  scale=0.8,
  anchor=west,
  text=black,
  rotate=0.0
]{LAM-NNLR};
\draw (axis cs:0.86,0.617647058823529) node[
  scale=0.7,
  anchor=west,
  text=black,
  rotate=0.0
]{1.9};
\draw (axis cs:0.91,0.647058823529412) node[
  scale=0.8,
  anchor=west,
  text=black,
  rotate=0.0
]{NNLR};
\end{axis}

\end{tikzpicture}
        }
     \end{subfigure} \\
     \begin{subfigure}[b]{0.49\textwidth}
        \centering
        \hspace{-0.02\columnwidth}
        \resizebox{\columnwidth}{!}{%
            % This file was created with tikzplotlib v0.9.17.
\begin{tikzpicture}[font=\huge]

\begin{axis}[
clip=false,
height=6cm,
hide x axis,
hide y axis,
tick align=outside,
tick pos=left,
title style={yshift=-0.4cm},
title={MCE},
width=16cm,
x grid style={white!69.0196078431373!black},
xmin=0, xmax=1,
xtick style={color=black},
y dir=reverse,
y grid style={white!69.0196078431373!black},
ymin=0, ymax=1,
ytick style={color=black}
]
\addplot [draw=black, fill=black, mark=*, only marks]
table{%
x  y
0.45125 0.333333333333333
0.553392857142857 0.333333333333333
0.281785714285714 0.333333333333333
0.305 0.333333333333333
};
\addplot [draw=black, fill=black, mark=*, only marks]
table{%
x  y
0.45125 0.372549019607843
0.553392857142857 0.372549019607843
0.574285714285714 0.372549019607843
0.305 0.372549019607843
};
\addplot [draw=black, fill=black, mark=*, only marks]
table{%
x  y
0.45125 0.411764705882353
0.553392857142857 0.411764705882353
0.574285714285714 0.411764705882353
0.472142857142857 0.411764705882353
};
\addplot [draw=black, fill=black, mark=*, only marks]
table{%
x  y
0.45125 0.450980392156863
0.553392857142857 0.450980392156863
0.574285714285714 0.450980392156863
0.67875 0.450980392156863
0.683392857142857 0.450980392156863
};
\addplot [semithick, white]
table {%
0 0
1 1
};
\addplot [thick, black]
table {%
0.11 0.254901960784314
0.89 0.254901960784314
};
\addplot [thick, black]
table {%
0.89 0.196078431372549
0.89 0.254901960784314
};
\addplot [thick, black]
table {%
0.834285714285714 0.225490196078431
0.834285714285714 0.254901960784314
};
\addplot [thick, black]
table {%
0.778571428571429 0.196078431372549
0.778571428571429 0.254901960784314
};
\addplot [thick, black]
table {%
0.722857142857143 0.225490196078431
0.722857142857143 0.254901960784314
};
\addplot [thick, black]
table {%
0.667142857142857 0.196078431372549
0.667142857142857 0.254901960784314
};
\addplot [thick, black]
table {%
0.611428571428571 0.225490196078431
0.611428571428571 0.254901960784314
};
\addplot [thick, black]
table {%
0.555714285714286 0.196078431372549
0.555714285714286 0.254901960784314
};
\addplot [thick, black]
table {%
0.5 0.225490196078431
0.5 0.254901960784314
};
\addplot [thick, black]
table {%
0.444285714285714 0.196078431372549
0.444285714285714 0.254901960784314
};
\addplot [thick, black]
table {%
0.388571428571429 0.225490196078431
0.388571428571429 0.254901960784314
};
\addplot [thick, black]
table {%
0.332857142857143 0.196078431372549
0.332857142857143 0.254901960784314
};
\addplot [thick, black]
table {%
0.277142857142857 0.225490196078431
0.277142857142857 0.254901960784314
};
\addplot [thick, black]
table {%
0.221428571428571 0.196078431372549
0.221428571428571 0.254901960784314
};
\addplot [thick, black]
table {%
0.165714285714286 0.225490196078431
0.165714285714286 0.254901960784314
};
\addplot [thick, black]
table {%
0.11 0.196078431372549
0.11 0.254901960784314
};
\addplot [thick, black]
table {%
0.281785714285714 0.254901960784314
0.281785714285714 0.647058823529412
0.1 0.647058823529412
};
\addplot [thick, black]
table {%
0.305 0.254901960784314
0.305 0.741176470588235
0.1 0.741176470588235
};
\addplot [thick, black]
table {%
0.45125 0.254901960784314
0.45125 0.835294117647059
0.1 0.835294117647059
};
\addplot [thick, black]
table {%
0.472142857142857 0.254901960784314
0.472142857142857 0.929411764705882
0.1 0.929411764705882
};
\addplot [thick, black]
table {%
0.553392857142857 0.254901960784314
0.553392857142857 0.929411764705882
0.9 0.929411764705882
};
\addplot [thick, black]
table {%
0.574285714285714 0.254901960784314
0.574285714285714 0.835294117647059
0.9 0.835294117647059
};
\addplot [thick, black]
table {%
0.67875 0.254901960784314
0.67875 0.741176470588235
0.9 0.741176470588235
};
\addplot [thick, black]
table {%
0.683392857142857 0.254901960784314
0.683392857142857 0.647058823529412
0.9 0.647058823529412
};
\addplot [very thick, black]
table {%
0.283785714285714 0.333333333333333
0.551392857142857 0.333333333333333
};
\addplot [very thick, black]
table {%
0.307 0.372549019607843
0.572285714285714 0.372549019607843
};
\addplot [very thick, black]
table {%
0.45325 0.411764705882353
0.572285714285714 0.411764705882353
};
\addplot [very thick, black]
table {%
0.45325 0.450980392156863
0.681392857142857 0.450980392156863
};
\draw (axis cs:0.89,0.176470588235294) node[
  scale=0.9,
  anchor=south,
  text=black,
  rotate=0.0
]{1};
\draw (axis cs:0.778571428571429,0.176470588235294) node[
  scale=0.9,
  anchor=south,
  text=black,
  rotate=0.0
]{2};
\draw (axis cs:0.667142857142857,0.176470588235294) node[
  scale=0.9,
  anchor=south,
  text=black,
  rotate=0.0
]{3};
\draw (axis cs:0.555714285714286,0.176470588235294) node[
  scale=0.9,
  anchor=south,
  text=black,
  rotate=0.0
]{4};
\draw (axis cs:0.444285714285714,0.176470588235294) node[
  scale=0.9,
  anchor=south,
  text=black,
  rotate=0.0
]{5};
\draw (axis cs:0.332857142857143,0.176470588235294) node[
  scale=0.9,
  anchor=south,
  text=black,
  rotate=0.0
]{6};
\draw (axis cs:0.221428571428571,0.176470588235294) node[
  scale=0.9,
  anchor=south,
  text=black,
  rotate=0.0
]{7};
\draw (axis cs:0.11,0.176470588235294) node[
  scale=0.9,
  anchor=south,
  text=black,
  rotate=0.0
]{8};
\draw (axis cs:0.14,0.617647058823529) node[
  scale=0.7,
  anchor=east,
  text=black,
  rotate=0.0
]{6.5};
\draw (axis cs:0.09,0.647058823529412) node[
  scale=0.8,
  anchor=east,
  text=black,
  rotate=0.0
]{LAM-ARM2};
\draw (axis cs:0.14,0.711764705882353) node[
  scale=0.7,
  anchor=east,
  text=black,
  rotate=0.0
]{6.2};
\draw (axis cs:0.09,0.741176470588235) node[
  scale=0.8,
  anchor=east,
  text=black,
  rotate=0.0
]{LAM-ARM1};
\draw (axis cs:0.14,0.805882352941176) node[
  scale=0.7,
  anchor=east,
  text=black,
  rotate=0.0
]{4.9};
\draw (axis cs:0.09,0.835294117647059) node[
  scale=0.8,
  anchor=east,
  text=black,
  rotate=0.0
]{LAM-NNLR};
\draw (axis cs:0.14,0.9) node[
  scale=0.7,
  anchor=east,
  text=black,
  rotate=0.0
]{4.8};
\draw (axis cs:0.09,0.929411764705882) node[
  scale=0.8,
  anchor=east,
  text=black,
  rotate=0.0
]{ARM1};
\draw (axis cs:0.86,0.9) node[
  scale=0.7,
  anchor=west,
  text=black,
  rotate=0.0
]{4.0};
\draw (axis cs:0.91,0.929411764705882) node[
  scale=0.8,
  anchor=west,
  text=black,
  rotate=0.0
]{NNLR};
\draw (axis cs:0.86,0.805882352941176) node[
  scale=0.7,
  anchor=west,
  text=black,
  rotate=0.0
]{3.8};
\draw (axis cs:0.91,0.835294117647059) node[
  scale=0.8,
  anchor=west,
  text=black,
  rotate=0.0
]{ARM2};
\draw (axis cs:0.86,0.711764705882353) node[
  scale=0.7,
  anchor=west,
  text=black,
  rotate=0.0
]{2.9};
\draw (axis cs:0.91,0.741176470588235) node[
  scale=0.8,
  anchor=west,
  text=black,
  rotate=0.0
]{XGB};
\draw (axis cs:0.86,0.617647058823529) node[
  scale=0.7,
  anchor=west,
  text=black,
  rotate=0.0
]{2.9};
\draw (axis cs:0.91,0.647058823529412) node[
  scale=0.8,
  anchor=west,
  text=black,
  rotate=0.0
]{MonoXGB};
\end{axis}

\end{tikzpicture}
        }
     \end{subfigure}
        \caption{Critical Difference diagrams for AUC, ECE and MCE. The $x$-axis represents the mean rank averaged over all datasets, with each classifier's mean rank reported adjacent to its name (lower rank $\equiv$ better). Classifiers connected by an edge \emph{cannot} be distinguished with significance $\alpha=0.05$.}
        \label{fig:three graphs}
\end{figure}
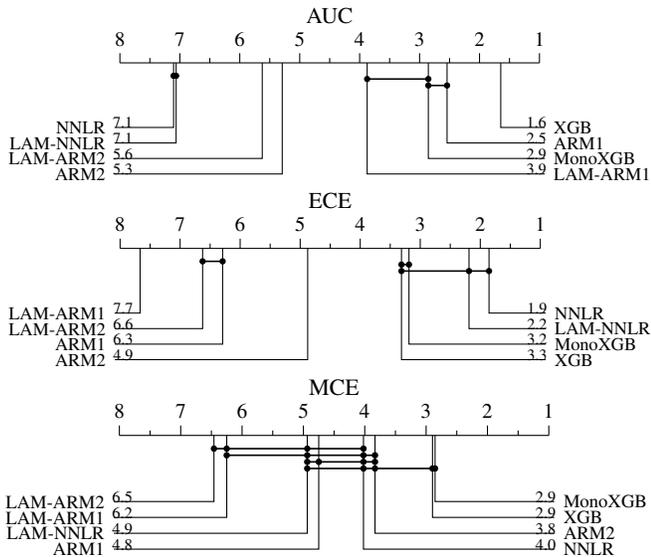

\paragraph{Results for Classification Performance.}
Raw AUC scores for each (classifier, dataset) combination are reported in Table~\ref{tab:auc_raw} in Appendix~\ref{sec:raw_metrics}.
We highlight several observations from the CD diagram for AUC (Figure~\ref{fig:three graphs}).
Unconstrained XGB is the strongest model in terms of AUC, as we may expect.
The weakest performing model families are \{NNLR, LAM-NNLR\}, presumably due to their simplicity as compared with the other models.
LAM-NNLR models are indistinguishable in AUC performance from their logistic counterparts, NNLR.
Interestingly, LAM-ARM1 is connected to MonoXGB in the CD graph.
MonoXGB models are considered state of the art for monotone constrained models on tabular data and LAM-ARM1 is more interpretable according to a number of criteria, yet here they display statistically indistinguishable classification performance.
However, LAM-ARM1 can be distiguished with statistical significance from its logistic counterpart, ARM1.
Nonetheless they are in the same connected component of the CD graph indicating a very similar level of proficiency.
The AUC performance of LAM-ARM2 can be separated from the performance of ARM2.
In this instance we hypothesise that this is happening due to linearisation being applied in two layers as opposed to just one, allowing errors to accumulate.
\emph{In terms of absolute numerical difference in AUC performance, linearisation incurs a very small penalty}, as shown in Table~\ref{tab:comparison_auc}.
The AUC penalties (point estimate) for linearising NNLR, ARM1 and ARM2 are 0.000, 0.003 and 0.003 respectively across the datasets.

\setlength{\tabcolsep}{2pt}

\begin{table}
\centering
        {\scriptsize
        \begin{tabular}{lrrrrrrrr}
\toprule
{} & \rotatebox{70}{NNLR} & \rotatebox{70}{LAM-NNLR} & \rotatebox{70}{ARM1} & \rotatebox{70}{ARM2} & \rotatebox{70}{LAM-ARM1} & \rotatebox{70}{LAM-ARM2} & \rotatebox{70}{MonoXGB} & \rotatebox{70}{XGB} \\
\midrule
NNLR    &                  --- &                   0.000 &      \textbf{-0.327} &      \textbf{-0.314} &         \textbf{-0.319} &         \textbf{-0.311} &         \textbf{-0.320} &     \textbf{-0.340} \\
LAM-NNLR &                  --- &                     --- &      \textbf{-0.327} &      \textbf{-0.314} &         \textbf{-0.319} &         \textbf{-0.311} &         \textbf{-0.320} &     \textbf{-0.340} \\
ARM1    &                  --- &                     --- &                  --- &       \textbf{0.008} &          \textbf{0.003} &          \textbf{0.013} &                   0.000 &     \textbf{-0.011} \\
ARM2    &                  --- &                     --- &                  --- &                  --- &         \textbf{-0.005} &          \textbf{0.003} &         \textbf{-0.009} &     \textbf{-0.017} \\
LAM-ARM1 &                  --- &                     --- &                  --- &                  --- &                     --- &          \textbf{0.009} &                  -0.003 &     \textbf{-0.013} \\
LAM-ARM2 &                  --- &                     --- &                  --- &                  --- &                     --- &                     --- &         \textbf{-0.014} &     \textbf{-0.020} \\
MonoXGB &                  --- &                     --- &                  --- &                  --- &                     --- &                     --- &                     --- &     \textbf{-0.002} \\
XGB     &                  --- &                     --- &                  --- &                  --- &                     --- &                     --- &                     --- &                 --- \\
\bottomrule
\end{tabular}
        }
        \caption{Point estimate for difference in AUC scores between classifiers. Negative values mean column model is better than row model. Bold values indicate statistical significance.}
        \label{tab:comparison_auc}
\end{table}

\paragraph{Results for Calibration.}\label{sec:results_calib}
Raw ECE and MCE scores for each (classifier, dataset) combination are reported in Tables~\ref{tab:ece_raw}~and~\ref{tab:mce_raw} respectively in the Appendix.
Inspecting the CD diagrams for the ECE and MCE calibration metrics in Figure~\ref{fig:three graphs}, we see there is a penalty incurred on model calibration for linearising.
Indeed for both metrics, only the linearisation of NNLR models gives a penalty in calibration that cannot be distinguished from the logistic model with statistical significance.
However, we note the numerical difference across the datasets is small, as evidenced in Tables~\ref{tab:comparison_ece}~and~\ref{tab:comparison_mce} in the Appendix, 
being of order $\sim$0.005 for ECE and $\sim$0.03 for MCE.
We believe this discrepancy in calibration is likely due to ``model certainty'' evinced by the linearised models, namely output values lying in $\{0, 1\}$.
Notably, the LC datasets get a very large fraction of predictions with certainty ($\sim$80\%) using LAM-ARM1 and LAM-ARM2 models, as shown in Table~\ref{tab:full_certainty} in the Appendix.
The MCE ranks are all fairly close to one another, meaning all of the classifiers are more closely matched on this metric as compared with AUC and ECE.

\section{User Survey}
\label{sec:user_survey}

\begin{figure*}
    \centering
    \includegraphics[height=4.0cm]{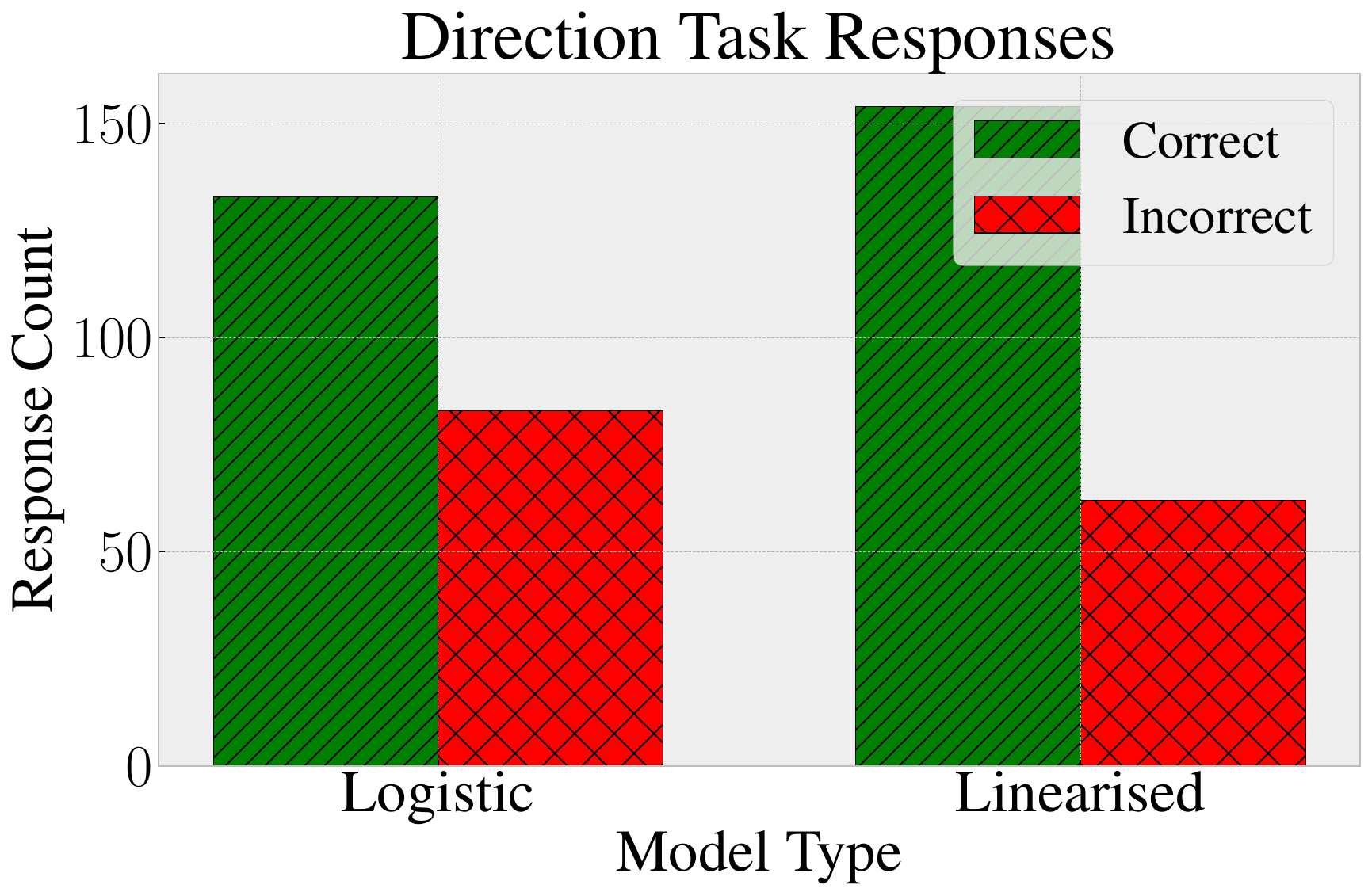}
    \includegraphics[height=4.0cm]{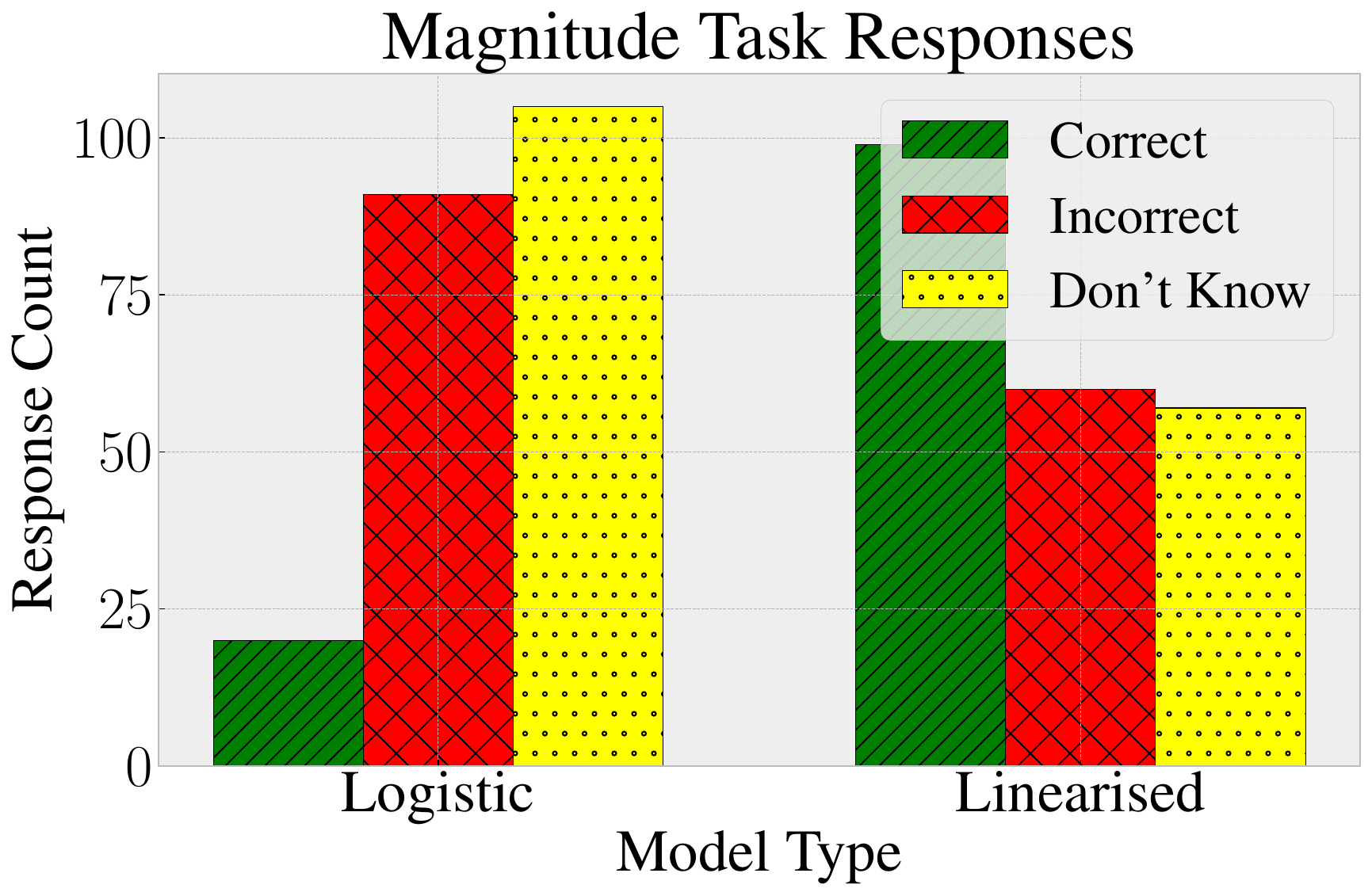}
    \includegraphics[height=4.0cm]{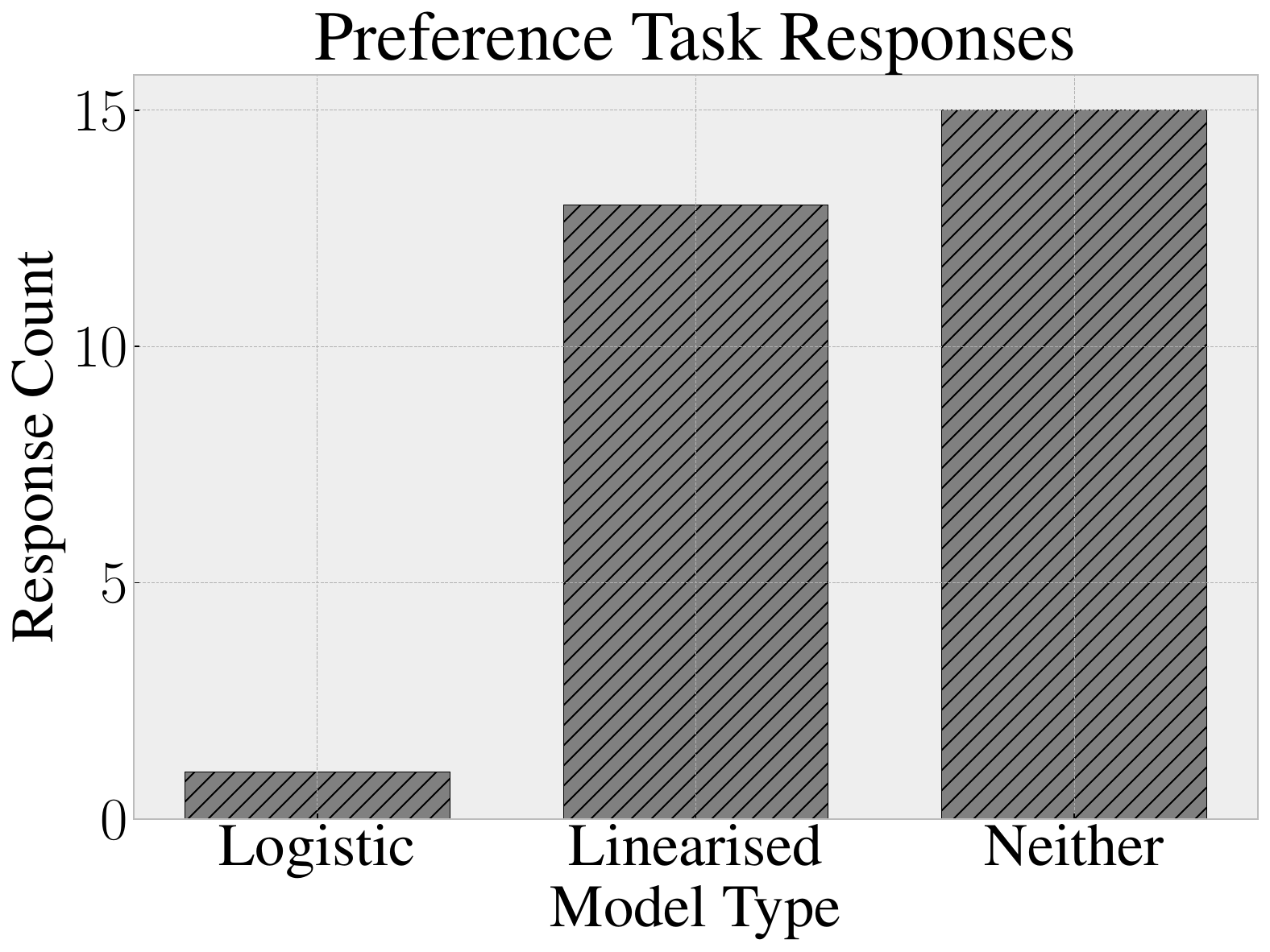}
    \caption{Summary of responses to User Survey comparing interpretability of LR models against LAM-LR. Users predict the change in direction of model output correctly at a slightly higher rate for linearised models (left). When users are asked about the magnitude of this change, users fare overwhelmingly better using LAM-LR as opposed to LR (center). When asked about which model they found easiest to use (right), the majority of users said neither model, with slightly fewer opting for LAM-LR and \emph{only one} for LR. }
    \label{fig:user_study}
\end{figure*}

The motivating example in Section~\ref{sec:motivating_example} suggests that linearised models will be easier to interpret and reason about by users as opposed to logistic models.
To substantiate this, we conduct a user study where the aim is to ascertain which class of models is more interpretable: logistic models or LAMs.
Our proxy for interpretability is how capable users are at carrying out basic reasoning tasks about the models' outputs, given the model coefficients.
This falls within the paradigm of Human-Grounded Evaluation~\cite{InterpretabilityRigorousScience}, where empirically measured human performance on a simplified task serves as a proxy for explanation quality.
Using human simulation of model outputs as a proxy for model interpretability is a similar strategy to that used in the work by~\cite{poursabzi2021manipulating}.

\subsection{Setup}
The study is structured as a questionnaire, wherein participants are shown a small LR model alongside its linearised counterpart and are asked to predict how the output of each model will change in both direction and magnitude from some initial $($input, output$)$ pairs.
Participants are shown several instantiations, which we call scenarios.
More formally, a \emph{scenario} consists of the following elements:
\setlist{nolistsep}
\begin{itemize}[noitemsep,leftmargin=.5cm]
% \begin{itemize}
    \item A tuple of model coefficients $(A_0, A_1, A_2)$.
    \item A tuple of model coefficients $(B_0, B_1, B_2)$.
    \item An input to the models $\vb*{x} := (x_1, x_2)$
    \item The result of applying models $A$ and $B$ to $\vb*{x}$, $\hat{y}_A(\vb*{x}) \in [0, 1]$ and $\hat{y}_B(\vb*{x}) \in [0, 1]$.
    \item A modified input to the models $\vb*{x}' = \vb*{x} + \delta \vb*{e}_m$, where $m \in \{1, 2\}$. Both $\vb*{x}'$ and  $\delta \in \mathbb{R}$ are shown to participants. The feature $m$ being modified is also highlighted.
\end{itemize}
Crucially, \emph{participants are given no indication as to what kind of model the coefficients represent}.
This choice was made so as to not prime participants with any expectation of the models' behaviour~\cite{natesan2016cognitive}.
Nor are participants led to believe the models are related to one another.
In fact, \emph{for all scenarios Model A was the linearisation of Model B}.
Model $B$ is the LR model $\hat{y}_B(\vb*{x}) = \sigma(B_0 + B_1 x_1 + B_2 x_2)$, with the Model $A$ coefficients computed using Definition~\ref{def:linearised_additive_models}.
Upon being shown a particular scenario a user is directed to carry out the following tasks:
\begin{description}[noitemsep,leftmargin=.5cm]
    \item[Direction Task.] For both models $A$ and $B$ give the change of direction in model output, that is, compute  $\operatorname{sign}(\hat{y}_A(\vb*{x}') - \hat{y}_A(\vb*{x}))$ and $ \operatorname{sign}(\hat{y}_B(\vb*{x}') - \hat{y}_B(\vb*{x}))$. The options given to  participants are $\{\textsc{``output increases''}, \textsc{``output decreases''}\}$.
    \item[Magnitude Task.]  For both models $A$ and $B$ give the magnitude of the change in model output, that is, compute  $\abs{\hat{y}_A(\vb*{x}') - \hat{y}_A(\vb*{x})}$ and $ \abs{\hat{y}_B(\vb*{x}') - \hat{y}_B(\vb*{x})}$. The options given to participants are $\{\approx \!0.05, \approx\!0.1, \approx\!0.2, \approx\!0.3, \textsc{``don't know''} \}$.
\end{description}

Respondents are shown six scenarios, not including three training scenarios shown at the beginning of the survey.
The training scenarios show the same input $\vb*{x}$ with progressively increasing (and decreasing) inputs $\vb*{x} + \delta_u \vb{e}_m$ such that $\abs{\delta_1} < \abs{\delta_2} < \cdots$ along with the corresponding outputs, so that the user can learn the behaviour of each model class.
The six test scenarios were designed to be balanced, in the sense that positive and negative changes were included as well as positive and negative values for the $x_m$, $(B_0, B_1, B_2)$, so as to not skew the results.
The scenarios were randomly shuffled three times and each permutation assigned to a different subgroup of participants at random.
As a final task after all scenarios are presented to study participants, we query the following.
\begin{description}[noitemsep,leftmargin=.5cm]
\item[Preference Task.] This is a question asked upon completion of the survey about which class of model was easiest to use. The choices are: $\{ \textsc{``model a''}, \textsc{``model b''}, \textsc{``neither''}\}.$
\end{description}
The full user survey is reproduced in Appendix~\ref{pdf:user_survey_questions}.
The survey was sent to 101 possible respondents via the SurveyMonkey platform~\cite{surveymonkey}, who are AI researchers and practitioners in the authors' firm.
Participation was anonymous and on a voluntary basis.
There were 46 respondents in total, from which 36 successfully completed the training scenarios and gave answers to more than 35\% of scenarios. This was the data that was analysed.
% There were 19 respondents who answered every question.
Despite the small number of participants, the design of the experiment, providing multiple scenarios per-respondent, meant that we can still report results with statistical significance.

\subsection{Results and Analysis}

We summarise the findings of the User Survey in Figure~\ref{fig:user_study}.
In the Direction Task we see that a slightly greater proportion of correct answers are given for the LAM as opposed to the logistic model.
% The difference is small but statistically significant, as we will see in detail later.
In the Magnitude Task we see that the logistic models received a response of ``\textsc{don't know}'' most often, followed by an incorrect response and a small number of correct responses.
The linearised counterparts of these models yielded mostly correct answers, followed by ``\textsc{don't know}'', then incorrect answers.
% The difference is statistically significant
In the Preference Task, the majority of users declared neither model easiest to use, followed closely by the LAM. 
\emph{Only one respondent declared logistic models easiest to use}.

\paragraph{Statistical Analysis.} The Direction and Magnitude Tasks are instances of a \emph{clustered matched-pair binary data} experiment, with the matched pairs being the logistic model and corresponding LAM within each scenario, the binary data comprising a correct vs not correct response to an individual task on each scenario and each cluster corresponding to an individual respondent's answers to multiple scenarios.
We use the statistical test for non-inferiority in clustered matched-pair binary data of~\cite{Yang2012}, which accounts for within cluster correlated responses. 
In our setting this corresponds to an individual's responses possibly being correlated with one another, but independent from the responses of other individuals.
Suppose that $p_{\text{log}}$ is the success probability of a respondent for a given task on a logistic model and $p_{\text{LAM}}$ the corresponding success probability for the linearised version.
The true difference between the two classes of model is $\delta = p_{\text{LAM}} - p_{\text{log}}$.
Choose a small non-inferiority margin\footnote{We choose $\delta_0 = 0.001$} $\delta_0 > 0$.
Then the hypothesis test we are conducting is
$$\text{H}_0 \: : \: p_{\text{LAM}} - p_{\text{log}} \leq \delta_0 ; \qq{vs}  \text{H}_1 \: : \: p_{\text{LAM}} - p_{\text{log}} > \delta_0.$$
\citeauthor{Yang2012}'s $Z_{\text{MO}}$ test statistic asymptotically follows a normal distribution assuming the null hypothesis $\text{H}_0$.
For the Direction Task we observe $Z_{\text{MO}} = 2.822$ corresponding to a $p$-value of $0.0024$, which is significant at the $\alpha = 0.05$ level.
Moreover, the 95\% confidence interval for the difference in success rates between LAMs and logistic models was $\delta \in (0.03, 0.16)$.
This small value of the performance difference in the Direction Task is expected, as for both logistic models and LAMs we can easily inspect the signs of the model coeficients to get the directions, a strategy which many respondents guessed from the training scenarios.
For the Magnitude Task we observe $Z_{\text{MO}} = 4.55$ corresponding to a $p$-value of $2.72 \times 10^{-6}$, which is significant at the $\alpha = 0.05$ level.
The 95\% confidence interval for the difference in success rates between LAMs and logistic models was $\delta \in (0.23, 0.50)$, a significant gap.
We attribute this gap to inherent non-interpretability of reasoning in log-odds space vs reasoning directly in probabilities.

For the Preference Task, the data correspond to matched pairs, each pair belonging to one participant and corresponding to a preference of logistic models vs LAM models, that is, possible responses are $\{A \prec B, A \succ B, A \sim B \}$, with Model $A$ corresponding to LAM and Model $B$ to logistic.  
As there is no numerical or ranked comparison, the usual appropriate statistical test is the Sign Test~\cite{sign_test}.
Given there are many ties $A \sim B$, we use the Trinomial Test of~\cite{BIAN20111153}, specially developed for this regime. See Appendix~\ref{sec:trinomial} for full details).
Let $p_A$ denote the probability a randomly chosen participant prefers LAMs to logistic models and let $p_B$ denote the converse.
Moreover, $p_0$ is the probability neither is preferred.
Then our null hypothesis is $\text{H}_0 : p_A = p_B$ and alternative is  $\text{H}_1 : p_A > p_B$.
Let $N_A$, $N_B$ and $N_0$ be the random variables denoting the observed counts corresponding to Model A preferred, Model B preferred and neither respectively, with $N := N_A + N_B + N_0$.
Assuming $\text{H}_0$, the test statistic $N_d = N_A - N_B$ has critical value at significance $\alpha = 0.05$ of $C_{0.05} = 6$, which is exceeded in our data with $n_d = 12$, corresponding to a $p$-value of 0.0009, which is statistically significant.
This supports there being a user preference for LAMs over logistic models, although the impact of this is somewhat dampened by the comparatively large number of responses preferring neither.

\paragraph{Findings.} We interpret these findings as follows: as measured by performance in basic reasoning tasks about model behaviour, \emph{logistic models are less interpretable than their linearised counterparts}.
This difference is more pronounced when it comes to reasoning about actual numerical model outputs in response to varying input variables as opposed to reasoning just about the general direction of change.
Interestingly, in spite of the gulf in respondents' performance between the LAMs and logistic models this difference is not necessarily felt strongly by the respondents themselves, a large number of whom stated that neither class of models was easier to reason about.
This suggests that when providing explanations for models in production, there should be a mechanism for confirming to consumers of explanations that their interpretation is correct, to increase confidence. 
In this study, consumers of LAM explanations were correctly interpreting them \emph{without even knowing what the underlying model was}, but were not generally confident in their interpretations.
For logistic models, participants did not correctly interpret the explanations and did not declare them easy to use, in spite of their preexisting modelling expertise.

\section{Conclusion}
\label{sec:conclusion}
This work introduces techniques for improving the interpretability characteristics of existing models while incurring only very small penalties in classification performance. 
For an additive logistic model such as ARM, one can use our linearisation scheme (LAM) with the model and incur only a very small reduction in ROC-AUC and a small increase in calibration error.
Via the User Study, we showed that when participants are required to simulate the output of a model, that is, predict its behaviour, their performance was far better with linearised models than their logistic counterparts.

\paragraph{Lessons Learned.}

In this work we closely examined a common tacit assumption within the XAI literature, that there is no reduction in interpretability of GAMs when moving from regression to classification via a non-linear link function (here, the logistic function).
We showed that this assumption in fact does not hold in general, with the commonly used explanation method of sharing model weights or shape functions proving to be misleading to human respondents.
Our LAM construction is shown to mostly overcome this issue while largely preserving an underlying model's behaviour.
The implication here is that when providing model coefficients as explanations, it may be worth paying a small price in performance by linearising a trained logistic model to ensure explanations are correctly understood. 

\paragraph{Limitations and Future Work.}
With this work we must take the following into consideration.
\setlist{nolistsep}
\begin{itemize}[noitemsep,leftmargin=.5cm]
\item{\textbf{Model Certainty.}} 
For a LAM $\hat{y}_{\text{LAM}}$ induced by a logistic model $ \hat{y}$, inputs $\vb*{x}$ such that $\hat{y}(\vb*{x}) \not\in [0.07, 0.93]$ correspond to LAM outputs $\hat{y}_{\text{LAM}}(\vb*{x}) \in \{0, 1\}$.
Risk scores of $\leq 7\%$ and $\geq 93\%$ correspond to confident predictions.
LAMs effectively round these risk scores to certainty -- we show empirically the prevalence of this phenomenon in Table~\ref{tab:full_certainty} (Appendix).
If a difference in risk score between, say, 97\% and 99.99\% is important, then using a LAM may not be appropriate. 
Possible mitigations are: \emph{i.} clip to the interval $[\epsilon, 1 - \epsilon]$ for some small $\epsilon > 0$; or \emph{ii.} increase $\alpha$ in the approximation to the sigmoid, thereby growing the set of model inputs with output in $(0, 1)$, while incurring a penalty in approximation.
% We delegate further investigation of the consequences of these interventions to future work.
\item{\textbf{Alternate Linearisation Schemes.}} 
One could consider alternate methods of linearising as opposed to LAMs. As an example, Average Marginal Effects (AME)~\cite{scholbeck2024marginal,Bartus2005} are local model explanations based on evaluating model prediction function derivatives.
The AME values could potentially be interpreted as coefficients of a clipped linear model, although they would lack the theoretical guarantees afforded by LAM.
% \item{\textbf{Interpretation of model coefficients under model certainty.}} Supposing the LAM model output lies in $\{0, 1\}$, our stated interpretation of $\beta_i / 2\alpha^\star$ requires care, since there will be many $\Delta \vb*{x}$ such that the model output does not change at all.
% Note that in the presence of binary features, logistic models such as ARM face a similar issue.
% Nonetheless, if a user is asking a counterfactual question the model coefficients along with additive structure still admit easy answers to ``what-if'' questions of the form: \emph{what changes in input $\Delta \vb*{x}$ can lead to a risk score $r$?}
% \textcolor{blue}{We hypothesise that extreme model output values do not affect majority of input instances materially. In the Appendix we record proportion of data points receiving risk outputs in $\{0, 1\}$ for all (algorithm, dataset) pairs.}
\item{\textbf{User Study Scope.}} One could increase the scope of the User Study in terms of tasks and information provided to users, considering interpretation of the model coefficients integrated as part of a downstream task, as opposed to predicting model outputs being the tasks' focus.
A potentially fruitful investigation is measuring the effect of linearisation on interpretability of general shape functions in GAMs, since LR models have linear shape functions.
\end{itemize}

{
{
\section{Disclaimer}
This paper was prepared for informational purposes by
the Artificial Intelligence Research group of JPMorgan Chase \& Co. and its affiliates (``JP Morgan''),
and is not a product of the Research Department of JP Morgan.
JP Morgan makes no representation and warranty whatsoever and disclaims all liability,
for the completeness, accuracy or reliability of the information contained herein.
This document is not intended as investment research or investment advice, or a recommendation,
offer or solicitation for the purchase or sale of any security, financial instrument, financial product or service,
or to be used in any way for evaluating the merits of participating in any transaction,
and shall not constitute a solicitation under any jurisdiction or to any person,
if such solicitation under such jurisdiction or to such person would be unlawful.
}
}

% \clearpage

% \clearpage
% \bibliographystyle{alpha}
\bibliographystyle{nicebib-alpha}
{%\small
\bibliography{references_2}

\newcommand{\etalchar}[1]{$^{#1}$}
\begin{thebibliography}{CLR{\etalchar{+}}22}
\providecommand{\url}[1]{\texttt{#1}}
\providecommand{\urlprefix}{}
\expandafter\ifx\csname urlstyle\endcsname\relax
  \providecommand{\doi}[1]{doi:\discretionary{}{}{}#1}\else
  \providecommand{\doi}{doi:\discretionary{}{}{}\begingroup \urlstyle{rm}\Url}\fi
\providecommand{\arxivId}[2][]{\href{https://arxiv.org/abs/#2}{\texttt{arXiv:#2}}}

\bibitem[CG16]{xgboost}
T.~Chen and C.~Guestrin.
\newblock \emph{{XGBoost: A Scalable Tree Boosting System}}.
\newblock In \textsl{{Proceedings of the 22nd ACM SIGKDD International Conference on Knowledge Discovery and Data Mining}}, KDD '16, \textsl{785–794} (2016).

\bibitem[CLR{\etalchar{+}}22]{CHEN2022113647}
C.~Chen \emph{et~al.}
\newblock \emph{{A holistic approach to interpretability in financial lending: Models, visualizations, and summary-explanations}}.
\newblock \textsl{Decision Support Systems}, \textsl{\textbf{152}, 113647} (2022).

\bibitem[{The}20]{sagemath}
{The Sage Developers}.
\newblock \emph{{S}ageMath, the {S}age {M}athematics {S}oftware {S}ystem ({V}ersion 9.2)} (2020).

\end{thebibliography}


\newcommand{\etalchar}[1]{$^{#1}$}
\begin{thebibliography}{AvdWKL20}
\providecommand{\url}[1]{\texttt{#1}}
\providecommand{\urlprefix}{}
\expandafter\ifx\csname urlstyle\endcsname\relax
  \providecommand{\doi}[1]{doi:\discretionary{}{}{}#1}\else
  \providecommand{\doi}{doi:\discretionary{}{}{}\begingroup \urlstyle{rm}\Url}\fi
\providecommand{\arxivId}[2][]{\href{https://arxiv.org/abs/#2}{\texttt{arXiv:#2}}}

\bibitem[AN84]{Aldrich1984}
J.~Aldrich and F.~Nelson.
\newblock \emph{{Linear Probability, Logit, and Probit Models}}.
\newblock {SagePub} (1984).

\bibitem[AvdS19]{NEURIPS2019_567b8f5f}
A.~M. Alaa and M.~van~der Schaar.
\newblock \emph{{Demystifying Black-box Models with Symbolic Metamodels}}.
\newblock In \textsl{{Advances in Neural Information Processing Systems}}, vol.~32 (2019).

\bibitem[AvdWKL20]{COGAM}
A.~Abdul, C.~von~der Weth, M.~Kankanhalli and B.~Y. Lim.
\newblock \emph{{COGAM: Measuring and Moderating Cognitive Load in Machine Learning Model Explanations}}.
\newblock In \textsl{Proceedings of the 2020 CHI Conference on Human Factors in Computing Systems}, CHI '20, \textsl{1–14} (2020).

\bibitem[Bar05]{Bartus2005}
T.~Bartus.
\newblock \emph{Estimation of Marginal Effects using Margeff}.
\newblock \textsl{The Stata Journal}, \textsl{\textbf{5}, 309} (2005).

\bibitem[BCM16]{JMLR:v17:benavoli16a}
A.~Benavoli, G.~Corani and F.~Mangili.
\newblock \emph{{Should We Really Use Post-Hoc Tests Based on Mean-Ranks?}}
\newblock \textsl{Journal of Machine Learning Research}, \textsl{\textbf{17}, 1} (2016).

\bibitem[BMW11]{BIAN20111153}
G.~Bian, M.~McAleer and W.-K. Wong.
\newblock \emph{A trinomial test for paired data when there are many ties}.
\newblock \textsl{Mathematics and Computers in Simulation}, \textsl{\textbf{81}, 1153} (2011).

\bibitem[Bra97]{BRADLEY19971145}
A.~P. Bradley.
\newblock \emph{{The use of the area under the ROC curve in the evaluation of machine learning algorithms}}.
\newblock \textsl{Pattern Recognition}, \textsl{\textbf{30}, 1145} (1997).

\bibitem[BXS{\etalchar{+}}20]{XAIDeployment}
U.~Bhatt \emph{et~al.}
\newblock \emph{Explainable Machine Learning in Deployment}.
\newblock In \textsl{Proceedings of the 2020 Conference on Fairness, Accountability, and Transparency}, FAT* '20, \textsl{648–657} (2020).

\bibitem[CG16]{xgboost}
T.~Chen and C.~Guestrin.
\newblock \emph{{XGBoost: A Scalable Tree Boosting System}}.
\newblock In \textsl{{Proceedings of the 22nd ACM SIGKDD International Conference on Knowledge Discovery and Data Mining}}, KDD '16, \textsl{785–794} (2016).

\bibitem[CGBH{\etalchar{+}}18]{Chen2018}
H.~Chen \emph{et~al.}
\newblock \emph{{Logistic regression over encrypted data from fully homomorphic encryption}}.
\newblock \textsl{BMC Medical Genomics}, \textsl{\textbf{11}, 81} (2018).

\bibitem[CLR{\etalchar{+}}22]{CHEN2022113647}
C.~Chen \emph{et~al.}
\newblock \emph{{A holistic approach to interpretability in financial lending: Models, visualizations, and summary-explanations}}.
\newblock \textsl{Decision Support Systems}, \textsl{\textbf{152}, 113647} (2022).

\bibitem[Com21]{EUAIAct}
E.~Commission.
\newblock \href{https://eur-lex.europa.eu/legal-content/EN/TXT/?uri=CELEX:52021PC0206}{\emph{{Proposal for a Regulation of the European Parliament and of the Council laying down harmonised rules on artificial intelligence (Artificial Intelligence Act) and amending certain Union legislative acts.}}} (2021).
\newblock \urlprefix\url{https://eur-lex.europa.eu/legal-content/EN/TXT/?uri=CELEX:52021PC0206}.

\bibitem[CSH{\etalchar{+}}22]{chen2022does}
Z.~Chen \emph{et~al.}
\newblock \emph{Does the explanation satisfy your needs?: A unified view of properties of explanations}.
\newblock \textsl{arXiv preprint arXiv:2211.05667} (2022).

\bibitem[CW20]{Campelo2020}
F.~Campelo and E.~F. Wanner.
\newblock \emph{{Sample size calculations for the experimental comparison of multiple algorithms on multiple problem instances}}.
\newblock \textsl{Journal of Heuristics}, \textsl{\textbf{26}, 851} (2020).

\bibitem[DBDC21]{de2021spline}
K.~W. De~Bock and A.~De~Caigny.
\newblock \emph{{Spline-rule ensemble classifiers with structured sparsity regularization for interpretable customer churn modeling}}.
\newblock \textsl{Decision Support Systems}, \textsl{\textsl{113523}} (2021).

\bibitem[DBDR20]{Dziugaite2020}
G.~K. Dziugaite, S.~Ben-David and D.~M. Roy.
\newblock \emph{{Enforcing Interpretability and its Statistical Impacts: Trade-offs between Accuracy and Interpretability}}.
\newblock \textsl{arXiv preprint arXiv:2010.13764} (2020).

\bibitem[Dem06]{Demsar2006}
J.~Dem\v{s}ar.
\newblock \emph{{Statistical Comparisons of Classifiers over Multiple Data Sets}}.
\newblock \textsl{J. Mach. Learn. Res.}, \textsl{\textbf{7}, 1–30} (2006).

\bibitem[DF83]{DeGroot1983}
M.~H. DeGroot and S.~E. Fienberg.
\newblock \emph{{The Comparison and Evaluation of Forecasters}}.
\newblock \textsl{Journal of the Royal Statistical Society. Series D (The Statistician)}, \textsl{\textbf{32}, 12} (1983).

\bibitem[DGW18]{Sanjeeb2018}
S.~Dash, O.~Gunluk and D.~Wei.
\newblock \emph{{Boolean Decision Rules via Column Generation}}.
\newblock In \textsl{{Advances in Neural Information Processing Systems (NIPS 2018)}}, \textsl{4655--4665} (2018).

\bibitem[DM46]{sign_test}
W.~J. Dixon and A.~M. Mood.
\newblock \emph{The Statistical Sign Test}.
\newblock \textsl{Journal of the American Statistical Association}, \textsl{\textbf{41}, 557} (1946).

\bibitem[DVK17]{InterpretabilityRigorousScience}
F.~Doshi-Velez and B.~Kim.
\newblock \href{http://dx.doi.org/10.48550/ARXIV.1702.08608}{\emph{Towards A Rigorous Science of Interpretable Machine Learning}}.
\newblock \href{http://dx.doi.org/10.48550/ARXIV.1702.08608}{\textsl{arXiv preprint arXiv:1702.08608}} (2017).

\bibitem[FIC18]{heloc}
FICO.
\newblock \emph{{FICO xML Challenge found at \url{community.fico.com/s/xml}}} (2018).

\bibitem[Fri40]{Friedman1940}
M.~Friedman.
\newblock \emph{{A Comparison of Alternative Tests of Significance for the Problem of $m$ Rankings}}.
\newblock \textsl{The Annals of Mathematical Statistics}, \textsl{\textbf{11}, 86 } (1940).

\bibitem[GH08]{JMLR:v9:garcia08a}
S.~Garc{{\'i}}a and F.~Herrera.
\newblock \emph{{An Extension on ``Statistical Comparisons of Classifiers over Multiple Data Sets'' for all Pairwise Comparisons}}.
\newblock \textsl{Journal of Machine Learning Research}, \textsl{\textbf{9}, 2677} (2008).

\bibitem[GTD{\etalchar{+}}23]{gkolemis2023regionally}
V.~Gkolemis \emph{et~al.}
\newblock \emph{Regionally Additive Models: Explainable-by-design models minimizing feature interactions} (2023).

\bibitem[Gun19]{DARPA}
D.~Gunning.
\newblock \emph{{DARPA's Explainable Artificial Intelligence (XAI) Program}}.
\newblock In \textsl{{Proceedings of the 24th International Conference on Intelligent User Interfaces}}, IUI '19, \textsl{ii} (2019).

\bibitem[Har17]{harris}
N.~Harris.
\newblock \href{https://www.naftaliharris.com/blog/logistic-regression-uninterpretable/}{\emph{Logistic regression isn't interpretable}} (2017).
\newblock \urlprefix\url{https://www.naftaliharris.com/blog/logistic-regression-uninterpretable/}.

\bibitem[Hel09]{Hellevik2009-cw}
O.~Hellevik.
\newblock \emph{Linear versus logistic regression when the dependent variable is a dichotomy}.
\newblock \textsl{Quality \& Quantity}, \textsl{\textbf{43}, 59} (2009).

\bibitem[HGLV22]{halliwell2022need}
N.~Halliwell, F.~Gandon, F.~Lecue and S.~Villata.
\newblock \emph{The Need for Empirical Evaluation of Explanation Quality}.
\newblock In \textsl{AAAI: Explainable Agency in Artificial Intelligence Workshop} (2022).

\bibitem[HM83]{Hanley1983}
J.~A. Hanley and B.~J. McNeil.
\newblock \emph{{A method of comparing the areas under receiver operating characteristic curves derived from the same cases.}}
\newblock \textsl{Radiology}, \textsl{\textbf{148}, 839} (1983).

\bibitem[Hof94]{misc_statlog_(german_credit_data)_144}
H.~Hofmann.
\newblock \emph{{Statlog (German Credit Data)}}.
\newblock UCI Machine Learning Repository (1994).

\bibitem[Hol79]{Holm1979}
S.~Holm.
\newblock \emph{{A Simple Sequentially Rejective Multiple Test Procedure}}.
\newblock \textsl{Scandinavian Journal of Statistics}, \textsl{\textbf{6}, 65} (1979).

\bibitem[HTF09]{hastie2009elements}
T.~Hastie, R.~Tibshirani and J.~Friedman.
\newblock \emph{{The Elements of Statistical Learning: Data Mining, Inference, and Prediction}}.
\newblock Springer (2009).

\bibitem[ID80]{Iman1980}
R.~L. Iman and J.~M. Davenport.
\newblock \emph{Approximations of the critical region of the friedman statistic}.
\newblock \textsl{Communications in Statistics - Theory and Methods}, \textsl{\textbf{9}, 571} (1980).

\bibitem[Inc22]{surveymonkey}
M.~Inc.
\newblock \href{www.momentive.ai}{\emph{SurveyMonkey}} (2022).
\newblock \urlprefix\url{www.momentive.ai}.

\bibitem[{Kag}19]{lending-club}
{Kaggle}.
\newblock \href{https://www.kaggle.com/datasets/wordsforthewise/lending-club}{\emph{{Lending Club Data}}} (2019).
\newblock \urlprefix\url{https://www.kaggle.com/datasets/wordsforthewise/lending-club}.

\bibitem[KF19]{kraus2019forecasting}
M.~Kraus and S.~Feuerriegel.
\newblock \emph{{Forecasting remaining useful life: Interpretable deep learning approach via variational Bayesian inferences}}.
\newblock \textsl{Decision Support Systems}, \textsl{\textbf{125}, 113100} (2019).

\bibitem[KLN]{UCI}
M.~Kelly, R.~Longjohn and K.~Nottingham.
\newblock \href{https://archive.ics.uci.edu}{\emph{{UCI} Machine Learning Repository}}.
\newblock \urlprefix\url{https://archive.ics.uci.edu}.

\bibitem[KNJ{\etalchar{+}}20]{InterpretingInterpretability}
H.~Kaur \emph{et~al.}
\newblock \emph{Interpreting Interpretability: Understanding Data Scientists' Use of Interpretability Tools for Machine Learning}.
\newblock In \textsl{Proceedings of the 2020 CHI Conference on Human Factors in Computing Systems}, CHI '20, \textsl{1–14} (2020).

\bibitem[Kra88]{kraft1988software}
D.~Kraft.
\newblock \emph{A software package for sequential quadratic programming}.
\newblock Deutsche Forschungs- und Versuchsanstalt f{\"u}r Luft- und Raumfahrt K{\"o}ln: Forschungsbericht. Wiss. Berichtswesen d. DFVLR (1988).

\bibitem[Kra94]{Kraft1994}
D.~Kraft.
\newblock \emph{{Algorithm 733: TOMP–Fortran Modules for Optimal Control Calculations}}.
\newblock \textsl{ACM Trans. Math. Softw.}, \textsl{\textbf{20}, 262–281} (1994).

\bibitem[KSB{\etalchar{+}}13]{Kulesza2013}
T.~Kulesza \emph{et~al.}
\newblock \emph{Too much, too little, or just right? Ways explanations impact end users' mental models}.
\newblock In \textsl{2013 IEEE Symposium on Visual Languages and Human Centric Computing}, \textsl{3--10} (2013).

\bibitem[KSK{\etalchar{+}}18]{Kim2018b}
A.~Kim \emph{et~al.}
\newblock \emph{Logistic regression model training based on the approximate homomorphic encryption}.
\newblock \textsl{BMC Medical Genomics}, \textsl{\textbf{11}, 83} (2018).

\bibitem[KSW{\etalchar{+}}18]{Kim2019}
M.~Kim \emph{et~al.}
\newblock \emph{{Secure Logistic Regression Based on Homomorphic Encryption: Design and Evaluation}}.
\newblock \textsl{JMIR Med Inform}, \textsl{\textbf{6}, e19} (2018).

\bibitem[LCGH13]{EBM}
Y.~Lou, R.~Caruana, J.~Gehrke and G.~Hooker.
\newblock \emph{Accurate intelligible models with pairwise interactions}.
\newblock In \textsl{Proceedings of the 19th ACM SIGKDD International Conference on Knowledge Discovery and Data Mining}, KDD '13, \textsl{623–631} (2013).

\bibitem[Lip18]{Lipton2018}
Z.~C. Lipton.
\newblock \emph{{The Mythos of Model Interpretability: In Machine Learning, the Concept of Interpretability is Both Important and Slippery.}}
\newblock \textsl{ACM Queue}, \textsl{\textbf{16}, 31–57} (2018).

\bibitem[LLTS20]{misc_taiwanese_bankruptcy_prediction_572}
D.~Liang, C.-C. Lu, C.-F. Tsai and G.-A. Shih.
\newblock \emph{{Taiwanese Bankruptcy Prediction}}.
\newblock UCI Machine Learning Repository (2020).

\bibitem[Mol22]{molnar2022}
C.~Molnar.
\newblock \href{https://christophm.github.io/interpretable-ml-book}{\emph{{Interpretable Machine Learning}}}.
\newblock 2\textsuperscript{nd} ed. (2022).

\bibitem[NCH{\etalchar{+}}18]{narayanan2018humans}
M.~Narayanan \emph{et~al.}
\newblock \emph{How do humans understand explanations from machine learning systems? an evaluation of the human-interpretability of explanation}.
\newblock \textsl{arXiv preprint arXiv:1802.00682} (2018).

\bibitem[Ng11]{NgLogReg}
A.~Ng.
\newblock \href{https://www.coursera.org/lecture/machine-learning/classification-wlPeP}{\emph{{Logistic Regression: Classification}}} (2011).
\newblock \urlprefix\url{https://www.coursera.org/lecture/machine-learning/classification-wlPeP}.

\bibitem[NMC05]{Niculescu_Mizil2005}
A.~Niculescu-Mizil and R.~Caruana.
\newblock \emph{{Predicting Good Probabilities with Supervised Learning}}.
\newblock In \textsl{Proceedings of the 22nd International Conference on Machine Learning}, ICML '05, \textsl{625–632}. Association for Computing Machinery, New York, NY, USA (2005).

\bibitem[NWC16]{natesan2016cognitive}
D.~Natesan, M.~Walker and S.~Clark.
\newblock \emph{Cognitive bias in usability testing}.
\newblock In \textsl{Proceedings of the International Symposium on Human Factors and Ergonomics in Health Care}, vol.~5, \textsl{86--88}. SAGE Publications Sage CA (2016).

\bibitem[OCC21]{OCC2021}
OCC.
\newblock \href{https://www.occ.treas.gov/publications-and-resources/publications/comptrollers-handbook/files/model-risk-management/index-model-risk-management.html}{\emph{{Comptroller’s Handbook on Model Risk Management (by US Office of the Comptroller of the Currency)}}} (2021).
\newblock \urlprefix\url{https://www.occ.treas.gov/publications-and-resources/publications/comptrollers-handbook/files/model-risk-management/index-model-risk-management.html}.

\bibitem[PSGH{\etalchar{+}}21]{poursabzi2021manipulating}
F.~Poursabzi-Sangdeh \emph{et~al.}
\newblock \emph{Manipulating and measuring model interpretability}.
\newblock In \textsl{Proceedings of the 2021 CHI conference on human factors in computing systems}, \textsl{1--52} (2021).

\bibitem[PVG{\etalchar{+}}11]{scikit-learn}
F.~Pedregosa \emph{et~al.}
\newblock \emph{Scikit-learn: Machine Learning in {P}ython}.
\newblock \textsl{Journal of Machine Learning Research}, \textsl{\textbf{12}, 2825} (2011).

\bibitem[Qui87]{misc_statlog_(australian_credit_approval)_143}
R.~Quinlan.
\newblock \emph{{Statlog (Australian Credit Approval)}}.
\newblock UCI Machine Learning Repository (1987).

\bibitem[Reg21]{RFI2021}
U.~F. Register.
\newblock \href{https://www.federalregister.gov/documents/2021/03/31/2021-06607/request-for-information-and-comment-on-financial-institutions-use-of-artificial-intelligence}{\emph{{Request for information and comment on financial institutions' use of artificial intelligence, including machine learning (by US Agencies)}}} (2021).
\newblock \urlprefix\url{https://www.federalregister.gov/documents/2021/03/31/2021-06607/request-for-information-and-comment-on-financial-institutions-use-of-artificial-intelligence}.

\bibitem[Res11]{SR11-7}
U.~F. Reserve.
\newblock \href{https://www.federalreserve.gov/supervisionreg/srletters/sr1107a1.pdf}{\emph{{SR 11-7/OCC11-12: Supervisory Guidance on Model Risk Management (by Federal Reserve Board and Office of the Comptroller of the Currency)}}} (2011).
\newblock \urlprefix\url{https://www.federalreserve.gov/supervisionreg/srletters/sr1107a1.pdf}.

\bibitem[RLN{\etalchar{+}}22]{rong2022towards}
Y.~Rong \emph{et~al.}
\newblock \emph{Towards Human-centered Explainable AI: User Studies for Model Explanations}.
\newblock \textsl{arXiv preprint arXiv:2210.11584} (2022).

\bibitem[Rud19]{Rudin19}
C.~Rudin.
\newblock \emph{{Stop Explaining Black Box Machine Learning Models for High Stakes Decisions and Use Interpretable Models Instead}}.
\newblock \textsl{Nature Machine Intelligence}, \textsl{\textbf{1}, 206} (2019).

\bibitem[San92]{misc_japanese_credit_screening_28}
C.~Sano.
\newblock \emph{{Japanese Credit Screening}}.
\newblock UCI Machine Learning Repository (1992).

\bibitem[SBRZ{\etalchar{+}}21]{Shaikhina2021}
T.~Shaikhina \emph{et~al.}
\newblock \emph{{Effects of Uncertainty on the Quality of Feature Importance Estimates}}.
\newblock In \textsl{AAAI Workshop on Explainable Agency in AI}. AAAI (2021).

\bibitem[SCM{\etalchar{+}}24]{scholbeck2024marginal}
C.~A. Scholbeck \emph{et~al.}
\newblock \emph{Marginal effects for non-linear prediction functions}.
\newblock \textsl{Data Mining and Knowledge Discovery}, \textsl{\textsl{1--46}} (2024).

\bibitem[SHJ{\etalchar{+}}20]{SlackEtAl20}
D.~Slack \emph{et~al.}
\newblock \emph{Fooling {LIME} and {SHAP}: Adversarial Attacks on Post hoc Explanation Methods}.
\newblock In \textsl{{AAAI}/{ACM} Conference on Artificial Intelligence, Ethics, and Society ({AIES})} (2020).

\bibitem[SZ21]{wells2021}
A.~Sudjianto and A.~Zhang.
\newblock \emph{{Designing Inherently Interpretable Machine Learning Models}}.
\newblock \textsl{arXiv preprint arXiv:2111.01743} (2021).

\bibitem[Tom16]{misc_polish_companies_bankruptcy_365}
S.~Tomczak.
\newblock \emph{{Polish Companies Bankruptcy}}.
\newblock UCI Machine Learning Repository (2016).

\bibitem[VGO{\etalchar{+}}20]{2020SciPy-NMeth}
P.~Virtanen \emph{et~al.}
\newblock \emph{{{SciPy} 1.0: Fundamental Algorithms for Scientific Computing in Python}}.
\newblock \textsl{Nature Methods}, \textsl{\textbf{17}, 261} (2020).

\bibitem[vH15]{vonHippel}
P.~von Hippel.
\newblock \href{https://statisticalhorizons.com/linear-vs-logistic/}{\emph{{Linear vs. Logistic Probability Models: Which is Better, and When?}}} (2015).
\newblock \urlprefix\url{https://statisticalhorizons.com/linear-vs-logistic/}.

\bibitem[VS20]{pmlr-v119-vidal20a}
T.~Vidal and M.~Schiffer.
\newblock \emph{{Born-Again Tree Ensembles}}.
\newblock In H.~D. III and A.~Singh, editors, \textsl{{Proceedings of the 37th International Conference on Machine Learning}}, vol. 119, \textsl{9743--9753} (2020).

\bibitem[VSB{\etalchar{+}}18]{vaughan2018explainable}
J.~Vaughan \emph{et~al.}
\newblock \emph{{Explainable Neural Networks based on Additive Index Models}}.
\newblock \textsl{arXiv preprint arXiv:1806.01933} (2018).

\bibitem[VVKB18]{veale2018fairness}
M.~Veale, M.~Van~Kleek and R.~Binns.
\newblock \emph{Fairness and accountability design needs for algorithmic support in high-stakes public sector decision-making}.
\newblock In \textsl{Proceedings of the 2018 {CHI} conference on human factors in computing systems}, \textsl{1--14} (2018).

\bibitem[Wil45]{Wilcoxon1945}
F.~Wilcoxon.
\newblock \emph{{Individual Comparisons by Ranking Methods}}.
\newblock \textsl{Biometrics Bulletin}, \textsl{\textbf{1}, 80} (1945).

\bibitem[Wil22]{WILCOX202245}
R.~R. Wilcox.
\newblock \emph{{Chapter 3 - Estimating Measures of Location and Scale}}.
\newblock In \textsl{{Introduction to Robust Estimation and Hypothesis Testing}}, \textsl{45--106}. Academic Press, {5\textsuperscript{th}} ed. (2022).

\bibitem[YSH12]{Yang2012}
Z.~Yang, X.~Sun and J.~W. Hardin.
\newblock \emph{{Testing non-inferiority for clustered matched-pair binary data in diagnostic medicine}}.
\newblock \textsl{Computational Statistics \& Data Analysis}, \textsl{\textbf{56}, 1301} (2012).

\bibitem[YZS21]{ExNN}
Z.~Yang, A.~Zhang and A.~Sudjianto.
\newblock \emph{{Enhancing Explainability of Neural Networks Through Architecture Constraints}}.
\newblock \textsl{IEEE Transactions on Neural Networks and Learning Systems}, \textsl{\textbf{32}, 2610} (2021).

\bibitem[ZBRS22]{zhou2022feature}
Y.~Zhou, S.~Booth, M.~T. Ribeiro and J.~Shah.
\newblock \emph{{Do Feature Attribution Methods Correctly Attribute Features?}}
\newblock In \textsl{{Proceedings of the 36th AAAI Conference on Artificial Intelligence}}. AAAI (2022).

\bibitem[ZTT16]{zikeba2016ensemble}
M.~Zi{k{e}}ba, S.~K. Tomczak and J.~M. Tomczak.
\newblock \emph{Ensemble Boosted Trees with Synthetic Features Generation in Application to Bankruptcy Prediction}.
\newblock \textsl{Expert Systems with Applications} (2016).

\end{thebibliography}
}

% %%
% %% If your work has an appendix, this is the place to put it.
\clearpage
\appendix
\onecolumn
\section{LAM Derivation and Proofs}\label{sec:lam_derivation}

% We shall introduce a piecewise linear approximation to logistic regression that we call Linearised Additive Models (LAM) that yields a positive answer to this question.

% Rearranging Definition~\ref{def:logistic_additive_model} gives
% \begin{equation}
% \log\qty(\frac{\hat{y}(\vb*{x})}{1 - \hat{y}(\vb*{x})}) = \beta_0 + \sum_{i=1}^{d} \beta_i f_i(x_i).
% \end{equation}
% Then exponentiate
% \begin{equation}
% \frac{\hat{y}(\vb*{x})}{1 - \hat{y}(\vb*{x})} = \exp(\beta_0) \cdot \prod_{i=1}^{d} \exp(\beta_i)^{ f_i(x_i)}
% \end{equation}

\subsection{Approximating the Sigmoid.}

\begin{definition}[Piecewise Linear Functions]
A piecewise linear function on one variable $f(x)$ with $n$ pieces determined by the $n + 1$ points $(x_1, y_1), \ldots, (x_{n + 1}, y_{n + 1})$ is given by 
\begin{equation}
    f_{(x_1, y_1), \ldots, (x_{n + 1}, y_{n + 1})}(x) = 
    \begin{cases}
        m_1 x + c_1, & x \leq x_2 ; \\
        m_2 x + c_2, & x_2 < x \leq x_3; \\ 
        \vdots \\
        m_{n} x + c_{n}, &  x >  x_n,
    \end{cases}
\end{equation}
where $m_k = \frac{y_{k + 1} - y_{k}}{x_{k + 1} - x_{k}}$ and $c_k = y_k - m_k x_k $ and we demand $x_1 < x_2 < \cdots < x_{n+1}$.

\begin{definition}[$\mathrm{PL}_n$] The space of piecewise linear functions comprised of $n$ pieces, $\mathrm{PL}_n$, is defined as
$$\mathrm{PL}_n = \left\{ f_{(x_1, y_1), \ldots, (x_{n + 1}, y_{n + 1})} \,\middle\vert\,  y_1, \ldots, y_{n+1} \in \mathbb{R},\ -\infty< x_1 < x_2 < \cdots < x_{n+1} < \infty\right\}.$$
\end{definition}

\end{definition}

For a function $f(x)$ that is to be approximated by $\widetilde{f}(x)$, the \emph{squared error} is defined as
$\operatorname{SE} = \int_{-\infty}^{+\infty} (\widetilde{f}(x) - f(x))^2 \dd x$.
The squared error is finite when $\widetilde{f}(x) - f(x)$ is square integrable.

\begin{proposition}\label{prop:opt_approx_sigmoid}
   Let $\piecelin$ be the space of 3-piece piecewise linear functions of one variable and consider the family of functions in $\piecelin$ parameterised by $\alpha > 0$
   \begin{equation}\label{eq:clip_lin_approx}
       \widetilde{\sigma}(x; \alpha) = 
    \begin{cases}
        0, & x \in (-\infty, -\alpha) ; \\
        \frac{1}{2} + \frac{x}{2\alpha}, & x \in [-\alpha, +\alpha] ; \\
        1, &  x \in (\alpha, +\infty).
    \end{cases}
   \end{equation}
   Then $\widetilde{\sigma}(\,\cdot\,; \alpha^\star)$ with $\alpha^\star \approx 2.5996 \approx \frac{80000}{30773}$ is the squared-error optimal approximator to the logistic sigmoid $\sigma(z) = (1 + e^{-z})^{-1}$, that is
    $$\widetilde{\sigma}(\,\cdot\,; \alpha^\star) = \arg\min_{\widetilde{\sigma} \in \piecelin} \qty{ \int_{-\infty}^{+\infty} \qty(\widetilde{\sigma}(z) - \sigma(z))^2 \dd z } .$$
\end{proposition}
\begin{proof}
The proof proceeds in two parts. We first show via symmetry arguments that the minimising approximator belongs to a one-parameter family of approximate functions, then we minimise squared error as a function of the parameter.

Consider the leftmost piece of the approximate function $\widetilde{\sigma}$, which takes the values $m_1 z + c_1$. 
Then we have $ \qty(\widetilde{\sigma}(z) - \sigma(z))^2 \sim (m_1 z + c_1)^2$ as $z \to -\infty$ and so the error integral diverges if $m_1 \neq 0$, that is, the first piece is a constant function $c_1$.
Moreover, suppose $m_1 = 0$, $c_1 \neq 0$.
Then the error integral will also diverge.
Thus the first piece of the optimal approximator $\widetilde{\sigma}$ is the constant function $z \mapsto 0$.
A similar argument considering the limit $z \to \infty$ demands that the third (rightmost) piece $m_3 z + c_3$ has $m_3 = 0$, $c_3 = 1$.

It remains to determine the middle line segment, with $(x, y)$-coordinates $(x_2, 0)$, $(x_3, 1)$.
We can see that for the optimal approximator we require $x_2 = - x_3$ by the following: we can reparametrise the middle line segment by its center $c = \frac{x_2 + x_3}{2}$ and its halfwidth $\alpha = \frac{x_3 - x_2}{2}$.
For a fixed halfwidth $w$ we can consider varying the center $c$ in order to minimise the error.
Since $\sigma(z)$ has $180^\circ$ rotational symmetry about the point $(0, \frac{1}{2})$, all else held equal the optimal approximator should also satisfy this same symmetry as the error integral will penalise violations of this symmetry.
The center $c$ is the only variable parameter remaining and the only center that gives rise to an approximator satisfying this symmetry given a fixed width is $c = 0$.

Thus our optimal approximator $\widetilde{\sigma} \in \piecelin$ has the functional form~\eqref{eq:clip_lin_approx} parameterised by $\alpha > 0$.
We refer to this function as $\widetilde{\sigma}(x; \alpha)$ from now on.
In Fig.~\ref{fig:opt_sig_approx} (right) we see an example of $\widetilde{\sigma}(x, \alpha^\star)$ with the optimally chosen value of $\alpha$.
Toward the goal of minimising with respect to $\alpha$, we compute the error integral as a function of $\alpha$
\begin{equation*}
\begin{aligned}
\operatorname{SE}(\alpha) &= 
    \int_{-\infty}^{+\infty} (\widetilde{\sigma}(z; \alpha) - \sigma(z))^2 \dd z \\
    &= \int_{-\infty}^{-\alpha} \qty(\sigma(z))^2 \, \dd z  + \int_{-\alpha}^{+\alpha} \left(\frac{1}{2}  + \frac{z}{2\alpha} - \sigma(z)\right)^2 \, \dd z + \int_{+\alpha}^{+\infty} \qty(1 - \sigma(z))^2 \, \dd z \\
    &= -\frac{1}{3\alpha} \bigg(7 \, {\alpha}^{2} + 6 \, {\alpha} \log {\left(1 + e^{{-\alpha}}\right)} - 6 \, {\alpha} \log\left(e^{{\alpha}} + 1\right) + 3 \, {\alpha} - 3 \, {\rm Li}_2\left(-e^{-{\alpha}}\right) + 3 \, {\rm Li}_2\left(-e^{{\alpha}}\right) \bigg),
    \end{aligned}
\end{equation*}
where ${\rm Li}_2$ is Spence's dilogarithm function, defined as ${\rm Li}_2(z) := -\int^z_0 \frac{\ln(1-u)}{u} \dd u$ for $z \in \mathbb{C}\setminus [1, \infty)$.
The closed form of $\operatorname{SE}(\alpha)$ was obtained using the SageMath system~\citeSM{sagemath}. 
The function $\operatorname{SE}(\alpha) $ is plotted in Fig.~\ref{fig:opt_sig_approx} (left) for $\alpha > 0$, is convex and takes its minimum at $\alpha^\star \approx 2.5996 \approx \frac{80000}{30773}$, with this minimum obtained via Newton’s method\footnote{We were unable solve this minimisation in closed form.}.
Thus the function $\widetilde{\sigma}(z; \alpha^\star)$ is the minimiser of squared error in $\piecelin$ and the result is proved.
\end{proof}

\begin{figure}
    \centering
    \begin{subfigure}[b]{0.46\textwidth}
        \centering
        \resizebox{\columnwidth}{!}{%
            % This file was created with tikzplotlib v0.9.16.
\begin{tikzpicture}[font=\LARGE]

\begin{axis}[
clip=false,
height=8cm,
legend cell align={left},
legend style={
  fill opacity=0.8,
  draw opacity=1,
  text opacity=1,
  at={(0.97,0.03)},
  anchor=south east,
  draw=white!80!black
},
tick align=outside,
tick pos=left,
width=12cm,
x grid style={white!69.0196078431373!black},
xlabel={\(\displaystyle \alpha\)},
xmin=-0.49895, xmax=10.49995,
xtick style={color=black},
y grid style={white!69.0196078431373!black},
ylabel={Squared Error},
ymin=-0.0311610017733549, ymax=0.872213469432495,
ytick style={color=black}
]
\addplot [semithick, black, forget plot]
table {%
0.001 0.385961111120247
0.0512462311557789 0.369431117718733
0.101492462311558 0.35332171273151
0.151738693467337 0.337632412652916
0.201984924623116 0.322362417802712
0.252231155778894 0.307510615191842
0.302477386934673 0.293075582492518
0.352723618090452 0.279055593079707
0.402969849246231 0.265448622102982
0.45321608040201 0.252252353539933
0.503462311557789 0.23946418817513
0.553708542713568 0.227081252442394
0.603954773869347 0.215100408062531
0.654201005025126 0.203518262404262
0.704447236180905 0.192331179492294
0.754693467336684 0.181535291583966
0.804939698492462 0.171126511234094
0.855185929648241 0.161100543766891
0.90543216080402 0.151452900074015
0.955678391959799 0.142178909658688
1.00592462311558 0.133273733847655
1.05617085427136 0.124732379095265
1.10641708542714 0.116549710306954
1.15666331658291 0.108720464113248
1.20690954773869 0.101239262029429
1.25715577889447 0.0941006234405752
1.30740201005025 0.0872989783565181
1.35764824120603 0.0808286798862533
1.40789447236181 0.0746840163865079
1.45814070351759 0.0688592232443939
1.50838693467337 0.0633484942592246
1.55863316582915 0.0581459925937176
1.60887939698492 0.0532458612697599
1.6591256281407 0.0486422331887337
1.70937185929648 0.0443292406609463
1.75961809045226 0.0403010244330563
1.80986432160804 0.0365517422064123
1.86011055276382 0.0330755766429967
1.9103567839196 0.0298667428590804
1.96060301507538 0.0269194954098286
2.01084924623116 0.0242281347709406
2.06109547738693 0.0217870133258369
2.11134170854271 0.0195905408691736
2.16158793969849 0.0176331896393115
2.21183417085427 0.0159094988939936
2.26208040201005 0.0144140790448647
2.31232663316583 0.0131416153674647
2.36257286432161 0.0120868713043507
2.41281909547739 0.0112446913794512
2.46306532663317 0.0106100037423955
2.51331155778894 0.0101778223616607
2.56355778894472 0.0099432488856377
2.6138040201005 0.00990147419054735
2.66405025125628 0.0100477796341037
2.71429648241206 0.0103775380334353
2.76454271356784 0.0108862143855143
2.81478894472362 0.0115693663478071
2.8650351758794 0.0124226444964397
2.91528140703518 0.0134417923785184
2.96552763819095 0.0146226463747096
3.01577386934673 0.0159611353874945
3.06602010050251 0.0174532803698515
3.11626633165829 0.0190951937084941
3.16651256281407 0.0208830784749807
3.21675879396985 0.022813227557485
3.26700502512563 0.0248820226851223
3.31725125628141 0.0270859333562607
3.36749748743719 0.0294215156813865
3.41774371859296 0.0318854111505416
3.46798994974874 0.0344743453346866
3.51823618090452 0.0371851265297126
3.5684824120603 0.0400146443512288
3.61872864321608 0.0429598682876753
3.66897487437186 0.0460178462187793
3.71922110552764 0.0491857029058264
3.76946733668342 0.0524606384596942
3.8197135678392 0.0558399267921794
3.86995979899498 0.0593209140556436
3.92020603015075 0.0629010170755928
3.97045226130653 0.0665777217804074
4.02069849246231 0.0703485816320635
4.07094472361809 0.0742112160613012
4.12119095477387 0.078163308910421
4.17143718592965 0.0822026068864899
4.22168341708543 0.0863269180275521
4.27192964824121 0.0905341101840744
4.32217587939699 0.0948221095176622
4.37242211055276 0.0991888990187999
4.42266834170854 0.103632517045242
4.47291457286432 0.108151055882353
4.5231608040201 0.112742660326602
4.57340703517588 0.11740552629324
4.62365326633166 0.122137899448977
4.67389949748744 0.126938073870403
4.72414572864322 0.131804390728659
4.774391959799 0.13673523700093
4.82463819095477 0.141729044208969
4.87488442211055 0.146784287185015
4.92513065326633 0.151899482865154
4.97537688442211 0.157073189110283
5.02562311557789 0.162304003554615
5.07586934673367 0.167590562481687
5.12611557788945 0.172931539727683
5.17636180904523 0.178325645611985
5.22660804020101 0.183771625894629
5.27685427135678 0.189268260760423
5.32710050251256 0.194814363829414
5.37734673366834 0.200408781193379
5.42759296482412 0.206050390477906
5.4778391959799 0.211738099929729
5.52808542713568 0.217470847528862
5.57833165829146 0.223247600125096
5.62857788944724 0.229067352598409
5.67882412060302 0.234929127042801
5.72907035175879 0.240831971973151
5.77931658291457 0.246774961554482
5.82956281407035 0.252757194853294
5.87980904522613 0.258777795110334
5.93005527638191 0.264835909034405
5.98030150753769 0.270930706116675
6.03054773869347 0.277061377964991
6.08079396984925 0.283227137657715
6.13104020100503 0.289427219116579
6.1812864321608 0.295660876498075
6.23153266331658 0.301927383602936
6.28177889447236 0.308226033303131
6.33202512562814 0.314556136986025
6.38227135678392 0.320917024015164
6.4325175879397 0.327308041207232
6.48276381909548 0.333728552324774
6.53301005025126 0.340177937584203
6.58325628140704 0.346655593178652
6.63350251256282 0.353160930815287
6.68374874371859 0.359693377266662
6.73399497487437 0.366252373935625
6.78424120603015 0.372837376433532
6.83448743718593 0.379447854171204
6.88473366834171 0.386083289962393
6.93497989949749 0.392743179639313
6.98522613065327 0.399427031679864
7.03547236180905 0.406134366846258
7.08571859296482 0.412864717834648
7.1359648241206 0.419617628935436
7.18621105527638 0.426392655703959
7.23645728643216 0.433189364641221
7.28670351758794 0.440007332884337
7.33694974874372 0.44684614790645
7.3871959798995 0.453705407225755
7.43744221105528 0.460584718123415
7.48768844221106 0.467483697370053
7.53793467336684 0.474401970960554
7.58818090452261 0.481339173856952
7.63842713567839 0.488294949739092
7.68867336683417 0.495268950762901
7.73891959798995 0.50226083732593
7.78916582914573 0.509270277840061
7.83941206030151 0.516296948511
7.88965829145729 0.523340533124548
7.93990452261307 0.530400722839193
7.99015075376885 0.537477215985039
8.04039698492462 0.544569717868755
8.0906432160804 0.551677940584365
8.14088944723618 0.558801602829747
8.19113567839196 0.565940429728569
8.24138190954774 0.573094152657641
8.29162814070352 0.580262509079308
8.3418743718593 0.587445242378907
8.39212060301508 0.594642101707063
8.44236683417085 0.601852841826589
8.49261306532663 0.609077222964019
8.54285929648241 0.616315010665449
8.59310552763819 0.623565975656665
8.64335175879397 0.63082989370742
8.69359798994975 0.638106545499603
8.74384422110553 0.645395716499385
8.79409045226131 0.652697196833043
8.84433668341708 0.660010781166385
8.89458291457287 0.667336268587768
8.94482914572864 0.674673462494414
8.99507537688442 0.682022170482114
9.0453216080402 0.689382204238046
9.09556783919598 0.696753379436749
9.14581407035176 0.704135515639025
9.19606030150754 0.711528436193806
9.24630653266332 0.718931968142774
9.2965527638191 0.726345942127767
9.34679899497488 0.733770192300752
9.39704522613065 0.741204556236409
9.44729145728643 0.748648874847188
9.49753768844221 0.756102992300733
9.54778391959799 0.763566755939664
9.59803015075377 0.771040016203622
9.64827638190955 0.778522626553433
9.69852261306533 0.786014443397507
9.74876884422111 0.793515326020191
9.79901507537688 0.801025136512151
9.84926130653266 0.80854373970268
9.89950753768844 0.816071003093885
9.94975376884422 0.823606796796696
10 0.831150993468593
};
\addplot [semithick, black, dashed]
table {%
2.59964289804225 -0.0311610017733549
2.59964289804225 0.872213469432495
};
\addlegendentry{$\alpha^\star = 2.5996...$}
\end{axis}

\end{tikzpicture}
        }
     \end{subfigure}
     \begin{subfigure}[b]{0.46\textwidth}
        \centering
        \resizebox{\columnwidth}{!}{%
        \input{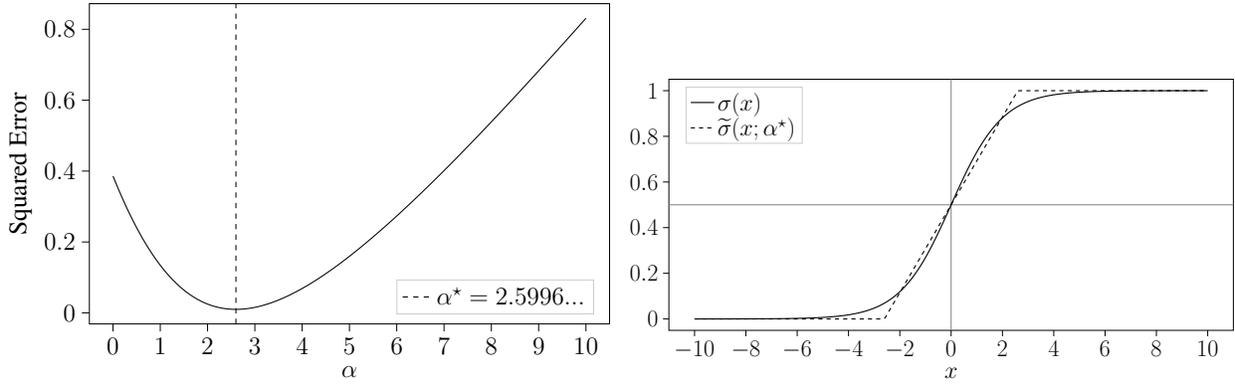}
        }
     \end{subfigure} 
    \caption{Squared error for linear approximation to sigmoid (left panel) and optimal (in square error) clipped linear approximation $\widetilde{\sigma}(x; \alpha \approx 2.5996)$ to sigmoid function $\sigma(x)$ (right panel).}
    \label{fig:opt_sig_approx}
\end{figure}

% \begin{align}
%     \frac{\partial^2 \operatorname{SE}(\alpha)}{\partial \alpha^2} &= \\
%     & \frac{{\alpha}^{2} e^{{\alpha}} - {\alpha}^{2} + 2 \, {\left(e^{{\alpha}} + 1\right)} {\rm Li}_2\left(-e^{\left(-{\alpha}\right)}\right) - 2 \, {\left(e^{{\alpha}} + 1\right)} {\rm Li}_2\left(-e^{{\alpha}}\right) - 2 \, {\left({\alpha} e^{{\alpha}} + {\alpha}\right)} \log\left({\left(e^{{\alpha}} + 1\right)} e^{\left(-{\alpha}\right)}\right) - 2 \, {\left({\alpha} e^{{\alpha}} + {\alpha}\right)} \log\left(e^{{\alpha}} + 1\right)}{{\alpha}^{3} e^{{\alpha}} + {\alpha}^{3}}
% \end{align}

% \begin{align}
%     \frac{\partial^2 \operatorname{SE}(\alpha)}{\partial \alpha^2}
%     = \frac{{\alpha}^{2} (e^{{\alpha}} - 1) + 2 \, {\left(e^{{\alpha}} + 1\right)}\left( {\rm Li}_2(-e^{-{\alpha}}) - {\rm Li}_2(-e^{{\alpha}}) \right) - 2 \, {\left({\alpha} e^{{\alpha}} + {\alpha}\right)}  \log\left({\left(e^{{\alpha}} + 1\right)^2} e^{-{\alpha}}\right) }{{\alpha}^{3} ( e^{{\alpha}} + 1)}
% \end{align}

\subsection{LAM Definition Motivation.}

Recall that an \emph{additive} model takes the form $\hat{y}(\vb*{x}) = \sigma(f(\vb*{x})) = \sigma ( \sum_{i=0}^{d} \beta_i f_i(x_i) )$,
where $\beta_i, x_i \in \mathbb{R}$ and $f_i : \mathbb{R} \to \mathbb{R}$ for all $i \in [d]$, $f_0: x \mapsto 1$.
Observe in the case of an additive model we have  for each $\vb*{x} '\in \{\vb*{z} \mid f(\vb*{z}) \in [-\alpha, \alpha] \}$, 
$\widetilde{\sigma}(f(\vb*{x}'); \alpha) = \frac{1}{2} + \sum_{i=0}^{d} \frac{\beta_i}{2 \alpha^\star} f_i (x'_i) =: \widetilde{f}(\vb*{x}')$, which follows immediately from linearity.
We shall call $\widetilde{f}$ the \emph{$\alpha^\star$-linearised model} relative to $f$.
Recall that $\projunit$ is the projector from $\mathbb{R}$ onto the unit interval, that is $\projunit(z) = \max(0, \min(1, z))$.
% \begin{equation}
% \label{eq:unit_proj_defn}
% \projunit(z) = \begin{cases}
%     0, & z \in (-\infty, 0); \\
%     z, & z \in [0, 1]; \\
%     1, & z \in (1, +\infty).
% \end{cases}
% \end{equation}
We then have the following result.
\begin{lemma}
\label{lem:add_model_equiv}
    Suppose $f$ is an additive model as defined in Definition~2.2 with $\widetilde{f}$  its $\alpha^\star$-linearised version.
    Then,
    $\widetilde{\sigma}(f(\vb*{x}); \alpha^\star) = \projunit(\widetilde{f}(\vb*{x}))$,
    that is, for an additive model $f$ the optimal piecewise linear approximation in $\piecelin$ to the logistic sigmoid evaluated on the output of $f$ is formally equivalent to the output of the $\alpha^\star$-linearised model relative to $f$, projected on the unit interval.
\end{lemma}
\begin{proof}
We proceed via case analysis. 
For brevity we denote $\widetilde{\sigma}(z; \alpha^\star)$ by $\widetilde{\sigma}(z)$.
\emph{Case i.} When $f(\vb*{x}) \in [-\alpha^\star, \alpha^\star]$ the required identity follows immediately from the definitions of $\widetilde{\sigma}$, $\widetilde{f}$ and $\projunit$.
Now suppose that $f(\vb*{x}) = \sum_{i=1}^{d} \beta_i f_i(x_i) = \gamma$.
It follows from substitution into the definition of $\widetilde{f}$ that 
\begin{equation}
\label{eq:approx_lemma_case_ii_iii}
\widetilde{f}(\vb*{x}) = \frac{1}{2}\qty(1 + \frac{\gamma}{\alpha^\star})
\end{equation}
\emph{Case ii.} When $\gamma > \alpha^\star$ we have that $\widetilde{f}(\vb*{x}) > 1$ by substitution into~\eqref{eq:approx_lemma_case_ii_iii}, in which case $\projunit\qty(\widetilde{f}(\vb*{x})) = 1 = \widetilde{\sigma}(\gamma) = \widetilde{\sigma}(f(\vb*{x}))$, yielding the required identity.
\emph{Case iii.} When $\gamma < \alpha^\star$ we have that $\widetilde{f}(\vb*{x}) < 0$ by substitution into~\eqref{eq:approx_lemma_case_ii_iii}, in which case $\projunit\qty(\widetilde{f}(\vb*{x})) = 0 = \widetilde{\sigma}(\gamma) = \widetilde{\sigma}(f(\vb*{x}))$, yielding the required identity.
We have shown that the identity holds for all values of $f(\vb*{x})$ and so the result is proved.
\end{proof}

Lemma~\ref{lem:add_model_equiv} provides the inspiration for the LAM definition (Definition~2.2 in the main text).

\subsection{Optimality Proof}

We are now able to present the LAM optimality proof for LR models. 
Recall that for an logistic model $\hat{y}(\vb*{x}) = \sigma(f(\vb*{x})) = \sigma ( \sum_{i=0}^{d} \beta_i f_i(x_i) )$, the induced LAM is given by  $\hat{y}_{\mathrm{LAM}}(\vb*{x}) = \projunit (\frac{1}{2} +  \sum_{i=0}^{d} \frac{\beta_i}{2 \alpha^\star} f_i(x_i) )$ for $\alpha^\star \approx 2.5996$.
We restate the result from the main text.

\textbf{Theorem 2.4} (LAM Optimality)\textbf{.}{\em   \ 
 Let $\piecelin$ be the space of 3-piece piecewise linear functions of one variable and $\mathcal{X} = \mathbb{R}^d$.
    For any LR model $\hat{y}(\vb*{x}) = \sigma(f(\vb*{x})) = \sigma(\beta_0 + \sum_{i = 1}^d \beta_i x_i)$ on $\mathcal{X}$ an approximator  $\widetilde{\sigma}(f(\vb*{x}))$ is defined for all $\widetilde{\sigma} \in \piecelin$. 
    Then, $\hat{y}_{\mathrm{LAM}}$ is the squared-error optimal approximator for arbitrary $f$, that is,
    $$\hat{y}_{\mathrm{LAM}}(\vb*{x}) = \widetilde{\sigma}(f(\vb*{x}); \alpha^\star) \qq{where}\widetilde{\sigma}(\,\cdot\,; \alpha^\star) = \arg\min_{\widetilde{\sigma} \in \piecelin} \qty{ \int_{\mathcal{X}} \qty(\widetilde{\sigma}(f(\vb*{x})) - \sigma(f(\vb*{x})))^2 \dd \vb*{x} } ,$$
    with $\widetilde{\sigma}(z; \alpha^\star) := \projunit(\frac{1}{2}(1 + \frac{z}{\alpha^\star}))$, $\alpha^\star \approx 2.5996$.
}

\begin{proof}
We examine the optimisation problem
\begin{align}\tag{OPT}
&\arg\min_{\widetilde{\sigma} \in \piecelin} \qty{ \int_{\mathcal{X}} \qty(\widetilde{\sigma}(f(\vb*{x})) - \sigma(f(\vb*{x})))^2 \dd \vb*{x} } \\
\intertext{Using similar symmetry arguments to those in the proof of Proposition~\ref{prop:opt_approx_sigmoid}}
\text{(OPT)}  &= \arg\min_{\alpha > 0} \qty{ \int_{\mathcal{X}} \qty(\widetilde{\sigma}(f(\vb*{x}); \alpha) - \sigma(f(\vb*{x})))^2 \dd \vb*{x} } \\
&= \arg\min_{\alpha > 0} \qty{ \lim_{R \to \infty} \int_{B(0, R)} \qty(\widetilde{\sigma}(f(\vb*{x}); \alpha) - \sigma(f(\vb*{x})))^2 \dd \vb*{x} } \\
\intertext{where $B(0, R)$ is the unit ball on $\mathbb{R}^d$ and we have equality since $\lim_{R \to \infty} \bigcup_R B(0, R) = \mathbb{R}^d$ and $B(0, R) \subset B(0, R')$ for all $R < R'$. We change integration variables to $u = f(\vb*{x})$, giving}
\text{(OPT)}  &= \arg\min_{\alpha > 0} \qty{ \lim_{R \to \infty} \int^{\infty}_{-\infty}\qty(\widetilde{\sigma}(u; \alpha) - \sigma(u))^2 g(u, R)\, \dd u }, \label{eq:opt_proof_min_1}
\end{align}
where $g(u, R)$ is the quantity
\begin{equation}\label{eq:g_uR_defn}
g(u, R) = \int_{B(0, R) \cap \{\vb*{z} \mid f(\vb*{z}) = u \}} 1 \dd \vb*{x} .
\end{equation}
Now the integral in Eq.~\eqref{eq:g_uR_defn} is describing the volume of the intersection of a sphere in $d$ dimensions centered on the origin $B(0, R)$ with a $d$-dimensional hyperplane -- the set of all $\vb*{z}$ such that $\sum_i \beta_i z_i= \beta_0 - u$.
For $- R \leq \frac{\beta_0 - u}{\sqrt{\sum_{i=1}^d \beta_i^2}} \leq R$ this volume is nonzero and given by the formula 
\begin{equation}\label{eq:g_uR_formula}
g(u, R) = \frac{\pi^{\frac{d - 1}{2}}}{\Gamma(\frac{d - 1}{2} + 1)} \qty[ 
R^2 - \frac{(\beta_0 - u)^2}{\qty(\sum_{i=1}^d \beta_i^2)^2} ]^{\frac{d - 1}{2}} \sim \frac{\pi^{\frac{d - 1}{2}}}{\Gamma(\frac{d - 1}{2} + 1)} R^{d-1} := G(R),
\end{equation}
where $q(R) \sim p(R)$ if and only if $\lim_{R \to \infty} \frac{q(R)}{p(R)} = 1$.
Making the limits explicit in Eq.~\eqref{eq:opt_proof_min_1} we have 
\begin{align}
\text{(OPT)}  &= \arg\min_{\alpha > 0} \qty{ \lim_{R \to \infty} \lim_{a \to \infty}\int^{a}_{- a}\qty(\widetilde{\sigma}(u; \alpha) - \sigma(u))^2 g(u, R)\, \dd u } \\
&= \arg\min_{\alpha > 0} \qty{ \lim_{a \to \infty} \lim_{R \to \infty} 
 \int^{a}_{- a}\qty(\widetilde{\sigma}(u; \alpha) - \sigma(u))^2 g(u, R)\, \dd u } \\
 \intertext{where we can exchange the limits by the Moore-Osgood theorem -- the integral converges uniformly for fixed $R$ as $a \to \infty$ due to the integrand tails being $O(u^{d-1} \exp(-\abs{u}))$. We then write}
 \text{(OPT)}  &= \arg\min_{\alpha > 0} \qty{ \lim_{a \to \infty} 
 \int^{a}_{- a}\qty(\widetilde{\sigma}(u; \alpha) - \sigma(u))^2 \lim_{R \to \infty}  g(u, R)\, \dd u } \\
 \intertext{bringing the limit inside the integral as $g(u, R) \sim G(R)$ uniformly on $[-a, a]$ for fixed $a$. Allowing $R$ to grow we have}
    \text{(OPT)} &= \arg\min_{\alpha > 0} \qty{ \lim_{a \to \infty} 
 \int^{a}_{- a}\qty(\widetilde{\sigma}(u; \alpha) - \sigma(u))^2 \lim_{R \to \infty}  G(R)\, \dd u } \\
     &= \arg\min_{\alpha > 0} \qty{ \lim_{R \to \infty}  G(R) \cdot \lim_{a \to \infty} 
 \int^{a}_{- a}\qty(\widetilde{\sigma}(u; \alpha) - \sigma(u))^2 \, \dd u } \\
 &= \arg\min_{\alpha > 0} \qty{ \lim_{a \to \infty} 
 \int^{a}_{- a}\qty(\widetilde{\sigma}(u; \alpha) - \sigma(u))^2 \, \dd u }, \\
 \intertext{where we ignore factors dependent on $R$ as we are interested in the optimal argument $\alpha^\star$ and not the value of the minimum itself. Taking the limit as $a \to \infty$ yields}
  \text{(OPT)} &= \arg\min_{\alpha > 0} \qty{  
 \int^{\infty}_{- \infty}\qty(\widetilde{\sigma}(u; \alpha) - \sigma(u))^2 \, \dd u } \\
 &= \widetilde{\sigma}(\,\cdot\, ; \alpha^\star) \label{eq:opt_proof_final}
\end{align}
from Proposition~\ref{prop:opt_approx_sigmoid}.
Finally, from Lemma~\ref{lem:add_model_equiv} we have that the LAM induced from an LR model is $$\hat{y}_{\mathrm{LAM}}(\vb*{x}) = \projunit \qty(\frac{1}{2} +  \sum_{i=0}^{d} \frac{\beta_i}{2 \alpha^\star} x_i) = \widetilde{\sigma}\qty(\sum_{i=0}^{d} \beta_i x_i; \alpha^\star),$$
which upon comparison with Eq.~\eqref{eq:opt_proof_final} completes the proof.
\end{proof}

\section{Experiment Details}

\subsection{Problem Domain} 
We explicitly consider the consumer credit modelling domain, as this is a high-stakes application where interpretability is critical.
From the perspective of a machine learning model, a risk factor is effectively a specific grouping of features and their values for a particular input.
Such groupings are known in the academic literature as \emph{subscales}.
Monotone constraints and explicit subscale modelling limit the logistic additive models that we consider for comparison with their equivalent LAMs.
We formally describe these aspects below.

\paragraph{Monotone Constraints.} The set of monotone increasing variables is denoted as $\mathcal{I} \subseteq [d]$, monotone decreasing variables $\mathcal{D} \subseteq [d]$ and the remaining unconstrained variables $\mathcal{U} = [d] \setminus (\mathcal{I} \cup \mathcal{D})$.
Note that a variable can either be monotone increasing, decreasing, or neither, that is, $\mathcal{I}$, $\mathcal{D}$ and $\mathcal{U}$ form a partition of $[d]$.
Monotone constraints for a particular dataset are specified by the modeller using domain knowledge.
% For instance, in the HELOC dataset~\cite{heloc} the feature \texttt{NumTradesOpeninLast12M}, corresponding to the number of lines of credit open in the preceding year, belongs to $\mathcal{I}$ -- higher values correspond to higher default risk.
Categorical features are one-hot encoded and belong to $\mathcal{U}$.

\paragraph{Subscales.} % An example from the HELOC dataset~\cite{heloc} is the subscale \texttt{TradeFrequency}, comprising four features relating to how often a customer opens a new line of credit.
Formally, the subscales partition the feature set: a dataset with $d$ input features has every subscale $S_i$ satisfying $S_i \subseteq [d]$, $S_i \cap S_j = \emptyset$ for all $i \neq j$ and $\bigcup_i S_i = [d]$.
The set of all subscales we denote by $\mathcal{S}$ and $\mathcal{S}$ is decided by the modeller a priori.

\subsection{Baseline Models}

We briefly describe a number of additive models, as defined in Definition~2.1, from the literature that will be used in the experimental comparison. 
These models will be used to test the linearisation procedure of Definition~2.1, namely LAMs will be constructed from the baseline additive models and their performance will be compared.

\subsubsection{Nonnegative Logistic Regression.}
We define \emph{Nonnegative Logistic Regression models (NNLR)} as those solving the following optimisation problem~\cite{CHEN2022113647}.
\begin{equation}
    \begin{aligned}
    \label{eq:nnlr_opt}
    \min_{\beta_i} \frac{1}{M}\sum_{j=1}^M L(\hat{y}(\vb*{x}^{(j)}), y^{(j)})) + C \sum_{i=1}^{d} \beta_i^2 \qq{s. t.} \beta_i \geq 0 \ \ \forall  i \in \mathcal{I};\ \ \beta_i \leq 0 \ \ \forall i \in \mathcal{D}.  
\end{aligned}
\end{equation}
where the logistic loss $L(\hat{y}, y) = -y \log \hat{y} - (1 - y ) \log(1 - \hat{y})$ and 
$\hat{y}(\vb*{x}) = \sigma( \beta_0 + \sum_{i=1}^d \beta_i x_i )$, with $C \geq 0$ being a regularisation hyperparameter.
We solve the optimisation problem~\eqref{eq:nnlr_opt} via Sequential Least Squares Programming (SLSQP)~\cite{Kraft1994,kraft1988software} as implemented in \texttt{scipy}~\cite{2020SciPy-NMeth}. 
NNLR models follow the familiar $\ell_2$-regularised (or `ridge') logistic regression~\cite{hastie2009elements}, with the added constraint on the model coefficients enforcing the monotone directions on the variables in $\mathcal{I}$ and $\mathcal{D}$.
The ``nonnegative'' label arises from the constraints all being positive in the Additive Risk Models of~\cite{CHEN2022113647}.%, despite $\beta_i$ being negative for monotone decreasing variables in $\mathcal{D}$.

\subsubsection{Additive Risk Models (ARMs).} ARMs are additive models introduced in~\cite{CHEN2022113647} and come in two flavours, \emph{One-Layer ARMs (ARM1)} and \emph{Two-Layer ARMs (ARM2)}.
We begin with a brief description of ARM1, and refer the reader to~\cite{CHEN2022113647} for a more detailed description.
These models are chosen as they are the state-of-the-art interpretable GAMs that allow monotone constraints and in the case of ARM2, explicit subscale modelling.

The ARM1 model comprises a NNLR model with input features transformed like so: categorical features and special values are one-hot encoded.
For a monotone decreasing continuous feature $u \in \mathcal{D}$, a series of $L_u \in \mathbb{N}$ indicator variables are created in its stead, signifying membership of half intervals $(-\infty, \theta_j]$ with right-side boundaries $\theta_1 < \theta_2 < \cdots < \theta_{L_u} = +\infty$.
The corresponding NNLR model coefficients $\beta_{u; \theta_j}$ for $j \in [L_u]$ are all constrained to be nonnegative, i.e. $x_{u; \theta_j} \in \mathcal{I}$.
This feature processing scheme enforces that the ARM1 model output is monotone decreasing in $x_u$ -- in the notation of Definition~2.1, $f_u(x_u)$ is piecewise constant and monotone decreasing.
Monotone increasing features receive a similar treatment, with the indicator variables corresponding to the half-intervals $[\theta_j, +\infty)$ with \emph{left-side} boundaries $\theta_1 > \theta_2 > \cdots > \theta_{L_u} = \phi_u$, where $\phi_u$ lower bounds $x_u$.
Again, $\beta_{u; \theta_j} \in \mathcal{I}$, ensuring that $f_u(x_u)$ is piecewise constant and monotone \emph{increasing}.
For unconstrained variables $u \in \mathcal{U}$, the intervals $(\theta_j, \theta_{j+1}]$ are \emph{two-sided}, with bin edges $\phi_u = \theta_1 < \theta_2 < \cdots < \theta_{L_u} = + \infty$.
The bin edges $\theta_j$ for all continuous variables are decided using an entropy-based scheme~\cite{CHEN2022113647} that is essentially equivalent to training a CART decision tree separately on each feature $x_u$ with the fixed number of leaves $L_u$ pre-determined as a hyperparameter.

The ARM2 model is built out of ARM1 models trained on solely the variables for a particular subscale, with the output of the subscale $S$ model designated by $r^{[S]}(\vb*{x})$.
The output for each subscale model $r^{[S]}$ lies within $[0, 1]$ and is interpreted as the risk arising solely from $S$.
The second layer of ARM2, $r(\vb*{x}) = \sigma(\beta_0 + \sum_{S \in \mathcal{S}} \beta_{S} r^{[S]}(\vb*{x}))$, is an additional NNLR model with the subscale risk scores as input variables.
The input variables are fixed to be monotone increasing, that is, $\beta_S \geq 0$.
Note that the individual subscale ARM1 models $r^{[S]}$ are trained first and their outputs treated as fixed for the training of the global model $r(\vb*{x})$.

% \section{Running the experiments}

% Unzip the file \texttt{repo.zip} and follow the instructions in \texttt{README.md} in the top level directory.

\subsection{Model training and hyperparameters}\label{sec:hyperparams}

For the ARM models on the HELOC and German Credit datasets, we replicate the hyperparameters (namely feature bin edges $\theta_j$) of~\cite{CHEN2022113647}. 
For the remaining datasets we restrict each continuous feature using feature binning to 5 bins and the \texttt{DecisionTreeClassifier} from scikit-learn~\cite{scikit-learn} is trained on each individual feature, with the splits being used to decide the bins.
NNLR-based models have regularisation parameter $C=0$ in all cases.
All XGB models use the following hyperparameters, with the remaining being the default parameters used in version 1.4.2~\cite{xgboost}.

\begin{lstlisting}[language=Python, caption=XGBoost hyperparameters]
xgb_base_params = {
    "max_depth": 2,
    "n_estimators": 50,
    "learning_rate": 0.1,
    "eval_metric": 'logloss', 
    "use_label_encoder": False,
    "missing": self.binariser_kwargs["special_value_threshold"]
}
\end{lstlisting}
The value \texttt{self.binariser\_kwargs[\textcolor{codepurple}{"special\_value\_threshold"}]} corresponds to $\phi_u$ in the main text and is decided for each dataset separately.

\subsection{Calibration}\label{sec:calibration}

Generally, a set of predictions of a binary outcome is well calibrated if the outcomes predicted to occur with probability $p$ occur about $p$ fraction of the time, for any probability $p \in [0, 1]$.
A common method for assessing the calibration of a binary classifier is the \emph{reliability diagram}~\cite{DeGroot1983,Niculescu_Mizil2005}.%, wherein the empirical frequency of the positive class is plotted against the predicted positive class probability by the model.
The reliability diagram of a perfect classifier would be a line segment from $(0, 0)$ to $(1, 1)$.
In practise, the reliability diagram is often plotted alongside the idealised version and the difference is visually inspected, with the empirical frequencies computed with respect to some binning scheme over the model-predicted probabilities. 
We consider two widely-used numerical summary statistics for the calibration, \emph{Expected Calibration Error (ECE)} and \emph{Maximum Calibration Error, (MCE)}~\cite{DeGroot1983,Niculescu_Mizil2005}.
To compute these metrics, the test set predictions are sorted and partitioned into $K$ equally spaced bins over $[0, 1]$ ($K = 15$ in our experiments). 
We then have
$\operatorname{ECE} = \sum^K_{i=1} P(i) \cdot \abs{o_i - e_i}$ and $\operatorname{MCE} =\max_{i\in\{1, \ldots, K\}} \abs{o_i - e_i}$,
where $o_i$ is the true fraction of positive instances in bin $i$, $e_i$ is the arithmetic mean of the model outputs for the instances in bin $i$, and $P(i)$ is the empirical probability (fraction) of all instances that fall into bin $i$.
Lower values of ECE and MCE correspond to better calibration of a particular model, with the idealised model having a value of zero for both.

\subsection{Dataset preprocessing}\label{sec:preprocessing}

If we wish to distinguish metrics of interest of the $k=8$ different algorithms under consideration, then it is important for the number of statistically independent datasets, $N$, to be as large as possible so that the power of any statistical tests is maximal.
The work~\cite{Campelo2020} provides an algorithm for determining the appropriate $N$ for a desired statistical power.
We restrict our experiments to publically available datasets that are not synthetically generated, so we cannot keep generating datasets until the desired $N$ is reached.
Nonetheless, we are able to use $N=24$ independent datasets.
For the Poland dataset, there are 5 separate datasets provided, corresponding to consecutive years of data.
We consider these datasets separately, and denote them by $\text{Poland}\_n$ for $n\in\{ 0, \ldots, 4\}$.
The LC dataset is an order of magnitude larger than the others, thus we split the data into temporally contiguous (and disjoint) regions comprising 100000 datapoints each (apart from the last region).
This procedure leaves us with 13 datasets, which we denote by $\text{LC}\_n$ for $n\in\{ 0, \ldots, 12\}$.
% LendingClub, a large dataset of loans issued from a peer-to-peer lending site which includes loan status and latest payment information, as well as a variety of features describing the applicants such as number of open bank accounts, age, and amount requested.
The subscales $\mathcal{S}$ and monotone constraints $(\mathcal{I}, \mathcal{D}, \mathcal{U})$ for the remaining datasets were decided based on domain knowledge and are included in full in the appendix.
% \textcolor{blue}{Table with dataset properties?}

In terms of the subscale groupings $\mathcal{S}$, monotone constraints $(\mathcal{I}, \mathcal{D}, \mathcal{U})$ and feature lower bounds $\{\phi_u\}_{u \in [d]}$ for each dataset, for the HELOC and German datasets we replicate the assignments from~\cite{CHEN2022113647}.
In the code listings below we show all of the subscale groupings  $\mathcal{S}$ as a Python \texttt{OrderedDict} with name \texttt{<dataset>\_RC\_FEATURE\_MAPPING}, with keys corresponding to the names of the subscale, and values comprising a list of the constituent features.
Monotone constraints are stored in a Python \texttt{dict} called \texttt{<dataset>\_MONOTONE\_CONSTRAINTS}, with keys corresponding to feature names, and values in $\{-1, 1, 0\}$ corresponding to monotone decreasing, increasing and no constraint respectively.
\texttt{<dataset>\_SPECIAL\_VALUES\_DICT} contains special values for each feature that are subsequently one-hot encoded, and \texttt{<dataset>\_SPECIAL\_VALUE\_THRESHOLD} corresponds to a global feature lower bound $\phi_u$ for each dataset.
The \texttt{dict} called   \texttt{<dataset>\_MAX\_LEAF\_NODES\_DICT} corresponds to the number of allowed bins for each (continuous) feature.
The Poland dataset has uninformative attribute names, which we replace with informative names from~\cite{zikeba2016ensemble} using the Python \texttt{dict} called \texttt{POLAND\_BANKRUPTCY\_FEATURE\_MAPPING}.

\begin{lstlisting}[language=Python, caption=HELOC preprocessing]
HELOC_RC_FEATURE_MAPPING = OrderedDict({
    "ExternalRiskEstimate": ["ExternalRiskEstimate"],
    "TradeOpenTime": ["MSinceOldestTradeOpen", "MSinceMostRecentTradeOpen", "AverageMInFile"],
    "NumSatisfactoryTrades": ["NumSatisfactoryTrades"],
    "TradeFrequency": ["NumTrades60Ever2DerogPubRec", "NumTrades90Ever2DerogPubRec", "NumTotalTrades", "NumTradesOpeninLast12M"],
    "Delinquency": ["PercentTradesNeverDelq", "MSinceMostRecentDelq", "MaxDelq2PublicRecLast12M", "MaxDelqEver"],
    "Installment": ["PercentInstallTrades", "NetFractionInstallBurden", "NumInstallTradesWBalance"],
    "Inquiry": ["MSinceMostRecentInqexcl7days", "NumInqLast6M", "NumInqLast6Mexcl7days"],
    "RevolvingBalance": ["NetFractionRevolvingBurden", "NumRevolvingTradesWBalance"],
    "Utilization": ["NumBank2NatlTradesWHighUtilization"],
    "TradeWBalance": ["PercentTradesWBalance"]
})


HELOC_MONOTONE_CONSTRAINTS = {
    'ExternalRiskEstimate': -1,
    'MSinceOldestTradeOpen': -1,
    'MSinceMostRecentTradeOpen': -1,
    'AverageMInFile': -1,
    'NumSatisfactoryTrades': -1,
    'NumTrades60Ever2DerogPubRec': 1,
    'NumTrades90Ever2DerogPubRec': 1,
    'NumTotalTrades': 0,
    'NumTradesOpeninLast12M': +1,
    'PercentTradesNeverDelq': -1,
    'MSinceMostRecentDelq': -1,
    'MaxDelq2PublicRecLast12M': 0,
    'MaxDelqEver': 0,
    'PercentInstallTrades': 0,
    'NetFractionInstallBurden': +1,
    'NumInstallTradesWBalance': 0,
    'MSinceMostRecentInqexcl7days': -1,
    'NumInqLast6M': +1,
    'NumInqLast6Mexcl7days': +1,
    'NetFractionRevolvingBurden': +1,
    'NumRevolvingTradesWBalance': 0,
    'NumBank2NatlTradesWHighUtilization': +1,
    'PercentTradesWBalance': 0 
}

HELOC_CATEGORICAL_COLS = ["MaxDelq2PublicRecLast12M", "MaxDelqEver"]

HELOC_MAX_LEAF_NODES_DICT = {
    'ExternalRiskEstimate': 5,
    'MSinceOldestTradeOpen': 4,
    'MSinceMostRecentTradeOpen': 2,
    'AverageMInFile': 4,
    'NumSatisfactoryTrades': 5,
    'NumTrades60Ever2DerogPubRec': 5,
    'NumTrades90Ever2DerogPubRec': 4,
    'NumTotalTrades': 5,
    'NumTradesOpeninLast12M': 5,
    'PercentTradesNeverDelq': 5,
    'MSinceMostRecentDelq': 4,
    'MaxDelq2PublicRecLast12M': 3,
    'MaxDelqEver': 2,
    'PercentInstallTrades': 5,
    'NetFractionInstallBurden': 3,
    'NumInstallTradesWBalance': 5,
    'MSinceMostRecentInqexcl7days': 5,
    'NumInqLast6M': 4,
    'NumInqLast6Mexcl7days': 2,
    'NetFractionRevolvingBurden': 4,
    'NumRevolvingTradesWBalance': 5,
    'NumBank2NatlTradesWHighUtilization': 5,
    'PercentTradesWBalance': 5
}

HELOC_SPECIAL_VALUES_DICT = { feature: [-7, -8, -9] 
                              for feature in HELOC_MAX_LEAF_NODES_DICT.keys()
                            } 

HELOC_SPECIAL_VALUE_THRESHOLD = -0.5
\end{lstlisting}

\begin{lstlisting}[language=Python, caption=German preprocessing]
GERMAN_CREDIT_HEADERS =["Status of existing checking account","Duration in month","Credit history",\
         "Purpose","Credit amount","Savings account/bonds","Present employment since",\
         "Installment rate in percentage of disposable income","Personal status and sex",\
         "Other debtors / guarantors","Present residence since","Property","Age in years",\
        "Other installment plans","Housing","Number of existing credits at this bank",\
        "Job","Number of people being liable to provide maintenance for","Telephone","foreign worker", "Target"]

GERMAN_CREDIT_FEATURES = copy(GERMAN_CREDIT_HEADERS); GERMAN_CREDIT_FEATURES.remove("Target")

GERMAN_CREDIT_CATEGORICAL_COLS = [
                                  "Credit history",
                                  "Purpose",
                                  "Present employment since",
                                  "Personal status and sex",
                                  "Other debtors / guarantors",
                                  "Property",
                                  "Other installment plans",
                                  "Housing",
                                  "Job",
                                  "Telephone",
                                  "foreign worker"
                                 ]

GERMAN_CREDIT_NON_CATEGORICAL_COLS = list(set(GERMAN_CREDIT_FEATURES) - set(GERMAN_CREDIT_CATEGORICAL_COLS))

GERMAN_CREDIT_RC_FEATURE_MAPPING = OrderedDict({
    "CreditLoanInfo": [
        "Status of existing checking account", 
        "Credit history",
        "Purpose",
        "Savings account/bonds"
    ],
    "PersonalInfo": [
        "Present employment since",
        "Personal status and sex",
        "Other debtors / guarantors",
        "Property",
        "Other installment plans",
        "Housing",
        "Job",
        "Telephone",
        "foreign worker"
    ]
})

GERMAN_CREDIT_MONOTONE_CONSTRAINTS = {feature: 0 for feature in GERMAN_CREDIT_FEATURES}
GERMAN_CREDIT_MONOTONE_CONSTRAINTS["Status of existing checking account"] = 1
GERMAN_CREDIT_MONOTONE_CONSTRAINTS["Savings account/bonds"] = 1

GERMAN_CREDIT_MAX_LEAF_NODES_DICT = {feature: 5 for feature in GERMAN_CREDIT_NON_CATEGORICAL_COLS}

GERMAN_CREDIT_SPECIAL_VALUE_THRESHOLD = -1000

GERMAN_CREDIT_SPECIAL_VALUES_DICT = { feature: [GERMAN_CREDIT_SPECIAL_VALUE_THRESHOLD] 
                              for feature in GERMAN_CREDIT_NON_CATEGORICAL_COLS
                            } 
\end{lstlisting}

\begin{lstlisting}[language=Python, caption=Taiwan preprocessing]
TAIWAN_CREDIT_FEATURES = ['LIMIT_BAL', 'SEX', 'EDUCATION', 'MARRIAGE', 'AGE', 'PAY_0', 'PAY_2',
       'PAY_3', 'PAY_4', 'PAY_5', 'PAY_6', 'BILL_AMT1', 'BILL_AMT2',
       'BILL_AMT3', 'BILL_AMT4', 'BILL_AMT5', 'BILL_AMT6', 'PAY_AMT1',
       'PAY_AMT2', 'PAY_AMT3', 'PAY_AMT4', 'PAY_AMT5', 'PAY_AMT6']

TAIWAN_CREDIT_CATEGORICAL_COLS = [ 
    'SEX', 'EDUCATION', 'MARRIAGE',
    'PAY_0', 'PAY_2', 'PAY_3', 'PAY_4', 'PAY_5', 'PAY_6'
]

TAIWAN_CREDIT_NON_CATEGORICAL_COLS = list(set(TAIWAN_CREDIT_FEATURES) - set(TAIWAN_CREDIT_CATEGORICAL_COLS))

TAIWAN_CREDIT_RC_FEATURE_MAPPING = OrderedDict({
    "CreditLoanInfo": [
        "LIMIT_BAL"
    ],
    "PersonalInfo": [
        'SEX', 'EDUCATION', 'MARRIAGE', 'AGE'
    ],
    "RepaymentStatus": [
        'PAY_0', 'PAY_2', 'PAY_3', 'PAY_4', 'PAY_5', 'PAY_6'
    ],
    "BillAmount": [
        'BILL_AMT1', 'BILL_AMT2', 'BILL_AMT3', 'BILL_AMT4', 'BILL_AMT5', 'BILL_AMT6'
    ],
    "PaymentAmounts": [
        'PAY_AMT1', 'PAY_AMT2', 'PAY_AMT3', 'PAY_AMT4', 'PAY_AMT5', 'PAY_AMT6'
    ]
})

TAIWAN_CREDIT_MONOTONE_CONSTRAINTS = {feature: 0 for feature in TAIWAN_CREDIT_FEATURES}
TAIWAN_CREDIT_MONOTONE_CONSTRAINTS['PAY_0'] = 1
for n in range(1, 7):
    TAIWAN_CREDIT_MONOTONE_CONSTRAINTS[f'PAY_AMT{n}'] = -1

TAIWAN_CREDIT_MAX_LEAF_NODES_DICT = {feature: 5 for feature in TAIWAN_CREDIT_NON_CATEGORICAL_COLS}

TAIWAN_CREDIT_SPECIAL_VALUE_THRESHOLD = -400000

TAIWAN_CREDIT_SPECIAL_VALUES_DICT = { feature: [TAIWAN_CREDIT_SPECIAL_VALUE_THRESHOLD] 
                              for feature in TAIWAN_CREDIT_NON_CATEGORICAL_COLS
                            } 
\end{lstlisting}

\begin{lstlisting}[language=Python, caption=Give Me Some Credit preprocessing]
GIVE_ME_SOME_CREDIT_FEATURES = [
    'RevolvingUtilizationOfUnsecuredLines', 
    'age',
    'NumberOfTime30-59DaysPastDueNotWorse',
    'DebtRatio',
    'MonthlyIncome',
    'NumberOfOpenCreditLinesAndLoans',
    'NumberOfTimes90DaysLate',
    'NumberRealEstateLoansOrLines',
    'NumberOfTime60-89DaysPastDueNotWorse',
    'NumberOfDependents'
]

GIVE_ME_SOME_CREDIT_CATEGORICAL_COLS = []

GIVE_ME_SOME_CREDIT_SPECIAL_VALUES_DICT = {feature: [-1] for feature in GIVE_ME_SOME_CREDIT_FEATURES}

GIVE_ME_SOME_CREDIT_MAX_LEAF_NODES_DICT = {feature: 5 for feature in GIVE_ME_SOME_CREDIT_FEATURES}

GIVE_ME_SOME_CREDIT_SPECIAL_VALUE_THRESHOLD = -0.5

GIVE_ME_SOME_CREDIT_MONOTONE_CONSTRAINTS = {feature: 0 for feature in GIVE_ME_SOME_CREDIT_FEATURES}
for feature in ['RevolvingUtilizationOfUnsecuredLines', 'NumberOfTime30-59DaysPastDueNotWorse', 'NumberOfTime60-89DaysPastDueNotWorse', 'NumberOfTimes90DaysLate']:
    GIVE_ME_SOME_CREDIT_MONOTONE_CONSTRAINTS[feature] = 1
    
for feature in ['MonthlyIncome']:
    GIVE_ME_SOME_CREDIT_MONOTONE_CONSTRAINTS[feature] = -1

GIVE_ME_SOME_CREDIT_REASON_CODE_MAPPING = OrderedDict({
    "HistoricalLatePayments": [
        'NumberOfTime30-59DaysPastDueNotWorse', 'NumberOfTime60-89DaysPastDueNotWorse', 'NumberOfTimes90DaysLate'
    ],
    "FinancialObligations": [
        'RevolvingUtilizationOfUnsecuredLines', 'NumberOfOpenCreditLinesAndLoans', 'NumberRealEstateLoansOrLines', 'NumberOfDependents'
    ],
    "FinancialCapabilities": [
        'MonthlyIncome'
    ],
    "Demographics": [
        "age"
    ]
})
\end{lstlisting}

\begin{lstlisting}[language=Python, caption=Japan preprocessing]
# From https://www.rpubs.com/kuhnrl30/CreditScreen
JAPAN_CREDIT_PUTATIVE_FEATURE_NAMES = [
    "Male",
    "Age",
    "Debt",
    "Married",
    "BankCustomer",
    "EducationLevel",
    "Ethnicity",
    "YearsEmployed",
    "PriorDefault",
    "Employed",
    "CreditScore",
    "DriversLicense",
    "Citizen",
    "ZipCode",
    "Income"
]

JAPAN_CREDIT_NUMERIC_COLS = ["Age", "Debt", "YearsEmployed", "CreditScore", "Income"]
JAPAN_CREDIT_CATEGORICAL_COLS = list(set(JAPAN_CREDIT_PUTATIVE_FEATURE_NAMES) - set(JAPAN_CREDIT_NUMERIC_COLS))

JAPAN_CREDIT_RC_FEATURE_MAPPING = OrderedDict({
    "Demographic": ["Male", "Age", "Married", "Ethnicity", "ZipCode", "Citizen"],
    "Career": ["EducationLevel", "YearsEmployed", "Employed", "DriversLicense"],
    "FinancialObligations": ["Debt", "PriorDefault"],
    "FinancialCapabilities": ["CreditScore", "BankCustomer", "Income"]
    
})

JAPAN_CREDIT_MONOTONE_CONSTRAINTS = {feature: 0 for feature in JAPAN_CREDIT_PUTATIVE_FEATURE_NAMES}
for feature in ['Debt']:
    JAPAN_CREDIT_MONOTONE_CONSTRAINTS[feature] = 1
for feature in ['Income', 'CreditScore']:
    JAPAN_CREDIT_MONOTONE_CONSTRAINTS[feature] = -1

JAPAN_CREDIT_MAX_LEAF_NODES_DICT = {feature: 5 for feature in JAPAN_CREDIT_NUMERIC_COLS}

JAPAN_CREDIT_SPECIAL_VALUES_DICT = {feature: [-1] for feature in JAPAN_CREDIT_NUMERIC_COLS}

JAPAN_CREDIT_SPECIAL_VALUE_THRESHOLD = -0.5
\end{lstlisting}

\begin{lstlisting}[language=Python, caption=Australia preprocessing]
# From https://www.rpubs.com/kuhnrl30/CreditScreen
AUSTRALIA_CREDIT_PUTATIVE_FEATURE_NAMES = [
    "Male",
    "Age",
    "Debt",
    "BankCustomer",
    "EducationLevel",
    "Ethnicity",
    "YearsEmployed",
    "PriorDefault",
    "Employed",
    "CreditScore",
    "DriversLicense",
    "Citizen",
    "ZipCode",
    "Income"
]

AUSTRALIA_CREDIT_NUMERIC_COLS = ["Age", "Debt", "YearsEmployed", "CreditScore", "Income"]
AUSTRALIA_CREDIT_CATEGORICAL_COLS = list(set(AUSTRALIA_CREDIT_PUTATIVE_FEATURE_NAMES) - set(AUSTRALIA_CREDIT_NUMERIC_COLS))

AUSTRALIA_CREDIT_RC_FEATURE_MAPPING = OrderedDict({
    "Demographic": ["Male", "Age", "Ethnicity", "ZipCode", "Citizen"],
    "Career": ["EducationLevel", "YearsEmployed", "Employed", "DriversLicense"],
    "FinancialObligations": ["Debt", "PriorDefault"],
    "FinancialCapabilities": ["CreditScore", "BankCustomer", "Income"]
    
})

AUSTRALIA_CREDIT_MONOTONE_CONSTRAINTS = {feature: 0 for feature in AUSTRALIA_CREDIT_PUTATIVE_FEATURE_NAMES}
for feature in ['Debt']:
    AUSTRALIA_CREDIT_MONOTONE_CONSTRAINTS[feature] = 1
for feature in ['Income', 'CreditScore']:
    AUSTRALIA_CREDIT_MONOTONE_CONSTRAINTS[feature] = -1

AUSTRALIA_CREDIT_MAX_LEAF_NODES_DICT = {feature: 5 for feature in AUSTRALIA_CREDIT_NUMERIC_COLS}

AUSTRALIA_CREDIT_SPECIAL_VALUES_DICT = {feature: [-1] for feature in AUSTRALIA_CREDIT_NUMERIC_COLS}

AUSTRALIA_CREDIT_SPECIAL_VALUE_THRESHOLD = -0.5

\end{lstlisting}

\begin{lstlisting}[language=Python, caption=Poland preprocessing]
POLAND_BANKRUPTCY_FEATURE_MAPPING = {
    "Attr1" : "net profit / total assets",
    "Attr2" : "total liabilities / total assets",
    "Attr3" : "working capital / total assets",
    "Attr4" : "current assets / short-term liabilities",
    "Attr5" : "((cash + short-term securities + receivables - short-term liabilities) / (operating expenses - depreciation)) * 365",
    "Attr6" : "retained earnings / total assets",
    "Attr7" : "EBIT / total assets",
    "Attr8" : "book value of equity / total liabilities",
    "Attr9" : "sales / total assets",
    "Attr10" : "equity / total assets",
    "Attr11" : "(gross profit + extraordinary items + financial expenses) / total assets",
    "Attr12" : "gross profit / short-term liabilities",
    "Attr13" : "(gross profit + depreciation) / sales",
    "Attr14" : "(gross profit + interest) / total assets",
    "Attr15" : "(total liabilities * 365) / (gross profit + depreciation)",
    "Attr16" : "(gross profit + depreciation) / total liabilities",
    "Attr17" : "total assets / total liabilities",
    "Attr18" : "gross profit / total assets",
    "Attr19" : "gross profit / sales",
    "Attr20" : "(inventory * 365) / sales",
    "Attr21" : "sales (n) / sales (n-1)",
    "Attr22" : "profit on operating activities / total assets",
    "Attr23" : "net profit / sales",
    "Attr24" : "gross profit (in 3 years) / total assets",
    "Attr25" : "(equity - share capital) / total assets",
    "Attr26" : "(net profit + depreciation) / total liabilities",
    "Attr27" : "profit on operating activities / financial expenses",
    "Attr28" : "working capital / fixed assets",
    "Attr29" : "logarithm of total assets",
    "Attr30" : "(total liabilities - cash) / sales",
    "Attr31" : "(gross profit + interest) / sales",
    "Attr32" : "(current liabilities * 365) / cost of products sold",
    "Attr33" : "operating expenses / short-term liabilities",
    "Attr34" : "operating expenses / total liabilities",
    "Attr35" : "profit on sales / total assets",
    "Attr36" : "total sales / total assets",
    "Attr37" : "(current assets - inventories) / long-term liabilities",
    "Attr38" : "constant capital / total assets",
    "Attr39" : "profit on sales / sales",
    "Attr40" : "(current assets - inventory - receivables) / short-term liabilities",
    "Attr41" : "total liabilities / ((profit on operating activities + depreciation) * (12/365))",
    "Attr42" : "profit on operating activities / sales",
    "Attr43" : "rotation receivables + inventory turnover in days",
    "Attr44" : "(receivables * 365) / sales",
    "Attr45" : "net profit / inventory",
    "Attr46" : "(current assets - inventory) / short-term liabilities",
    "Attr47" : "(inventory * 365) / cost of products sold",
    "Attr48" : "EBITDA (profit on operating activities - depreciation) / total assets",
    "Attr49" : "EBITDA (profit on operating activities - depreciation) / sales",
    "Attr50" : "current assets / total liabilities",
    "Attr51" : "short-term liabilities / total assets",
    "Attr52" : "(short-term liabilities * 365) / cost of products sold)",
    "Attr53" : "equity / fixed assets",
    "Attr54" : "constant capital / fixed assets",
    "Attr55" : "working capital",
    "Attr56" : "(sales - cost of products sold) / sales",
    "Attr57" : "(current assets - inventory - short-term liabilities) / (sales - gross profit - depreciation)",
    "Attr58" : "total costs /total sales",
    "Attr59" : "long-term liabilities / equity",
    "Attr60" : "sales / inventory",
    "Attr61" : "sales / receivables",
    "Attr62" : "(short-term liabilities *365) / sales",
    "Attr63" : "sales / short-term liabilities",
    "Attr64" : "sales / fixed assets"
}

POLAND_BANKRUPTCY_CATEGORICAL_COLS = []
POLAND_BANKRUPTCY_NUMERIC_COLS = [
    'net profit / total assets', 'total liabilities / total assets', 'working capital / total assets', 
    'current assets / short-term liabilities', '((cash + short-term securities + receivables - short-term liabilities) / (operating expenses - depreciation)) * 365',
    'retained earnings / total assets', 'EBIT / total assets', 'book value of equity / total liabilities',
    'sales / total assets', 'equity / total assets', '(gross profit + extraordinary items + financial expenses) / total assets',
    'gross profit / short-term liabilities', '(gross profit + depreciation) / sales', '(gross profit + interest) / total assets',
    '(total liabilities * 365) / (gross profit + depreciation)', '(gross profit + depreciation) / total liabilities',
    'total assets / total liabilities', 'gross profit / total assets', 'gross profit / sales',
    '(inventory * 365) / sales', 'sales (n) / sales (n-1)', 'profit on operating activities / total assets',
    'net profit / sales', 'gross profit (in 3 years) / total assets', '(equity - share capital) / total assets', '(net profit + depreciation) / total liabilities',
    'profit on operating activities / financial expenses', 'working capital / fixed assets', 'logarithm of total assets',
    '(total liabilities - cash) / sales', '(gross profit + interest) / sales', '(current liabilities * 365) / cost of products sold',
    'operating expenses / short-term liabilities', 'operating expenses / total liabilities', 'profit on sales / total assets',
    'total sales / total assets', 'constant capital / total assets', 'profit on sales / sales',
    '(current assets - inventory - receivables) / short-term liabilities', 'total liabilities / ((profit on operating activities + depreciation) * (12/365))',
    'profit on operating activities / sales', 'rotation receivables + inventory turnover in days', '(receivables * 365) / sales',
    'net profit / inventory', '(current assets - inventory) / short-term liabilities', '(inventory * 365) / cost of products sold',
    'EBITDA (profit on operating activities - depreciation) / total assets', 'EBITDA (profit on operating activities - depreciation) / sales',
    'current assets / total liabilities', 'short-term liabilities / total assets', '(short-term liabilities * 365) / cost of products sold)',
    'equity / fixed assets', 'constant capital / fixed assets', 'working capital', '(sales - cost of products sold) / sales',
    '(current assets - inventory - short-term liabilities) / (sales - gross profit - depreciation)', 'total costs /total sales',
    'long-term liabilities / equity', 'sales / inventory', 'sales / receivables', '(short-term liabilities *365) / sales',
    'sales / short-term liabilities', 'sales / fixed assets'
]

POLAND_BANKRUPTCY_MONOTONE_CONSTRAINTS = {
    'net profit / total assets': -1,
    'total liabilities / total assets': 1,
    'working capital / total assets': -1,
    'current assets / short-term liabilities': -1,
    '((cash + short-term securities + receivables - short-term liabilities) / (operating expenses - depreciation)) * 365': -1,
    'retained earnings / total assets': -1,
    'EBIT / total assets': -1,
    'book value of equity / total liabilities': -1,
    'sales / total assets': 0,
    'equity / total assets': -1,
    '(gross profit + extraordinary items + financial expenses) / total assets': -1,
    'gross profit / short-term liabilities': -1,
    '(gross profit + depreciation) / sales': -1,
    '(gross profit + interest) / total assets': -1,
    '(total liabilities * 365) / (gross profit + depreciation)': 0,
    '(gross profit + depreciation) / total liabilities': -1,
    'total assets / total liabilities': -1,
    'gross profit / total assets': -1,
    'gross profit / sales': -1,
    '(inventory * 365) / sales': 0,
    'sales (n) / sales (n-1)': 0,
    'profit on operating activities / total assets': 0,
    'net profit / sales': -1,
    'gross profit (in 3 years) / total assets': -1,
    '(equity - share capital) / total assets': -1,
    '(net profit + depreciation) / total liabilities': -1,
    'profit on operating activities / financial expenses': 0,
    'working capital / fixed assets': 0,
    'logarithm of total assets': 0,
    '(total liabilities - cash) / sales': 1,
    '(gross profit + interest) / sales': -1,
    '(current liabilities * 365) / cost of products sold': 0,
    'operating expenses / short-term liabilities': 0,
    'operating expenses / total liabilities': 0,
    'profit on sales / total assets': 0,
    'total sales / total assets': 0,
    'constant capital / total assets': -1,
    'profit on sales / sales': 0,
    '(current assets - inventory - receivables) / short-term liabilities': -1,
    'total liabilities / ((profit on operating activities + depreciation) * (12/365))': 0,
    'profit on operating activities / sales': -1,
    'rotation receivables + inventory turnover in days': 0,
    '(receivables * 365) / sales': 0,
    'net profit / inventory': -1,
    '(current assets - inventory) / short-term liabilities': -1,
    '(inventory * 365) / cost of products sold': 0,
    'EBITDA (profit on operating activities - depreciation) / total assets': 0,
    'EBITDA (profit on operating activities - depreciation) / sales': 0,
    'current assets / total liabilities': -1,
    'short-term liabilities / total assets': 1,
    '(short-term liabilities * 365) / cost of products sold)': 0,
    'equity / fixed assets': 0,
    'constant capital / fixed assets': 0,
    'working capital': -1,
    '(sales - cost of products sold) / sales': 0,
    '(current assets - inventory - short-term liabilities) / (sales - gross profit - depreciation)': 0,
    'total costs /total sales': 0,
    'long-term liabilities / equity': 0,
    'sales / inventory': 0,
    'sales / receivables': 0,
    '(short-term liabilities *365) / sales': 1,
    'sales / short-term liabilities': -1,
    'sales / fixed assets': 0
}

POLAND_BANKRUPTCY_MAX_LEAF_NODES_DICT = {feature: 5 for feature in POLAND_BANKRUPTCY_NUMERIC_COLS}

POLAND_BANKRUPTCY_SPECIAL_VALUES_DICT = {feature: [-6] for feature in POLAND_BANKRUPTCY_NUMERIC_COLS}

POLAND_BANKRUPTCY_SPECIAL_VALUE_THRESHOLD = -5.5

POLAND_BANKRUPTCY_RC_FEATURE_MAPPING = OrderedDict({
    "Financing": [
        "(equity - share capital) / total assets"
    ],
    "CurrentRatio": [
        "total liabilities / ((profit on operating activities + depreciation) * (12/365))"
    ],
    "WorkingCapital":[
        "working capital / total assets", "working capital / fixed assets", "constant capital / total assets", "working capital",
        "constant capital / fixed assets", "logarithm of total assets"
    ],
    "LiabilitiesTurnoverRatios": [
        "equity / fixed assets", "book value of equity / total liabilities", "equity / total assets", "(net profit + depreciation) / total liabilities",
        "sales / short-term liabilities", "(current liabilities * 365) / cost of products sold", "operating expenses / short-term liabilities",
        "(short-term liabilities * 365) / cost of products sold)", "short-term liabilities / total assets", "(current assets - inventory) / short-term liabilities",
        "(current assets - inventory - receivables) / short-term liabilities", "operating expenses / total liabilities"
    ],
    "ProfitabilityRatios": [
        "(gross profit + depreciation) / sales", "profit on operating activities / total assets", "(gross profit + interest) / sales",
        "rotation receivables + inventory turnover in days", "net profit / total assets",
        "(gross profit + extraordinary items + financial expenses) / total assets", "(gross profit + interest) / total assets",
        "gross profit / total assets", "gross profit (in 3 years) / total assets"
        
    ],
    "LeverageRatios": [
        "(total liabilities * 365) / (gross profit + depreciation)", "total liabilities / total assets", "current assets / short-term liabilities",
        "total assets / total liabilities", "gross profit / short-term liabilities", "(gross profit + depreciation) / total liabilities",
        "long-term liabilities / equity",
    ],
    "TurnoverRatios": [
        "(inventory * 365) / sales", "sales (n) / sales (n-1)", "(receivables * 365) / sales", "net profit / inventory",
        "(current assets - inventory) / short-term liabilities", "(inventory * 365) / cost of products sold",
        "current assets / total liabilities",
    ],
    "OperatingPerformanceRatios": [
        "sales / total assets", "total sales / total assets", "EBITDA (profit on operating activities - depreciation) / sales",
        "((cash + short-term securities + receivables - short-term liabilities) / (operating expenses - depreciation)) * 365",
        "retained earnings / total assets", "EBIT / total assets", "EBITDA (profit on operating activities - depreciation) / total assets",
        "(current assets - inventory - short-term liabilities) / (sales - gross profit - depreciation)", "profit on operating activities / financial expenses"
    ],
    "SalesInventoryRatios": [
        "(sales - cost of products sold) / sales",
        "total costs /total sales", "sales / inventory", "sales / receivables", "net profit / inventory"
    ],
    "SalesLiabilityRatios": [
        "(short-term liabilities *365) / sales", "(total liabilities - cash) / sales"
    ],
    "ProfitabilitySalesRatios": [
        "gross profit / sales", "net profit / sales", "profit on sales / sales", "profit on operating activities / sales"
    ],
    "SalesCapitalRatios": [
        "profit on sales / total assets", "sales / fixed assets"
    ]
})

\end{lstlisting}

\begin{lstlisting}[language=Python, caption=Lending Club preprocessing]
LENDING_CLUB_SPECIAL_VALUE_THRESHOLD = -0.5

LENDING_CLUB_CATEGORICAL_COLS = [
    'term', 'emp_length', 'home_ownership', 'verification_status',
    'pymnt_plan', 'purpose', 'initial_list_status', 'application_type',
    'hardship_flag', 'disbursement_method', 'debt_settlement_flag'
]

LENDING_CLUB_NON_CATEGORICAL_COLS = [
    'loan_amnt', 'funded_amnt', 'funded_amnt_inv', 'int_rate', 'installment', 'sub_grade',
    'annual_inc', 'issue_d', 'dti', 'delinq_2yrs', 'fico_range_low', 'fico_range_high',
    'inq_last_6mths', 'mths_since_last_delinq', 'mths_since_last_record', 'open_acc', 'pub_rec', 
    'revol_bal', 'revol_util', 'total_acc', 'out_prncp', 'out_prncp_inv',
    'total_pymnt', 'total_pymnt_inv', 'total_rec_prncp', 'total_rec_int', 'total_rec_late_fee', 
    'recoveries', 'collection_recovery_fee', 'last_pymnt_amnt', 'last_fico_range_high', 
    'last_fico_range_low', 'collections_12_mths_ex_med', 'mths_since_last_major_derog', 'acc_now_delinq',
    'tot_coll_amt', 'tot_cur_bal', 'open_acc_6m', 'open_act_il',
    'open_il_12m', 'open_il_24m', 'mths_since_rcnt_il', 'total_bal_il', 'il_util',
    'open_rv_12m', 'open_rv_24m', 'max_bal_bc', 'all_util',
    'total_rev_hi_lim', 'inq_fi', 'total_cu_tl', 'inq_last_12m',
    'acc_open_past_24mths', 'avg_cur_bal', 'bc_open_to_buy', 'bc_util', 'chargeoff_within_12_mths',
    'delinq_amnt', 'mo_sin_old_il_acct', 'mo_sin_old_rev_tl_op', 'mo_sin_rcnt_rev_tl_op',
    'mo_sin_rcnt_tl', 'mort_acc', 'mths_since_recent_bc', 'mths_since_recent_bc_dlq',
    'mths_since_recent_inq', 'mths_since_recent_revol_delinq', 'num_accts_ever_120_pd', 'num_actv_bc_tl',
    'num_actv_rev_tl', 'num_bc_sats', 'num_bc_tl', 'num_il_tl',
    'num_op_rev_tl', 'num_rev_accts', 'num_rev_tl_bal_gt_0', 'num_sats',
    'num_tl_120dpd_2m', 'num_tl_30dpd', 'num_tl_90g_dpd_24m', 'num_tl_op_past_12m',
    'pct_tl_nvr_dlq', 'percent_bc_gt_75', 'pub_rec_bankruptcies', 'tax_liens',
    'tot_hi_cred_lim', 'total_bal_ex_mort', 'total_bc_limit', 'total_il_high_credit_limit'
]

LENDING_CLUB_MONOTONE_CONSTRAINTS = OrderedDict({
    'issue_d': 1,
    'tot_coll_amt': 0,
    'num_bc_sats': 1,
    'total_rev_hi_lim': -1,
    'last_fico_range_low': -1,
    'num_rev_tl_bal_gt_0': 1,
    'mths_since_recent_bc': -1,
    'revol_bal': 0,
    'mo_sin_rcnt_rev_tl_op': -1,
    'out_prncp_inv': 1,
    'total_bal_ex_mort': 0,
    'mths_since_last_delinq': 0,
    'recoveries': 0,
    'chargeoff_within_12_mths': 1,
    'fico_range_high': -1,
    'total_pymnt': -1,
    'mths_since_rcnt_il': 0,
    'open_act_il': 0,
    'bc_util': 1,
    'revol_util': 1,
    'open_il_24m': 1,
    'total_pymnt_inv': -1,
    'il_util': 1,
    'mths_since_recent_bc_dlq': 0,
    'num_rev_accts': 0,
    'pub_rec': 0,
    'num_sats': 1,
    'num_il_tl': 0,
    'out_prncp': 1,
    'all_util': 1,
    'num_tl_30dpd': 1,
    'collections_12_mths_ex_med': 1,
    'total_acc': 0,
    'num_actv_bc_tl': 1,
    'delinq_amnt': 1,
    'num_tl_op_past_12m': 1,
    'mths_since_last_major_derog': 0,
    'tot_cur_bal': 0,
    'total_rec_prncp': 0,
    'inq_last_12m': 1,
    'inq_fi': 1,
    'dti': 1,
    'num_bc_tl': 0,
    'total_rec_int': 0,
    'tot_hi_cred_lim': 0,
    'avg_cur_bal': 0,
    'funded_amnt_inv': 0,
    'num_tl_120dpd_2m': 1,
    'open_acc_6m': 1,
    'funded_amnt': 0,
    'tax_liens': 0,
    'open_rv_24m': 1,
    'percent_bc_gt_75': 0,
    'mo_sin_old_rev_tl_op': 0,
    'num_op_rev_tl': 1,
    'int_rate': 1,
    'total_cu_tl': 0,
    'pct_tl_nvr_dlq': 0,
    'num_accts_ever_120_pd': 0,
    'sub_grade': 1,
    'total_il_high_credit_limit': 0,
    'mo_sin_old_il_acct': 0,
    'mths_since_recent_revol_delinq': 0,
    'open_rv_12m': 1,
    'acc_now_delinq': 1,
    'last_fico_range_high': -1,
    'mths_since_recent_inq': -1,
    'mort_acc': -1,
    'total_bal_il': 0,
    'inq_last_6mths': 1,
    'last_pymnt_amnt': 0,
    'max_bal_bc': 0,
    'collection_recovery_fee': 0,
    'num_actv_rev_tl': 1,
    'open_il_12m': 0,
    'delinq_2yrs': 1,
    'mo_sin_rcnt_tl': -1,
    'bc_open_to_buy': -1,
    'loan_amnt': 0,
    'mths_since_last_record': 0,
    'installment': 0,
    'fico_range_low': -1,
    'total_rec_late_fee': 0,
    'open_acc': 1,
    'acc_open_past_24mths': 1,
    'annual_inc': -1,
    'num_tl_90g_dpd_24m': 0,
    'total_bc_limit': -1,
    'pub_rec_bankruptcies': 1
})

LENDING_CLUB_MONOTONE_CONSTRAINTS.update({col: 0 for col in LENDING_CLUB_CATEGORICAL_COLS})


LENDING_CLUB_SPECIAL_VALUES_DICT = { feature: [LENDING_CLUB_SPECIAL_VALUE_THRESHOLD] 
                              for feature in LENDING_CLUB_NON_CATEGORICAL_COLS
                            } 

LENDING_CLUB_RC_FEATURE_MAPPING = OrderedDict({
    "LoanInfo": [
        "loan_amnt", "funded_amnt", "funded_amnt_inv", "term", "sub_grade", "issue_d", 
        "pymnt_plan", "purpose", "initial_list_status", "application_type", "disbursement_method"
    ],
    "LoanStatus": [
        "out_prncp", "out_prncp_inv", "total_pymnt", "total_pymnt_inv", 
        "total_rec_prncp", "total_rec_int", "total_rec_late_fee", "recoveries", 
        "collection_recovery_fee", "last_pymnt_amnt"
    ],
    "PersonalInfo": [
        "emp_length", "home_ownership"
    ],
    "FinancialCapabilities": [
        "annual_inc", "verification_status", "open_acc", "total_acc", "tot_cur_bal",
        "bc_util", "bc_open_to_buy", "mort_acc", "num_actv_bc_tl",
        "num_bc_sats", "num_bc_tl", "num_sats", "tot_hi_cred_lim", 
        "total_bc_limit", "total_il_high_credit_limit", "avg_cur_bal"
    ],
    "FinancialLiabilities": [
        "dti", "percent_bc_gt_75", "pub_rec_bankruptcies", "tax_liens", "hardship_flag",
        "debt_settlement_flag", "total_bal_ex_mort"
    ],
    "ExternalRiskEstimate": [
        "fico_range_low", "fico_range_high", "last_fico_range_high", "last_fico_range_low"
    ],
    "TradeOpenTime": [
        "total_cu_tl", "acc_open_past_24mths", "mths_since_recent_bc"
    ],
    "TradeQuality": [
        "pub_rec", "pct_tl_nvr_dlq"
    ],
    "TradeFrequency": [
        "mths_since_last_record", "open_acc_6m", "mo_sin_old_il_acct", "mo_sin_old_rev_tl_op", 
        "mo_sin_rcnt_rev_tl_op", "mo_sin_rcnt_tl", "num_actv_rev_tl", "num_tl_op_past_12m"
    ],
    "Delinquency": [
        "delinq_2yrs", "mths_since_last_delinq", "collections_12_mths_ex_med", 
        "mths_since_last_major_derog", "acc_now_delinq", "tot_coll_amt", "chargeoff_within_12_mths", 
        "delinq_amnt", "mths_since_recent_bc_dlq", "mths_since_recent_revol_delinq", 
        "num_accts_ever_120_pd", "num_tl_120dpd_2m", "num_tl_30dpd", "num_tl_90g_dpd_24m", 
    ],
    "Installment": [
        "int_rate", "installment", "open_act_il", "open_il_12m", "open_il_24m",
        "mths_since_rcnt_il", "total_bal_il", "il_util", "num_il_tl"
    ],
    "Inquiry": [
        "inq_last_6mths", "inq_fi", "inq_last_12m", "mths_since_recent_inq"
    ],
    "RevolvingBalance": [
        "revol_bal", "open_rv_12m", "open_rv_24m", "max_bal_bc", "total_rev_hi_lim", 
        "num_op_rev_tl", "num_rev_accts", "num_rev_tl_bal_gt_0", 
    ],
    "Utilization": [
        "revol_util", "all_util"
    ],
})


LENDING_CLUB_MAX_LEAF_NODES_DICT = {feature: 5 for feature in LENDING_CLUB_NON_CATEGORICAL_COLS}
\end{lstlisting}

\clearpage

\section{Raw Experimental Metrics}
\label{sec:raw_metrics}
See Tables~\ref{tab:auc_raw},~\ref{tab:ece_raw}~and~\ref{tab:mce_raw} for AUC, ECE and MCE respectively.

% \newgeometry{top=0.1cm,bottom=0.1cm}
% \setlength{\tabcolsep}{10pt}
\begin{table*}[h]
        \captionof{table}{Raw AUC scores (mean over 10 CV folds) along with their ranks in parenthesis. Scores in bold are best for a particular dataset.}
        \label{tab:auc_raw}
        {\scriptsize
        \begin{longtable}{lrrrrrrrr}
\toprule
\textbf{Classifier} &       \rotatebox{60}{NNLR} &    \rotatebox{60}{LAM-NNLR} &             \rotatebox{60}{ARM1} &       \rotatebox{60}{ARM2} &          \rotatebox{60}{LAM-ARM1} &    \rotatebox{60}{LAM-ARM2} &          \rotatebox{60}{MonoXGB} &              \rotatebox{60}{XGB} \\
\textbf{Dataset  } &                            &                            &                                  &                            &                                  &                            &                                  &                                  \\
\midrule
\endfirsthead

\toprule
\textbf{Classifier} &       \rotatebox{60}{NNLR} &    \rotatebox{60}{LAM-NNLR} &             \rotatebox{60}{ARM1} &       \rotatebox{60}{ARM2} &          \rotatebox{60}{LAM-ARM1} &    \rotatebox{60}{LAM-ARM2} &          \rotatebox{60}{MonoXGB} &              \rotatebox{60}{XGB} \\
\textbf{Dataset  } &                            &                            &                                  &                            &                                  &                            &                                  &                                  \\
\midrule
\endhead
\midrule
\multicolumn{9}{r}{{Continued on next page}} \\
\midrule
\endfoot

\bottomrule
\endlastfoot
\textbf{Australia} &  0.9263\,\hphantom{1}(3.0) &  0.9231\,\hphantom{1}(4.0) &        0.9004\,\hphantom{1}(7.0) &  0.9112\,\hphantom{1}(5.0) &        0.8779\,\hphantom{1}(8.0) &  0.9039\,\hphantom{1}(6.0) &  {\bf 0.9392\,\hphantom{1}(1.0)} &        0.9385\,\hphantom{1}(2.0) \\
\textbf{GMSC     } &  0.5609\,\hphantom{1}(7.5) &  0.5609\,\hphantom{1}(7.5) &  {\bf 0.8589\,\hphantom{1}(1.0)} &  0.8560\,\hphantom{1}(4.5) &        0.8586\,\hphantom{1}(2.0) &  0.8560\,\hphantom{1}(4.5) &        0.8557\,\hphantom{1}(6.0) &        0.8585\,\hphantom{1}(3.0) \\
\textbf{German   } &  0.7568\,\hphantom{1}(5.5) &  0.7568\,\hphantom{1}(5.5) &  {\bf 0.7790\,\hphantom{1}(1.0)} &  0.7486\,\hphantom{1}(8.0) &        0.7789\,\hphantom{1}(2.0) &  0.7499\,\hphantom{1}(7.0) &        0.7700\,\hphantom{1}(4.0) &        0.7735\,\hphantom{1}(3.0) \\
\textbf{HELOC    } &  0.7558\,\hphantom{1}(7.5) &  0.7558\,\hphantom{1}(7.5) &  {\bf 0.7963\,\hphantom{1}(1.5)} &  0.7952\,\hphantom{1}(4.0) &  {\bf 0.7963\,\hphantom{1}(1.5)} &  0.7958\,\hphantom{1}(3.0) &        0.7797\,\hphantom{1}(6.0) &        0.7931\,\hphantom{1}(5.0) \\
\textbf{Japan    } &  0.9338\,\hphantom{1}(3.0) &  0.9313\,\hphantom{1}(4.0) &        0.9081\,\hphantom{1}(5.0) &  0.9063\,\hphantom{1}(6.0) &        0.8968\,\hphantom{1}(8.0) &  0.9015\,\hphantom{1}(7.0) &        0.9378\,\hphantom{1}(2.0) &  {\bf 0.9381\,\hphantom{1}(1.0)} \\
\textbf{LC\_0    } &  0.5418\,\hphantom{1}(7.5) &  0.5418\,\hphantom{1}(7.5) &        0.9956\,\hphantom{1}(3.0) &  0.9836\,\hphantom{1}(5.0) &        0.9919\,\hphantom{1}(4.0) &  0.9798\,\hphantom{1}(6.0) &  {\bf 0.9996\,\hphantom{1}(1.5)} &  {\bf 0.9996\,\hphantom{1}(1.5)} \\
\textbf{LC\_1    } &  0.5247\,\hphantom{1}(7.5) &  0.5247\,\hphantom{1}(7.5) &  {\bf 0.9988\,\hphantom{1}(1.5)} &  0.9955\,\hphantom{1}(5.0) &        0.9965\,\hphantom{1}(4.0) &  0.9902\,\hphantom{1}(6.0) &        0.9987\,\hphantom{1}(3.0) &  {\bf 0.9988\,\hphantom{1}(1.5)} \\
\end{longtable}
        }
\end{table*}

\begin{table*}
        \captionof{table}{Raw ECE scores (mean over 10 CV folds) along with their ranks in parenthesis. Scores in bold are best for a particular dataset.}
        \label{tab:ece_raw}
        {\scriptsize
        \begin{longtable}{lrrrrrrrr}
\toprule
\textbf{Classifier} &             \rotatebox{60}{NNLR} &          \rotatebox{60}{LAM-NNLR} &       \rotatebox{60}{ARM1} &       \rotatebox{60}{ARM2} &    \rotatebox{60}{LAM-ARM1} &    \rotatebox{60}{LAM-ARM2} &          \rotatebox{60}{MonoXGB} &              \rotatebox{60}{XGB} \\
\textbf{Dataset  } &                                  &                                  &                            &                            &                            &                            &                                  &                                  \\
\midrule
\endfirsthead

\toprule
\textbf{Classifier} &             \rotatebox{60}{NNLR} &          \rotatebox{60}{LAM-NNLR} &       \rotatebox{60}{ARM1} &       \rotatebox{60}{ARM2} &    \rotatebox{60}{LAM-ARM1} &    \rotatebox{60}{LAM-ARM2} &          \rotatebox{60}{MonoXGB} &              \rotatebox{60}{XGB} \\
\textbf{Dataset  } &                                  &                                  &                            &                            &                            &                            &                                  &                                  \\
\midrule
\endhead
\midrule
\multicolumn{9}{r}{{Continued on next page}} \\
\midrule
\endfoot

\bottomrule
\endlastfoot
\textbf{Australia} &        0.2095\,\hphantom{1}(3.0) &        0.2215\,\hphantom{1}(6.0) &  0.2353\,\hphantom{1}(7.0) &  0.2142\,\hphantom{1}(4.0) &  0.2372\,\hphantom{1}(8.0) &  0.2167\,\hphantom{1}(5.0) &        0.2042\,\hphantom{1}(2.0) &  {\bf 0.2015\,\hphantom{1}(1.0)} \\
\textbf{GMSC     } &  {\bf 0.2166\,\hphantom{1}(1.5)} &  {\bf 0.2166\,\hphantom{1}(1.5)} &  0.2171\,\hphantom{1}(6.0) &  0.2169\,\hphantom{1}(5.0) &  0.2193\,\hphantom{1}(8.0) &  0.2185\,\hphantom{1}(7.0) &        0.2167\,\hphantom{1}(3.5) &        0.2167\,\hphantom{1}(3.5) \\
\textbf{German   } &        0.1297\,\hphantom{1}(4.0) &        0.1273\,\hphantom{1}(2.0) &  0.1689\,\hphantom{1}(8.0) &  0.1481\,\hphantom{1}(6.0) &  0.1660\,\hphantom{1}(7.0) &  0.1448\,\hphantom{1}(5.0) &        0.1296\,\hphantom{1}(3.0) &  {\bf 0.1268\,\hphantom{1}(1.0)} \\
\textbf{HELOC    } &        0.0986\,\hphantom{1}(3.0) &  {\bf 0.0880\,\hphantom{1}(1.0)} &  0.1173\,\hphantom{1}(7.0) &  0.1177\,\hphantom{1}(8.0) &  0.1121\,\hphantom{1}(6.0) &  0.1085\,\hphantom{1}(5.0) &        0.0931\,\hphantom{1}(2.0) &        0.1068\,\hphantom{1}(4.0) \\
\textbf{Japan    } &        0.2012\,\hphantom{1}(3.0) &        0.2030\,\hphantom{1}(4.0) &  0.2293\,\hphantom{1}(7.0) &  0.2066\,\hphantom{1}(5.0) &  0.2338\,\hphantom{1}(8.0) &  0.2159\,\hphantom{1}(6.0) &  {\bf 0.1874\,\hphantom{1}(1.5)} &  {\bf 0.1874\,\hphantom{1}(1.5)} \\
\textbf{LC\_0    } &  {\bf 0.1720\,\hphantom{1}(1.0)} &        0.1722\,\hphantom{1}(2.0) &  0.2370\,\hphantom{1}(5.0) &  0.2256\,\hphantom{1}(3.0) &  0.2414\,\hphantom{1}(8.0) &  0.2345\,\hphantom{1}(4.0) &        0.2411\,\hphantom{1}(6.5) &        0.2411\,\hphantom{1}(6.5) \\
\textbf{LC\_1    } &  {\bf 0.1718\,\hphantom{1}(1.5)} &  {\bf 0.1718\,\hphantom{1}(1.5)} &  0.2457\,\hphantom{1}(6.0) &  0.2406\,\hphantom{1}(5.0) &  0.2476\,\hphantom{1}(8.0) &  0.2471\,\hphantom{1}(7.0) &        0.2322\,\hphantom{1}(3.0) &        0.2329\,\hphantom{1}(4.0) \\
\end{longtable}
        }
\end{table*}

\begin{table*}
        \captionof{table}{Raw MCE scores (mean over 10 CV folds) along with their ranks in parenthesis. Scores in bold are best for a particular dataset.}
        \label{tab:mce_raw}
        {\scriptsize
        \begin{longtable}{lrrrrrrrr}
\toprule
\textbf{Classifier} &             \rotatebox{60}{NNLR} &          \rotatebox{60}{LAM-NNLR} &             \rotatebox{60}{ARM1} &             \rotatebox{60}{ARM2} &    \rotatebox{60}{LAM-ARM1} &    \rotatebox{60}{LAM-ARM2} &          \rotatebox{60}{MonoXGB} &              \rotatebox{60}{XGB} \\
\textbf{Dataset  } &                                  &                                  &                                  &                                  &                            &                            &                                  &                                  \\
\midrule
\endfirsthead

\toprule
\textbf{Classifier} &             \rotatebox{60}{NNLR} &          \rotatebox{60}{LAM-NNLR} &             \rotatebox{60}{ARM1} &             \rotatebox{60}{ARM2} &    \rotatebox{60}{LAM-ARM1} &    \rotatebox{60}{LAM-ARM2} &          \rotatebox{60}{MonoXGB} &              \rotatebox{60}{XGB} \\
\textbf{Dataset  } &                                  &                                  &                                  &                                  &                            &                            &                                  &                                  \\
\midrule
\endhead
\midrule
\multicolumn{9}{r}{{Continued on next page}} \\
\midrule
\endfoot

\bottomrule
\endlastfoot
\textbf{Australia} &        0.7326\,\hphantom{1}(2.0) &        0.7467\,\hphantom{1}(4.0) &  {\bf 0.7107\,\hphantom{1}(1.0)} &        0.7961\,\hphantom{1}(7.0) &  0.7342\,\hphantom{1}(3.0) &  0.8254\,\hphantom{1}(8.0) &        0.7951\,\hphantom{1}(6.0) &        0.7821\,\hphantom{1}(5.0) \\
\textbf{GMSC     } &        0.9116\,\hphantom{1}(7.0) &        0.9208\,\hphantom{1}(8.0) &        0.8934\,\hphantom{1}(4.0) &        0.8871\,\hphantom{1}(3.0) &  0.9094\,\hphantom{1}(6.0) &  0.9045\,\hphantom{1}(5.0) &        0.8862\,\hphantom{1}(2.0) &  {\bf 0.8853\,\hphantom{1}(1.0)} \\
\textbf{German   } &  {\bf 0.7185\,\hphantom{1}(1.0)} &        0.8270\,\hphantom{1}(7.0) &        0.7700\,\hphantom{1}(4.0) &        0.8261\,\hphantom{1}(6.0) &  0.8014\,\hphantom{1}(5.0) &  0.9089\,\hphantom{1}(8.0) &        0.7198\,\hphantom{1}(2.0) &        0.7468\,\hphantom{1}(3.0) \\
\textbf{HELOC    } &        0.6182\,\hphantom{1}(8.0) &        0.5901\,\hphantom{1}(7.0) &        0.5601\,\hphantom{1}(6.0) &        0.5074\,\hphantom{1}(2.0) &  0.5406\,\hphantom{1}(5.0) &  0.5399\,\hphantom{1}(4.0) &        0.5383\,\hphantom{1}(3.0) &  {\bf 0.4691\,\hphantom{1}(1.0)} \\
\textbf{Japan    } &        0.7487\,\hphantom{1}(3.0) &        0.8163\,\hphantom{1}(7.0) &        0.8086\,\hphantom{1}(6.0) &  {\bf 0.7170\,\hphantom{1}(1.0)} &  0.8301\,\hphantom{1}(8.0) &  0.7430\,\hphantom{1}(2.0) &        0.7600\,\hphantom{1}(5.0) &        0.7563\,\hphantom{1}(4.0) \\
\textbf{LC\_0    } &        0.4441\,\hphantom{1}(2.0) &  {\bf 0.4430\,\hphantom{1}(1.0)} &        0.8356\,\hphantom{1}(4.0) &        0.8181\,\hphantom{1}(3.0) &  0.8423\,\hphantom{1}(8.0) &  0.8407\,\hphantom{1}(7.0) &        0.8365\,\hphantom{1}(5.0) &        0.8371\,\hphantom{1}(6.0) \\
\end{longtable}
        }
\end{table*}

\begin{table*}
        \captionof{table}{Fraction of test examples with prediction in $\{0, 1\}$ (mean over 10 CV folds) along with their ranks in parenthesis. Scores in bold are best for a particular dataset.}
        \label{tab:full_certainty}
        {\scriptsize        \begin{longtable}{lrrrrrrrr}
\toprule
\textbf{Classifier} &             \rotatebox{60}{NNLR} &          \rotatebox{60}{LAM-NNLR} &             \rotatebox{60}{ARM1} &             \rotatebox{60}{ARM2} &    \rotatebox{60}{LAM-ARM1} &    \rotatebox{60}{LAM-ARM2} &          \rotatebox{60}{MonoXGB} &              \rotatebox{60}{XGB} \\
\textbf{Dataset  } &                                  &                                  &                                  &                                  &                            &                            &                                  &                                  \\
\midrule
\endfirsthead

\toprule
\textbf{Classifier} &             \rotatebox{60}{NNLR} &          \rotatebox{60}{LAM-NNLR} &             \rotatebox{60}{ARM1} &             \rotatebox{60}{ARM2} &    \rotatebox{60}{LAM-ARM1} &    \rotatebox{60}{LAM-ARM2} &          \rotatebox{60}{MonoXGB} &              \rotatebox{60}{XGB} \\
\textbf{Dataset  } &                                  &                                  &                                  &                                  &                            &                            &                                  &                                  \\
\midrule
\endhead
\midrule
\multicolumn{9}{r}{{Continued on next page}} \\
\midrule
\endfoot

\bottomrule
\endlastfoot
\textbf{Australia} &        0.0043\,\hphantom{1}(4.0) &        0.4739\,\hphantom{1}(6.0) &        0.0681\,\hphantom{1}(5.0) &  {\bf 0.0000\,\hphantom{1}(2.0)} &  0.7797\,\hphantom{1}(8.0) &  0.5681\,\hphantom{1}(7.0) &  {\bf 0.0000\,\hphantom{1}(2.0)} &  {\bf 0.0000\,\hphantom{1}(2.0)} \\
\textbf{GMSC     } &  {\bf 0.0000\,\hphantom{1}(3.0)} &        0.0002\,\hphantom{1}(6.0) &  {\bf 0.0000\,\hphantom{1}(3.0)} &  {\bf 0.0000\,\hphantom{1}(3.0)} &  0.0481\,\hphantom{1}(8.0) &  0.0323\,\hphantom{1}(7.0) &  {\bf 0.0000\,\hphantom{1}(3.0)} &  {\bf 0.0000\,\hphantom{1}(3.0)} \\
\textbf{German   } &  {\bf 0.0000\,\hphantom{1}(3.0)} &        0.0010\,\hphantom{1}(6.0) &  {\bf 0.0000\,\hphantom{1}(3.0)} &  {\bf 0.0000\,\hphantom{1}(3.0)} &  0.1930\,\hphantom{1}(8.0) &  0.0540\,\hphantom{1}(7.0) &  {\bf 0.0000\,\hphantom{1}(3.0)} &  {\bf 0.0000\,\hphantom{1}(3.0)} \\
\textbf{HELOC    } &  {\bf 0.0000\,\hphantom{1}(3.0)} &        0.0119\,\hphantom{1}(6.0) &  {\bf 0.0000\,\hphantom{1}(3.0)} &  {\bf 0.0000\,\hphantom{1}(3.0)} &  0.0483\,\hphantom{1}(8.0) &  0.0379\,\hphantom{1}(7.0) &  {\bf 0.0000\,\hphantom{1}(3.0)} &  {\bf 0.0000\,\hphantom{1}(3.0)} \\
\textbf{Japan    } &        0.0043\,\hphantom{1}(4.0) &        0.4826\,\hphantom{1}(6.0) &        0.1043\,\hphantom{1}(5.0) &  {\bf 0.0000\,\hphantom{1}(2.0)} &  0.7696\,\hphantom{1}(8.0) &  0.5942\,\hphantom{1}(7.0) &  {\bf 0.0000\,\hphantom{1}(2.0)} &  {\bf 0.0000\,\hphantom{1}(2.0)} \\
\textbf{LC\_0    } &  {\bf 0.0000\,\hphantom{1}(2.5)} &        0.0042\,\hphantom{1}(5.0) &        0.1469\,\hphantom{1}(6.0) &  {\bf 0.0000\,\hphantom{1}(2.5)} &  0.8708\,\hphantom{1}(8.0) &  0.7672\,\hphantom{1}(7.0) &  {\bf 0.0000\,\hphantom{1}(2.5)} &  {\bf 0.0000\,\hphantom{1}(2.5)} \\
\end{longtable}
        }
\end{table*}
% \restoregeometry
\clearpage
\section{Statistical Calculations}
\subsection{Performance Evaluation}
\label{sec:stats}

We follow the advice of~\cite{Demsar2006} and first conduct the Friedman omnibus test~\cite{Friedman1940} with Iman-Davenport correction~\cite{Iman1980}.
Having rejected the null hypothesis that the ranks of each of the algorithms are identical, we conduct the pairwise post-hoc analysis recommended by~\cite{JMLR:v17:benavoli16a}, whereby a Wilcoxon signed-rank test~\cite{Wilcoxon1945} is conducted with Holm’s alpha correction~\cite{Holm1979,JMLR:v9:garcia08a} to control the family-wise error rate.

The Wilcoxon signed-rank test has an associated \emph{Hodges-Lehmann estimator}~\cite{WILCOX202245}, namely, a point estimate of this observed difference across the $k$ datasets.
This estimator is the \emph{psuedomedian},  $\hat{\theta}_{\text{HL}}$, and is defined as
$\hat{\theta}_{\text{HL}} = \operatorname{median}_{1 \leq i \leq j \leq N} \{ \frac{d_i + d_j}{2} \}$,
namely the median of the pairwise \emph{Walsh averages} $\frac{d_i + d_j}{2}$.
The quantities $d_i$ and $d_j$ correspond to the difference between a fixed pair of algorithms $(\mathcal{A}, \mathcal{A}')$ in performance on the $i$\textsuperscript{th} and $j$\textsuperscript{th} of $N$ datasets respectively.
The pseudomedian constitutes a robust estimate of the difference in score between a pair of algorithms.
Care must be taken~\cite{Demsar2006} since we are implicitly assuming that the pairwise score differences between two algorithms are commensurable across datasets.
Nonetheless we believe these point estimates are useful to report for the reader to have an idea of the scales involved.
In Tables~\ref{tab:comparison_auc},~\ref{tab:comparison_ece} (Appendix)~and~\ref{tab:comparison_mce}~(Appendix) we tabulate the differences $\hat{\theta}_{\text{HL}}$ between all pairings of the $k=8$ algorithms for the AUC, ECE and MCE metrics respectively.

\paragraph{Friedman omnibus test.}
The Friedman test~\cite{Friedman1940} is a non-parametric version of the ANOVA test.
Let $r_{i,j}$ be the rank of the $i$\textsuperscript{th} of $k$ algorithms on the $j$\textsuperscript{th} of $N$ datasets, where a rank of 1 corresponds to the best algorithm, 2 the second-best, and so on.
Ties are assigned the arithmetic mean of the constituent ranks, for instance, if two algorithms are joint first, then they are assigned the rank $\frac{1 + 2}{2} = 1.5$.

The Friedman test is a comparison of the average ranks over all of the datasets, $R_i = \frac{1}{N}\sum^N_{j=1} r_{i, j}$.
Under the null hypothesis that the average ranks are all equal, the Friedman statistic
\begin{equation}
    \label{eq:friedman}
    \chi^2_F = \frac{12 N}{k (k + 1)} \qty[\sum_{i=1}^k R_i^2 - \frac{k(k+1)^2}{4}]
\end{equation}
is distributed according to $\chi^2_F$ with $k - 1$ degrees of freedom.
It is well known that the Friedman statistic is often unnecessarily conservative, so we use a more accurate $F_F$ statistic~\cite{Iman1980} defined as
\begin{equation}
    \label{eq:imandavenport}
    F_F = \frac{(N-1) \chi^2_F}{ N(k - 1)- \chi^2_F}
\end{equation}
Under the null hypothesis the $F_F$ statistic is distributed according to the $F$ distribution with $k - 1$ and $(k - 1)(N - 1)$ degrees of freedom.

\paragraph{Post-hoc analysis.}
If the null hypothesis is rejected, we proceed with pairwise comparisons between the $k$ algorithms.
Fixing such a pair, let $d_j$ be the difference in performance on the $j$\textsuperscript{th} of $N$ datasets.
The Wilcoxon signed-rank test~\cite{Wilcoxon1945} is a non-parametric version of the paired $t$-test.
The differences $d_j$ are ranked according to their absolute values with average ranks being assigned in case of ties. 
Let $R^+$ be the sum of ranks for the data sets on which the second algorithm outperformed the first, and $R^-$ the sum of ranks for the converse.
Ranks of $d_j = 0$ are split evenly among the sums. 
If there is an odd number of ties, one is ignored. More precisely
\begin{equation}
\begin{aligned}
    \label{eq:Rplusminus}
    R^{+} &= \sum_{d_j > 0} \operatorname{rank}(d_j) + \frac{1}{2}\sum_{d_j = 0} \operatorname{rank}(d_j),\\ 
    R^{-} &= \sum_{d_j < 0} \operatorname{rank}(d_j) + \frac{1}{2}\sum_{d_j = 0} \operatorname{rank}(d_j)
\end{aligned}
\end{equation}
Let $T$ be the smaller of the sums, $T = \min(R^+, R^-)$.
At $\alpha=0.05$ the exact critical value for $N=24$ is 81, that is if the smaller of $R^+$, $R^-$ is less that 81, we reject the null hypothesis.
For exact $p$-values under the null hypothesis that $d_j=0$ for all $j \in \{1, \ldots, N\}$, we use the precomputed values provided in the \texttt{scipy.stats.wilcoxon} python module~\cite{2020SciPy-NMeth}.
The $k(k - 1) / 2$ $p$-values are computed in this manner between each pair of the $k$ algorithms.

In Holm’s method of multiple hypothesis testing, the individual $p$-values are compared with adjusted $\alpha$ values as follows.
First, the $p$-values are sorted so that $p_1 \leq p_2 \leq \ldots \leq p_{k(k-1)/2}$.
Then, each $p_i$ is compared to $\frac{\alpha}{ k(k-1)/2 - i + 1}$ sequentially. 
So the most significant $p$-value, $p_1$, is compared with $\frac{\alpha}{k(k-1)/2}$. 
If $p_1$ is below $\frac{\alpha}{k(k-1)/2}$, the corresponding hypothesis is rejected and we continue to compare $p_2$ with $\frac{\alpha}{k(k-1)/2 - 1}$, and so on. 
As soon as a certain null
hypothesis cannot be rejected, all the remaining hypotheses must be retained as well. 

\subsection{Trinomial Test}
\label{sec:trinomial}

For the Preference Task, the data correspond to matched pairs, each pair belonging to one participant and corresponding to a preference of logistic models vs LAM models, that is, possible responses are $\{A \prec B, A \succ B, A \sim B \}$, with Model $A$ corresponding to LAM and Model $B$ to logistic.  
As there is no numerical or ranked comparison, the usual appropriate statistical test is the Sign Test~\cite{sign_test}.
Given there are many ties $A \sim B$, we use the Trinomial Test of~\cite{BIAN20111153}, specially developed for this regime.
Let $p_A$ denote the probability a randomly chosen participant prefers LAMs to logistic models and let $p_B$ denote the converse.
Moreover, $p_0$ is the probability neither is preferred.
Then our null hypothesis is $\text{H}_0 : p_A = p_B$ and alternative is  $\text{H}_1 : p_A > p_B$.
Let $N_A$, $N_B$ and $N_0$ be the random variables denoting the observed counts corresponding to Model A preferred, Model B preferred and neither respectively, with $N := N_A + N_B + N_0$.
Assuming $\text{H}_0$, the test statistic $N_d = N_A - N_B$ follows the distribution
\begin{equation*}
P(N_d = n_d) =   \sum_{k=0}^{\left\lfloor{\frac{n - n_d}{2}}\right\rfloor}\frac{n!}{(n_d + k)!k!(n - n_d - 2k)!}\qty(\frac{1 - p_0}{2})^{n_d + 2k} (p_0)^{n - n_d - 2k},
\end{equation*}
where for $p_0$ the unbiased estimate $n_0 / n$ is used in practice.
The critical value of the test statistic $N_d$ at $\alpha = 0.05$ is $C_{0.05} = 5$.

\clearpage
\section{Calibration Differences Across Classifiers}
Tables~\ref{tab:comparison_ece}~and~\ref{tab:comparison_mce} respectively show point estimates for the difference in performance of the models under consideration for ECE and MCE calibration metrics.

\setlength{\tabcolsep}{2pt}
\begin{table}[h]
\centering
        \caption{Difference in ECE scores between classifiers. Value in cell $(i, j)$ corresponds to pseudomedian over all datasets of classifier $i$ cross-validated score minus classifier $j$ cross-validated score. Bold values indicate a difference that is statistically significant.}
        \label{tab:comparison_ece}
        {\footnotesize
        \begin{tabular}{lrrrrrrrr}
\toprule
{} & \rotatebox{75}{NNLR} & \rotatebox{75}{LAM-NNLR} & \rotatebox{75}{ARM1} & \rotatebox{75}{ARM2} & \rotatebox{75}{LAM-ARM1} & \rotatebox{75}{LAM-ARM2} & \rotatebox{75}{MonoXGB} & \rotatebox{75}{XGB} \\
\midrule
NNLR    &                  --- &                   0.000 &       \textbf{0.055} &       \textbf{0.046} &          \textbf{0.057} &          \textbf{0.052} &          \textbf{0.038} &               0.040 \\
LAM-NNLR &                  --- &                     --- &       \textbf{0.053} &       \textbf{0.046} &          \textbf{0.057} &          \textbf{0.052} &          \textbf{0.039} &               0.040 \\
ARM1    &                  --- &                     --- &                  --- &      \textbf{-0.006} &          \textbf{0.003} &                   0.000 &         \textbf{-0.011} &     \textbf{-0.009} \\
ARM2    &                  --- &                     --- &                  --- &                  --- &          \textbf{0.009} &          \textbf{0.006} &         \textbf{-0.006} &     \textbf{-0.004} \\
LAM-ARM1 &                  --- &                     --- &                  --- &                  --- &                     --- &         \textbf{-0.003} &         \textbf{-0.014} &     \textbf{-0.012} \\
LAM-ARM2 &                  --- &                     --- &                  --- &                  --- &                     --- &                     --- &         \textbf{-0.012} &     \textbf{-0.010} \\
MonoXGB &                  --- &                     --- &                  --- &                  --- &                     --- &                     --- &                     --- &               0.001 \\
XGB     &                  --- &                     --- &                  --- &                  --- &                     --- &                     --- &                     --- &                 --- \\
\bottomrule
\end{tabular}
        }
\end{table}

\begin{table}[h]
\centering
        \caption{Difference in MCE scores between classifiers. Value in cell $(i, j)$ corresponds to pseudomedian over all datasets of classifier $i$ cross-validated score minus classifier $j$ cross-validated score. Bold values indicate a difference that is statistically significant.}
        \label{tab:comparison_mce}
        {\footnotesize
        \begin{tabular}{lrrrrrrrr}
\toprule
{} & \rotatebox{75}{NNLR} & \rotatebox{75}{LAM-NNLR} & \rotatebox{75}{ARM1} & \rotatebox{75}{ARM2} & \rotatebox{75}{LAM-ARM1} & \rotatebox{75}{LAM-ARM2} & \rotatebox{75}{MonoXGB} & \rotatebox{75}{XGB} \\
\midrule
NNLR    &                  --- &                   0.010 &                0.210 &                0.194 &                   0.211 &                   0.213 &                   0.190 &               0.184 \\
LAM-NNLR &                  --- &                     --- &                0.190 &                0.183 &                   0.196 &                   0.208 &                   0.171 &               0.172 \\
ARM1    &                  --- &                     --- &                  --- &               -0.013 &          \textbf{0.011} &          \textbf{0.009} &         \textbf{-0.018} &     \textbf{-0.018} \\
ARM2    &                  --- &                     --- &                  --- &                  --- &                   0.020 &          \textbf{0.024} &                  -0.004 &              -0.005 \\
LAM-ARM1 &                  --- &                     --- &                  --- &                  --- &                     --- &                   0.000 &         \textbf{-0.026} &     \textbf{-0.028} \\
LAM-ARM2 &                  --- &                     --- &                  --- &                  --- &                     --- &                     --- &         \textbf{-0.027} &     \textbf{-0.029} \\
MonoXGB &                  --- &                     --- &                  --- &                  --- &                     --- &                     --- &                     --- &               0.000 \\
XGB     &                  --- &                     --- &                  --- &                  --- &                     --- &                     --- &                     --- &                 --- \\
\bottomrule
\end{tabular}
        }
\end{table}

% \section{User Study Questions}
\clearpage
\onecolumn

% \includepdf[nup=2x2,pages=-,delta=15mm 15mm,frame=true]{plots/user_study/XAI_CoE_User_Study_Questions.pdf}
\includepdf[nup=2x2,pages=1-4,scale=.79,pagecommand={\section{User Study Questions}\label{pdf:user_survey_questions}},linktodoc=true]{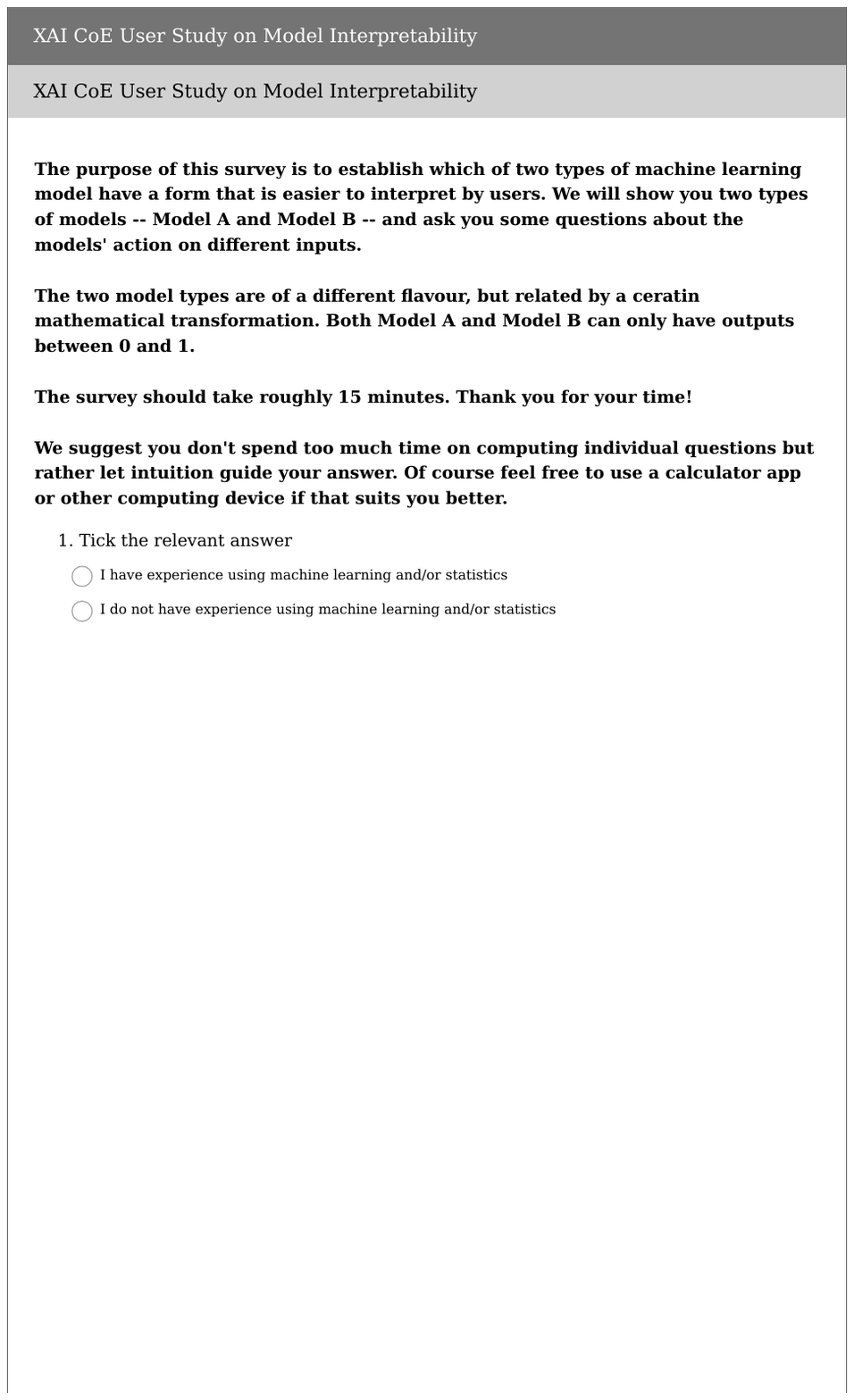}
\includepdf[nup=2x2,pages=5-,scale=.79,pagecommand={},linktodoc=true]{plots/user_study/XAI_CoE_User_Study_Questions.pdf}

% \bibliographystyleSM{alpha}
\bibliographystyleSM{nicebib-alpha}
{\small
\bibliographySM{references_2}
}

\end{document}